\documentclass[12pt]{article}
\usepackage{natbib}
\usepackage{times}
\usepackage{hyperref}       
\usepackage{url}            
\usepackage{booktabs}       
\usepackage{amsfonts}       
\usepackage{nicefrac}       
\usepackage{microtype}      
\usepackage{ltablex}
\usepackage{multirow}
\usepackage{pifont}
\usepackage{hhline}

\usepackage{amsmath, amssymb, amsthm, graphicx, bm}
\usepackage[ruled,vlined,linesnumbered,noend]{algorithm2e}
\usepackage{dsfont}
\usepackage{tikz}
\usetikzlibrary{automata, positioning, arrows}
\usepackage{mwe}
\usepackage[section]{placeins}
\usepackage{enumitem}
\setlist[itemize]{leftmargin=*}
\setlist[enumerate]{leftmargin=*}
\usepackage{nicefrac} 

\ifdefined\usebigfont
\usepackage[yyyymmdd,hhmmss]{datetime}
\AtBeginDocument{Compiled on {\ddmmyyyydate\today} at \currenttime  \\}

\usepackage{times}
\usepackage[fontsize=13pt]{scrextend}
\usepackage[left=1.56in,right=1.56in,top=1.71in,bottom=1.77in]{geometry}
\setlist[itemize]{leftmargin=*}
\setlist[enumerate]{leftmargin=*}
\else
\usepackage{times}
\usepackage{geometry}
\fi

\usepackage{xcolor, color, colortbl}
\usepackage{cleveref}
\usepackage{algorithmic}
\definecolor{LightCyan}{rgb}{0.7,1,1}
\definecolor{LightRed}{rgb}{1,0.7,0.7}

\newcommand{\E}{\mathbb{E}}
\newcommand{\R}{\ensuremath{\mathbb{R}}}

\newcommand{\argmax}{\mathop{{}\textrm{argmax}}}

\newcommand{\ind}{\ensuremath{\mathds{1}}}

\newcommand{\algname}{BiLin-UCB}
\newcommand{\classname}{Bilinear Class}

\newtheoremstyle{definition}
  {15pt}
  {15pt}
  {\itshape}
  {}
  {\bfseries}
  {.}
  { }
  {\thmname{#1}\thmnumber{ #2}\thmnote{ (#3)}}

  \newtheoremstyle{theorem}
  {}
  {}
  {\itshape}
  {}
  {\bfseries}
  {.}
  { }
  {\thmname{#1}\thmnumber{ #2}\thmnote{ (#3)}}

\theoremstyle{theorem}
\newtheorem{theorem}{Theorem}[section]
\newtheorem{lemma}{Lemma}[section]
\newtheorem{corollary}{Corollary}[section]

\theoremstyle{definition}
\newtheorem{definition}{Definition}[section]

\newtheorem{assumption}{Assumption}[section]

\newtheorem{remark}{Remark}[section]

\newcommand{\Bracks}[1]{\left[#1\right]}

\newcommand{\norm}[1]{\lVert#1\rVert}
\newcommand{\abs}[1]{\lvert#1\rvert}
\newcommand{\Abs}[1]{\left\lvert#1\right\rvert}

\newcommand{\say}[1]{``#1''}
\newcommand{\simplex}{\triangle}

\newcommand{\iden}{\textrm{I}}

\newcommand{\Wcal}{\ensuremath{\mathcal{W}}} 
\newcommand{\Fcal}{\ensuremath{\mathcal{F}}} 
\newcommand{\Scal}{\ensuremath{\mathcal{S}}} 
\newcommand{\Ccal}{\ensuremath{\mathcal{C}}} 
\newcommand{\Acal}{\ensuremath{\mathcal{A}}} 
 
\newcommand{\Dcal}{\ensuremath{\mathcal{D}}}

\newcommand{\Ocal}{\ensuremath{\mathcal{O}}} 
\newcommand{\Hcal}{\ensuremath{\mathcal{H}}} 
\newcommand{\Mcal}{\ensuremath{\mathcal{M}}} 
\newcommand{\Xcal}{\ensuremath{\mathcal{X}}} 
 
\newcommand{\Vcal}{\ensuremath{\mathcal{V}}}
\newcommand{\Lcal}{\ensuremath{\mathcal{L}}}
\newcommand{\Tcal}{\ensuremath{\mathcal{T}}}
\newcommand{\Zcal}{\ensuremath{\mathcal{Z}}}
\newcommand{\epsg}{\ensuremath{\varepsilon_{\textrm{gen}}}}
\newcommand{\epso}{\ensuremath{R}}

\newcommand{\eps}{\ensuremath{\epsilon}}

\DeclareMathOperator{\trace}{trace}
\newcommand{\piest}{\ensuremath{\pi_{\textrm{est}}}}
\newcommand{\conf}{\ensuremath{\textrm{conf}}}
\newcommand{\op}{\textrm{op}}
\newcommand{\conc}{\circ}

\newcommand{\xmark}{\ding{55}}

\newcommand{\jnote}[1]{\textcolor{blue}{[JL: #1]}}

\title{Bilinear Classes: A Structural Framework for Provable Generalization in RL}
\author{
	Simon S. Du\thanks{University of Washington. Email: \texttt{ssdu@cs.washington.edu}}
	\and Sham M. Kakade\thanks{University of Washington and
          Microsoft Research. Email: \texttt{sham@cs.washington.edu}}
	\and Jason D. Lee\thanks{Princeton University. Email: \texttt{jasonlee@princeton.edu}}
	\and Shachar Lovett\thanks{University of California, San Diego. Email: \texttt{slovett@cs.ucsd.edu}}
	\and Gaurav Mahajan\thanks{University of California, San Diego. Email: \texttt{gmahajan@eng.ucsd.edu}}
	\and Wen Sun\thanks{Cornell University. Email: \texttt{ws455@cornell.edu}}
	\and Ruosong Wang\thanks{Carnegie Mellon University. Email:\texttt{ruosongw@andrew.cmu.edu}}
}
\date{}

\begin{document}
\maketitle

\newcommand{\sk}[1]{\textsf{\color{magenta} SK: #1}}
\newcommand{\Wen}[1]{\textsf{\color{red} WS: #1}}
\newcommand{\simon}[1]{\textsf{\color{green} SSD: #1}}
\newcommand{\gm}[1]{\textsf{\color{cyan}{GM: #1}}}
\newcommand{\shachar}[1]{[\textsf{\color{purple}{SL: #1}}]}

\begin{abstract} 
  This work introduces Bilinear Classes, a new structural
  framework, which permit generalization in reinforcement learning in
  a wide variety of settings through the use of function
  approximation.  The framework incorporates nearly all existing
  models in which a polynomial sample complexity is achievable, and,
  notably, also includes new models, such as the Linear $Q^*/V^*$
  model in which both the optimal $Q$-function and the optimal
  $V$-function are linear in some known feature space. Our main result
  provides an RL algorithm which has polynomial sample complexity for
  Bilinear Classes; notably, this sample complexity is stated in terms
  of a reduction to the generalization error of an underlying supervised
  learning sub-problem. These bounds nearly match the best known
  sample complexity bounds for existing models. Furthermore, this
  framework also extends to the infinite dimensional (RKHS) setting:
  for the the Linear $Q^*/V^*$ model, linear MDPs, and linear mixture
  MDPs, we provide sample complexities that have no explicit
  dependence on the explicit feature dimension (which could be
  infinite), but instead depends only on information theoretic
  quantities.
\end{abstract}

\section{Introduction}
Tackling large state-action spaces is a central challenge in reinforcement learning (RL).
Here, function approximation and supervised learning schemes are often
employed for generalization across large state-action spaces. 
While there have been a number of successful
applications \citep{mnih2013playing, kober2013reinforcement,silver2017mastering,wu2017framework}.
there is also a realization that practical RL approaches are quite
sample inefficient.

Theoretically, there is a growing body of results showing how sample
efficiency is possible in RL for particular model classes (often with
restrictions on the model dynamics though in some cases on the class
of value functions), e.g. State
Aggregation \citep{lihong2009disaggregation,dong2020provably}, Linear
MDPs \citep{yang2019sample,jin2019provably}, Linear Mixture
MDPs \citep{modi2019sample,ayoub2020model}, Reactive
POMDPs \citep{krishnamurthy2016pac}, Block
MDPs \citep{du2019provably}, FLAMBE \citep{agarwal2020flambe},
Reactive PSRs \citep{littman2001predictive}, Linear Bellman
Complete \citep{munos2005error,zanette2020learning}.  

More generally,
there are also a few lines of work which propose more general frameworks,
consisting of \emph{structural conditions} which permit sample
efficient RL; these include the low-rankness structure (e.g. the
Bellman rank \citep{jiang2017contextual} and Witness
rank \citep{sun2018model}) or under a complete
condition \citep{munos2005error,zanette2020learning}. The goal in
these latter works is to develop a unified theory of generalization in
RL, analogous to more classical notions of statistical complexity
(e.g. VC-theory and Rademacher complexity) relevant for supervised
learning.  These latter frameworks are not contained in each other
(see Table~\ref{table:comp_framework}), and, furthermore, there are a number
of natural RL models that cannot be incorporated into each of these
frameworks (see Table~\ref{table:comp}).

Motivated by this latter line of work, we aim 
to understand if there are simple and natural structural conditions which capture
the learnability in a general class of RL models. 

\begin{table*}[!t]
\centering\resizebox{\columnwidth}{!}{
\begin{tabular}{ | c | c | c | c | c |  }
\hline
Framework & B-Rank
& B-Complete &
                                                              W-Rank
                                                               & \classname~(this work)  \\ 
\hline
B-Rank  & \checkmark & \xmark& \checkmark &  \checkmark \\
\hline
B-Complete & \xmark  &\checkmark &  \xmark  & \checkmark \\
\hline
W-Rank & \xmark  &\xmark &  \checkmark  & \checkmark \\
\hline
\classname~(this work) & \xmark  &\xmark &  \xmark  & \checkmark \\
\hline
\end{tabular}
}
\caption{Relations between frameworks.
\checkmark: the column framework contains the row framework. \xmark: the column framework does not contains the row framework.
B-Rank: Bellman Rank \citep{jiang2017contextual}, which is defined in terms of the roll-in distribution and the function approximation class for $Q^*$.
B-Complete: Bellman Complete \citep{munos2005error} (\citet{zanette2020learning} proposed a sample efficient algorithm), which assumes the function class is closed under the Bellman operator.
W-Rank: Witness Rank \citep{sun2018model}: a model-based analogue of Bellman Rank.
\classname: our proposed framework.
}
\label{table:comp_framework}
\end{table*}

\begin{table*}[!t]
\centering\resizebox{\columnwidth}{!}{
\begin{tabular}{ | c | c | c | c | c |  }
\hline
& B-Rank
& B-Complete & W-Rank
 & \classname~(this work)  \\ 
\hline
Tabular MDP & \checkmark  & \checkmark  & \checkmark &  \checkmark  \\
\hline
Reactive POMDP \citep{krishnamurthy2016pac}  & \checkmark  & \xmark  & \checkmark &  \checkmark  \\
\hline
Block MDP \citep{du2019provably}  & \checkmark  & \xmark  & \checkmark &  \checkmark  \\
\hline
Flambe / Feature Selection \citep{agarwal2020flambe} & \checkmark & \xmark &  \checkmark & \checkmark \\
\hline
Reactive PSR \citep{littman2002predictive} & \checkmark  &\xmark &  \checkmark  & \checkmark \\
\hline
Linear Bellman Complete \citep{munos2005error} & \xmark  & \checkmark  & \xmark &  \checkmark  \\
\hline
Generalized Linear Bellman Complete \citep{wang2019optimism} & \xmark  & \xmark  & \xmark &  \checkmark  \\
\hline
Linear MDPs \citep{yang2019sample,jin2019provably} & \checkmark!  & \checkmark  & \checkmark! &  \checkmark  \\
\hline
Linear Mixture Model \citep{modi2020sample} & \xmark & \xmark & \xmark &  \checkmark \\
\hline
Linear Quadratic Regulator & \xmark  &\checkmark &  \xmark  & \checkmark \\
\hline
 Kernelized Nonlinear Regulator \citep{kakade2020information} & \xmark & \xmark & \checkmark &  \checkmark \\ 
 \hline
 Factored MDP \citep{kearns1999efficient} & \xmark & \xmark & \xmark &  \checkmark \\ 
 \hline
$Q^\star$ ``irrelevant" State Aggregation \citep{lihong2009disaggregation}  & \checkmark & \xmark & \xmark & \checkmark \\
\hline
\rowcolor{LightCyan}
Linear $Q^\star$/$V^\star$ (this work) &  \xmark &  \xmark& \xmark &  \checkmark  \\
\hline
\rowcolor{LightCyan}
RKHS Linear MDP (this work) &  \xmark &  \xmark& \xmark &  \checkmark  \\
\hline
\rowcolor{LightCyan}
RKHS Linear Mixture MDP (this work) &  \xmark &  \xmark& \xmark &  \checkmark  \\
\hline
\rowcolor{LightCyan}
Low Occupancy Complexity  (this work)&  \xmark &  \xmark& \xmark &  \checkmark  \\
\hhline{| = = = = = |}
\hline
$Q^\star$ State-action Aggregation \citep{dong2020provably}  & \xmark & \xmark & \xmark & \xmark \\
\hline
Deterministic linear $Q^\star$ \citep{wen2013efficient}  &  \xmark & \xmark & \xmark& \xmark \\ 
\hhline{| = = = = = |}
\hline
\rowcolor{LightRed}
Linear $Q^\star$ \citep{weisz2020exponential} & \multicolumn{4}{c|}{Sample efficiency is not possible} \\
\hline
\end{tabular}
}
\caption{
Whether a framework includes a model that permits a sample efficient algorithm.
\checkmark means the framework includes the model, \xmark ~means not, and $\checkmark!$ means the sample complexity using that framework needs to scale with the number of action (which is not necessary).
``Sample efficient is not possible" means the sample complexity needs to scale exponentially with at least one problem parameter.
See \Cref{sec:rel}, \Cref{sec:new_examples}, \Cref{sec:extensions} and \Cref{sec:examples} for detailed descriptions of the models.
}
\label{table:comp}
\end{table*}

\paragraph{Our Contributions.}
This work provides a simple structural condition on the
hypothesis class (which may be either model-based or value-based),
where the Bellman error has a particular bilinear form, under which
sample efficient learning is possible; we refer such a framework as a
Bilinear Class. This structural assumption can be seen as generalizing the
Bellman rank \citep{jiang2017contextual}; furthermore, it not only contains existing
frameworks, it also covers a number of new settings that are not
easily incorporated in previous frameworks (see
Tables~\ref{table:comp_framework} and~\ref{table:comp}). 

Our main result presents an optimization-based algorithm, \algname,
which provably enjoys a polynomial sample
complexity guarantee for Bilinear Classes (cf. Theorem~\ref{thm:main-result}).
Although our framework is more general
than existing ones, our proof is substantially simpler -- we give a
unified analysis based on the elliptical potential
lemma, developed for the theory of linear bandits \citep{dani2008stochastic,srinivas2009gaussian}.

Furthermore, as a point of emphasis, our results are non-parametric in
nature (stated in terms of an information gain
quantity \citep{srinivas2009gaussian}), as opposed to finite dimensional
as in prior work.  From a technical point of view, it is not evident
how to extend prior approaches to this non-parametric
setting. Notably, the non-parametric regime is particularly
relevant to RL due to that, in RL,  performance bounds do \emph{not} degrade gracefully
with approximation error or
model mis-specification (e.g. see \citet{du2019good} for discussion of these issues); the relevance of the
non-parametric regime is that it may provide additional flexibility to
avoid the catastrophic quality degradation due to approximation error or model mis-specification.

A few further notable contributions are:
\begin{itemize}
	\item \emph{Definition of Bilinear Class:} Our key conceptual contribution is the definition of the Bilinear Class, which isolates two key critical properties. The first property is that the Bellman error can be upper bounded by a bilinear form depending on the hypothesis. The second property is that the corresponding bilinear form for all hypothesis in the hypothesis class can be estimated with the same dataset. Analogous to supervised learning, this allows for efficient data reuse to estimate the Bellman error for all hypothesis simultaneously and eliminate those with high error.
    \item \emph{A reduction to supervised learning:} One appealing aspect of
      this framework is that the our main sample complexity result for 
      RL is quantified via a reduction to
      the generalization error of a supervised learning problem, where
      we have a far better understanding of the latter. This is particularly important
      due to that we make no explicit assumptions on the hypothesis
      class $\Hcal$ itself, thus allowing for neural hypothesis
      classes in some cases
      (the \classname~posits an \emph{implicit} relationship between
      $\Hcal$ and the underlying MDP $\Mcal$).
    \item \emph{New models:}  We show our \classname~framework incorporates
      new natural models, that are not easily incorporated into
      existing frameworks, e.g. linear $Q^*/V^*$, Low Occupancy
      Complexity, along with (infinite-dimensional) RKHS versions of linear MDPs and
      linear mixture MDPs.  The linear $Q^*/V^*$ result is particularly notable
      due to a recent and remarkable lower bound which showed
      that if we only assume $Q^*$ is linear in some given set of
      features, then sample efficient learning is information
      theoretically not
      possible \citep{weisz2020exponential}.
    In
      perhaps a surprising contrast, our works shows that if we assume
      that both $Q^\star$ and $V^\star$ are linear in some given
      features then sample efficient learning is in fact possible.\vspace{0pt}
    \item \emph{Non-parametric rates:} Our work is 
      applicable to the non-parametric setting, where we develop new
      analysis tools to handle a number of technical challenges.  This
      is notable as non-parametric rates for RL are
      few and far between. Our results are stated in terms
      of the \emph{critical information gain} which can viewed as
      an analogous quantity to the \emph{critical radius}, a quantity which
      is used to obtain sharp rates in non-parametric statistical settings \citep{wainwright2019high}.
      \item \emph{Flexible Framework:} The \classname~framework is easily modified to include cases that do not strictly fit the definition. We show several examples of this in Section \ref{sec:extensions}, where we show simple modifications of \classname~framework include Witness Rank and Kernelized Nonlinear Regulator. 
\end{itemize}

\paragraph{Organization}
Section~\ref{sec:rel} provides further related work.
Section~\ref{sec:setting} introduce some technical background and notation.
Section~\ref{sec:bilinear_model} introduces our Bilinear Class
framework, where we instantiate it on the several RL models, and 
Section~\ref{sec:alg} describes our algorithm and provides our main
theoretical results. In \Cref{sec:extensions}, we introduce further extensions of Bilinear Classes.
We conclude in Section~\ref{sec:conclusion}.  Appendix~\ref{sec:examples} provides additional examples of the Bilinear Class including the feature selection model~\cite{agarwal2020flambe}, $\mathcal{Q}^\ast$  state aggregation, LQR, Linear MDP, and Block MDP. Appendix~\ref{app:proofs-main-thm} provides missing proofs of Section \ref{sec:alg}. Appendix \ref{sec:elliptic-cover} provides a key technical theorem to attain non-parametric convergence rates in terms of the information gain, and Appendix \ref{app:conc-covering} uses this to show concentration inequalities for all the models in a unified approach. Appendix \ref{sec:monotone} provides proofs for Section \ref{sec:extensions}. Finally, Appendix \ref{sec:lowerbound} shows that low information gain is necessary in both Bellman Complete and Linear MDP by showing that small RKHS norm is not sufficient for sample-efficient reinforcement learning.

\section{Related Work: Frameworks and Models}
\label{sec:rel}

\paragraph{Relations Among Frameworks.}
We first review existing frameworks and the relations among them. 
See Table~\ref{table:comp_framework} for a summary.

\citet{jiang2017contextual} defines a notion, Bellman Rank (B-Rank in Tables), in terms of the roll-in distribution and the function approximation class for $Q^*$, and give an algorithm with a polynomial sample complexity in terms of the Bellman Rank.
They also showed a class of models, including tabular MDP, LQR, Reactive POMDP~\citep{krishnamurthy2016pac}, and Reactive PSR~\citep{littman2002predictive} admit a low Bellman Rank, and thus they can be solved efficiently.
Some recently proposed models, such as Block MDP~\citep{du2019provably}, linear MDP~\citep{yang2019sample,jin2019provably} can also be shown to have a low Bellman rank.
One caveat is that their algorithm requires a finite number of actions, so cannot be directly applied to (infinite-action) linear MDP and LQR.
Subsequently, \citet{sun2018model} proposed a new framework, Witness Rank (W-Rank in tables), which generalizes Bellman Rank to model-based setting. 

Bellman Complete (B-Complete in tables) is a framework of another style, which assumes that the class used for approximating the $Q$-function is closed under the Bellman operator.
As shown in Table~\ref{table:comp_framework},  neither the low-rank-style framework (Bellman Rank and Witness Rank) nor the complete-style framework (B-Complete) contains the other (See e.g., \cite{zanette2020learning}).

Eluder dimension~\citep{russo2014learning} is another structural condition which directly assumes the function class allows for strong extrapolation after observing dimension number of samples. 
With appropriate representation conditions (stronger than Bellman Complete), there is an efficient algorithm for function classes with small eluder dimension~\citep{wang2020reinforcement}. However due to Eluder dimension requiring extrapolation, there are few examples of function classes with small eluder dimension beyond linear functions and monotone transformations of linear functions both of which are captured by the bilinear class.


\paragraph{Comparison to Bellman Eluder}  Concurrently, \cite{jin2021bellman} proposes a new structural model called Bellman Eluder dimension (BE dimension) which takes both the MDP structure and the function class into consideration. \textit{We note that neither BE nor \classname~capture each other.} Notably, Bilinear Classes, via use of flexible Bellman error estimators, naturally captures model-based settings including linear mixture MDPs, KNRs, and factored MDPs, which are hard for model-free algorithms and frameworks to capture since the value functions of these models could be arbitrarily complicated. Specifically, \cite{sun2018model} shows that for factored MDPs, model-free algorithms such as OLIVE \cite{jiang2017contextual} suffer exponential sample complexity in worst case which implies that both BE dimension and Bellman rank are large for factored MDPs. However, Bilinear Class and Witness rank \cite{sun2018model} properly capture the complexity of factored MDPs.  Similar situation may also apply to KNRs. For instance, \cite{dong2020expressivity} showed that for a simple piecewise linear dynamics (thus captured by KNRs) and piecewise reward functions, the optimal policy could contain exponentially many linear pieces and the optimal Q and V functions are fractals which are not differentiable anywhere and cannot be approximated by any neural networks with a polynomial width. It is unclear if such models have low BE dimension.



 The primary difference is that the two complexity measures are applied to different structural aspects of the MDP: Bellman eluder framework is applied to the Bellman error and the bilinear class is applied to any loss estimator of the Bellman error. The actual complexity measures of eluder dimension and information gain are very similar and in fact equivalent for RKHS~\citep{equivalencenote,jin2021bellman}. As these two complexity measures are different in general, an interesting direction for further work is to understand how eluder dimension can address new settings of practical interest beyond (generalized) linear models and whether Bellman eluder dimension can be broadened to capture model-based approaches (like the linear mixture model). Finally, we comment that there are models (e.g., deterministic linear $Q^\star$ and  $Q^\star$ state-action aggregation) that are captured by neither frameworks; we leave to future work to propose a framework that can capture these models that do not have error amplification.

With an additional Bellman completeness assumption on the function class, \cite{jin2021bellman} gives an algorithm which extends Eleanor from \cite{zanette2020learning} to nonlinear function approximation that achieves a regret guarantee with faster rates than our algorithm. We note that our algorithm and OLIVE (as shown by \cite{jin2021bellman}) does not require Bellman completeness which is a much stronger assumption than realizability. As examples, the low occupancy complexity, feature selection model, linear mixture model, and many other model-based models are not Bellman complete. While our work focuses on PAC bounds, we conjecture that the techniques from \cite{dong2020root} can be used for deriving regret bounds without completeness.


\paragraph{Reinforcement Learning Models.}
Now we discuss existing RL models.
A summary on whether a model can be incorporated into a framework is provided in Table~\ref{table:comp}.

Tabular MDP is the most basic model, which has a finite number of states and actions, and all frameworks incorporate this model.
When the state-action space is large, different RL models have been proposed to study when one can generalize across the state-action pairs.

Reactive POMDP~\citep{krishnamurthy2016pac} assumes there is a small number of hidden states and the $Q^*$-function belongs to a pre-specified function class.
Block MDP~\citep{du2019provably} also assumes there is a small number of hidden states and further assumes the hidden states are decodable.
Reactive PSR~\citep{littman2001predictive} considers partial observable systems whose parameters are grounded in observable quantities. 
FLAMBE~\citep{agarwal2020flambe} considers the feature selection and removes the assumption of known feature in linear MDP.  
These models all admit a low-rank structure, and thus can be incorporated into the Bellman Rank or Witness Rank and our Bilinear Classes.

The Linear Bellman Complete model~\citep{munos2005error} uses linear functions to approximate the $Q$-function, and assumes the linear function class is closed under the Bellman operator.
\citet{zanette2020learning} presented a statistically efficient algorithm for this model.
This model does not have a low Bellman Rank or Witness Rank but can be incorporated into the Bellman Complete framework and ours.

Linear MDP~\citep{yang2019sample,jin2019provably} assumes the transition probability and the reward are linear in given features. This model not only admits a low-rank structure, but also satisfies the complete condition. Therefore, this model belongs in all frameworks.
However, when the number of action is infinite, the algorithms for Bellman Rank and Witness Rank are not applicable because their sample complexity scales with the number of actions.
Linear mixture MDP~\citep{modi2019sample,ayoub2020model} assumes the transition probability is a linear mixture of some base models.
This model cannot be included in Bellman Rank, Witness Rank, or Bellman Complete, but our Bilinear Classes includes this model.

LQR is a fundamental model for continuous control that can be efficiently solvable~\citep{dean2019sample}. 
While LQR has a low Bellman Rank and low Witness Rank, since the algorithms for Bellman Rank and Witness Rank scale with the number of actions and LQR's action set is uncountable, these two frameworks cannot incorporate LQR.

There is a line of work on state-action aggregation. 
$Q^*$ ``irrelevance" state aggregation assumes one can aggregate states to a meta-state if these states share the same $Q^*$ value, and the number of meta-states is small~\citep{lihong2009disaggregation,jiang2015abstraction}.
$Q^*$ state-action aggregation aggregates state-action pairs to a meta-state-action pair if these pairs have the same $Q^*$-value~\citep{dong2020provably,lihong2009disaggregation}.

Lastly, when only assuming $Q^*$ is linear, there exists an exponential lower bound~\citep{weisz2020exponential}, but with the additional assumption that the MDP is (nearly) deterministic and has large sub-optimality gap, there exists sample efficient algorithms~\citep{wen2013efficient,du2019dsec,du2020agnostic}.
\section{Setting}
\label{sec:setting}


We denote an episodic finite horizon, non-stationary MDP with horizon
$H$, by $\Mcal = \left\{\Scal,\Acal, r, H, \{P_h\}_{h=0}^{H-1},
  s_0\right\}$, where $\Scal$ is the state space, $\Acal$ is the
action space, $r:\Scal\times\Acal\mapsto [0,1]$ is the expected reward
function with the corresponding random variable $R(s,a)$, $P_h:\Scal\times\Acal\mapsto \simplex(\Scal)$ (where
$\simplex(\Scal)$ denotes the probability simplex over $\Scal$) is the
transition kernel for all $h$, $H \in \mathbb{Z}_+$ is the planning
horizon and $s_0$ is a fixed initial state\footnote{Our results
generalizes to any fixed initial state distribution}. For ease of
exposition, we use the notation $o_h$ for \say{observed transition info at timestep $h$} i.e. $o_h = (r_h, s_h, a_h, s_{h+1})$ where $r_h$ is the observed reward $r_h = R(s_h,a_h)$ and $s_h, a_h, s_{h+1}$ is the observed state transition at timestep $h$.  

A deterministic, stationary policy $\pi:\Scal \mapsto \Acal$ specifies
a decision-making strategy in which the agent chooses actions
adaptively based on the current state, i.e. $a_h\sim \pi(s_h)$. We denote a non-stationary policy 
$\pi = \{\pi_0,\dots, \pi_{H-1}\}$ as a sequence of stationary policies where
$\pi_h: \Scal\mapsto \Acal$.

Given a policy $\pi$ and a state-action pair $(s,a) \in \Scal \times \Acal$, the $Q$-function at time step $h$ is defined as \[
Q_h^\pi(s,a) = \E\left[\sum_{h' = h}^{H-1}R(s_{h'},a_{h'})\mid s_h =s, a_h = a, \pi\right]\, ,
\]
and, similarly, a value function time step $h$ of a given state $s$ under a policy $\pi$ is defined as \[
V_h^\pi(s)=\E\left[\sum_{h' = h}^{H-1}R(s_{h'},a_{h'})\mid s_h =s, \pi\right]\, ,
\]
where both expectations are with respect to $s_0, a_0, \ldots s_{H-1},a_{H-1}\sim d^\pi$.
We use $Q_h^\star$ and $V_h^\star$ to denote the $Q$ and $V$-functions of the optimal policy.

\paragraph{Sample Efficient Algorithms.} Throughout the paper, we will consider an algorithm as sample-efficient, if it uses number of trajectories polynomial in the problem horizon $H$, inherent dimension $d$, accuracy parameter $1/\epsilon$ and poly-logarithmic in the number of candidate value-functions.

\paragraph{Notation.} For any two vectors $x,y$, we denote $[x,y]$ as
the vector that concatenates $x,y$, i.e., $[x,y] := [
x^{\top},y^{\top} ]^{\top}$. For any set $S$, we write $\simplex(S)$
to denote the probability simplex. We often use $U(S)$ as the uniform
distribution over set $S$.  We will let $\Vcal$ denote a Hilbert space
(which we assume is either finite dimensional or separable).  

We let $[H]$ denote the set $\{0, \ldots H-1\}$. We slightly abuse
notation (overloading $d^\pi$ with its marginal distributions), where
$s_h\sim d^\pi, (s_h,a_h)\sim d^\pi, (r_h, s_h, a_h, s_{h+1})\sim d^\pi$ and most frequently $o_h \sim d^\pi$ denotes the marginal distributions at timestep $h$. We also use the
shorthand notation
$s_0, a_0, \ldots s_{H-1},a_{H-1}\sim \pi$, $s_h, a_h\sim \pi$ for
$s_0, a_0, \ldots s_{H-1},a_{H-1}\sim d^\pi$, $s_h, a_h\sim d^\pi$.


\section{Bilinear Classes}
\label{sec:bilinear_model}

Before, we define our structural framework -- Bilinear Class, we first define our hypothesis class.

\paragraph{Hypothesis Classes.} 
We assume access to a hypothesis class $\Hcal = \Hcal_0 \times \ldots\times \Hcal_{H-1}$, which can be abstract sets that permit for both
\emph{model-based and value-based} hypotheses. The only restriction we
make is that for all $f\in\Hcal$, we have an associated state-action
value function $Q_{h,f}$ and a value function $V_{h,f}$. We next provide some examples:\begin{enumerate}
\item An example of \emph{value-based hypothesis class} $\Hcal$ is an explicit set of state-action value $Q$ and value functions $V$ i.e. \begin{align*}
  \Hcal_h \subset \{(Q_h, V_h) \mid~ &Q_h ~\text{is a function from}~ \Scal \times \Acal \mapsto \R~\text{and}\\
  & V_h ~\text{is a function from}~ \Scal \mapsto \R\} \, .
\end{align*} Note that in this case, for any hypothesis
$f:= ((Q_0, V_0), (Q_1, V_1), \ldots, (Q_{H-1}, V_{H-1})) \in \Hcal$, we can take the associated $Q_{h,f} = Q_h$ and associated $V_{h,f} = V_h$. 
\item Another example of \emph{value-based hypothesis class} $\Hcal$ is when $\Hcal$ is just a set of state-action value $Q$ functions i.e. \begin{equation*}
  \Hcal_h \subset \{Q_h \mid Q_h ~\text{is a function from}~ \Scal \times \Acal \mapsto \R\} \, .
\end{equation*} In
this case, for any hypothesis
$f:= (Q_0, Q_1, \ldots, Q_{H-1}) \in \Hcal$, we can take the associated $Q_{h,f} = Q_h$ and the associated $V_{h,f}$ function to be greedy with respect to the $Q_{h,f}$ function i.e. $V_{h,f}(\cdot) = \max_{a\in \Acal} Q_{h,f}(\cdot, a)$.
\item An example of \emph{model-based hypothesis class} is when $\Hcal_h$ is a set of models/transition kernels $P_h$ and reward functions $R_h$ i.e. \begin{align*}
  \Hcal_h \subset \{(P_h, R_h) \mid~ &P_h ~\text{is a function from}~ \Scal \times \Acal \mapsto \simplex(\Scal)~\text{and}\\
  &R_h ~\text{is a function from}~ \Scal \times \Acal \mapsto \simplex(\R)\} \, .
\end{align*} In this case, for any hypothesis
$f:= ((P_0, R_0), (P_1, R_1), \ldots, (P_{H-1}, R_{H-1})) \in \Hcal$, we can take the associated $Q_{h,f}$ and $V_{h,f}$ functions to be the optimal value functions corresponding to the transition kernels $\{P_h\}_{h=0}^{H-1}$ and reward functions $\{R_h\}_{h=0}^{H-1}$.
\end{enumerate}
 Furthermore, we assume the hypothesis class
is constrained so that  $V_{h,f}(s) = \max_a Q_{h,f}(s,a)$ for all
$f\in \Hcal$, $h\in [H]$, and $s\in\Scal$,
which is always possible as we can remove hypothesis for which this is not
true. We let $\pi_{h,f}$ be the greedy policy with respect to $Q_{h,f}$, i.e., $\pi_{h,f}(s) = \argmax_{a\in\Acal} Q_{h,f}(s,a)$, and $\pi_{f}$ as the sequence of time-dependent policies $\{\pi_{h,f}\}_{h=0}^{H-1}$.\\

\subsection{Warmup: Bellman rank, the $Q$ and $V$ versions.}
\label{subsec:bellman}
As a motivation for our structural framework, we next discuss Bellman rank framework considered in \cite{jiang2017contextual}. In this case, the hypothesis class $\Hcal_h$ contains Q value functions, i.e., \begin{equation*}
  \Hcal_h \subset \{Q_h \mid Q_h ~\text{is a function from}~ \Scal \times \Acal \mapsto [0, H]\} \, .
\end{equation*} In
this case, for any hypothesis
$f:= (Q_0, Q_1, \ldots, Q_{H-1}) \in \Hcal$, we take the associated state-action value function $Q_{h,f} = Q_h$ and the associated state value $V_{h,f}$ function to be greedy with respect to the $Q_{h,f}$ function i.e. $V_{h,f}(\cdot) = \max_{a\in \Acal} Q_{h,f}(\cdot, a)$. 

\begin{definition}[$V$-Bellman Rank]
  A MDP has a \emph{$V$-Bellman rank} of dimension $d$ if for all $h\in [H]$, there exist functions $W_h: \Hcal \to \R^d$
  and $X_h: \Hcal \to \R^d$, such that for all $f,g \in \Hcal$:
  \begin{align*}
   \E_{a_{0:h-1} \sim d^{\pi_f}, a_h = \pi_{g}(s_h) }&\big[ V_{h,g}(s_h) - r(s_h, a_h)-   \E\left[ V_{h+1,g}(s_{h+1}) | s_h,a_h \right] \big]\\& = 
  \langle W_h(g) - W_h(f^\star), X_h(f)\rangle .
  \end{align*}
  \end{definition}

Even though \cite{jiang2017contextual} only considered $V$-Bellman Rank, as a natural extension of this definition, we can also consider the $Q$-Bellman Rank.
 
\begin{definition}[$Q$-Bellman Rank]
\label{def:bellman_rank}
For a given MDP $\mathcal{M}$, we say that our state-action value hypothesis class
$\Hcal$ has a  \emph{$Q$-Bellman rank} of dimension $d$ if for all $h\in[H]$, there exist functions $W_h: \Hcal \to \R^d$
and $X_h: \Hcal \to \R^d$, such that for all  $f,g \in \Hcal$
\begin{align*}
\E_{a_{0:h} \sim d^{\pi_f}}&\big[ Q_{h,g}(s_h,a_h) -  r(s_h,a_h) - V_{h+1,g}(s_{h+1}) \big] =\langle W_h(g) - W_h(f^\star), X_h(f)\rangle.
\end{align*}
\end{definition}

Let us interpret how the two definitions differ in the usage of functions $V_{h,f}$ vs $Q_{h,f}$ (along with the usage
of the \say{estimation} policies $a_{0: h} \sim \pi_f$ vs $a_{0:h-1}\sim
\pi_f$ and $a_h \sim \pi_g$). Recall that the Bellman equations can be
written in terms of the value functions or the state-action values;
here, the intuition is that the former
definition corresponds to enforcing Bellman consistency of the value functions
while the latter definition corresponds to enforcing Bellman consistency of
the state-action value functions.  
 Our more general structural framework,
Bilinear Classes, will cover both these definitions for infinite
dimensional hypothesis class (note that \cite{jiang2017contextual}
only considered finite dimensional hypothesis class).

\subsection{Bilinear Classes}
We now introduce a new structural framework -- the Bilinear Class.

\paragraph{Realizability.} We say that $\Hcal$ is \emph{realizable} for an MDP
$\mathcal{M}$ if, for all $h\in[H]$, there
exists a hypothesis $f^\star \in \Hcal$ such that
$Q_h^\star(s,a)  = Q_{h,f^\star}(s,a)$, where $Q_h^\star$ is the optimal
state-action value at time step $h$ in the ground truth MDP
$\mathcal{M}$. For instance, for the model-based perspective, the
realizability assumption is implied if the ground truth transition $P$
belongs to our hypothesis class $\Hcal$.

Now we are ready to introduce the \classname. 

\begin{definition}[\classname]
\label{def:linear_regret}
Consider an MDP $\Mcal$, a hypothesis class $\Hcal$,
a discrepancy function
$\ell_f:(\R \times \Scal\times\Acal\times\Scal)\times\Hcal\rightarrow \R$ (defined
for each $f\in\Hcal$), and a set of estimation
policies $\Pi_{\mathrm{est}}=\{\piest(f):f\in\Hcal\}$. We say $(\Hcal, \ell_f, \Pi_{\mathrm{est}},\Mcal)$ is (\emph{implicitly}) a
\emph{Bilinear Class} if 
$\Hcal$ is realizable in $\Mcal$ and if 
there exist functions $W_h: \Hcal \to \Vcal$
and $X_h: \Hcal \to \Vcal$ for some Hilbert space $\Vcal$, such that
the following two properties hold for all $f\in \Hcal$ and $h \in [H]$:
\begin{enumerate} 
  \item 
We have:
  \begin{align}
   \Abs{\E_{a_{0:h} \sim \pi_f}\big[ Q_{h,f}(s_h,a_h) -  r(s_h,a_h) - V_{h+1,f}(s_{h+1})\big]} \leq 
  \Abs{\langle W_h(f) - W_h(f^\star), X_h(f)\rangle}     \label{eq:assume1}
  \end{align}
  \item 
    The policy $\piest(f)$
   and discrepancy measure $\ell_{f}(o_h, g)$ can be used for
   estimation in the following sense: for any $g\in \Hcal$, we have that (here $o_h = (r_h, s_h, a_h, s_{h+1})$ is the \say{observed transition info})
  \begin{align}
  \Abs{\E_{a_{0:h-1} \sim \pi_f} \E_{a_h \sim \piest(f)} \big[
  \ell_{f}(o_h, g) \big]} = \Abs{\langle W_h(g) - W_h(f^\star), X_h(f)\rangle}.
  \label{eq:assume2}
  \end{align} Typically, $\pi_{\textrm{est}}(f)$ will be either the
  uniform distribution on $\mathcal{A}$ or $\pi_f$ itself; in the
  latter case, we refer to the estimation strategy as being on-policy.
\end{enumerate} 
We also define $\Xcal_h := \{X_h(f)\colon f\in \Hcal\}$ and   $\Xcal := \{\Xcal_h: h \in [H]\}$. 
\end{definition} 

We emphasize the above definition only assumes the existence of $W$
and $X$ functions. Particularly, our algorithm only uses the
discrepancy function $\ell_f$, and does not need to know $W$ or
$X$. A typical example of discrepancy function $\ell_f(o_h, g)$ would be the bellman error $Q_{h,g}(s_h,a_h) -  r_h - V_{h+1,g}(s_{h+1})$, but we would often need to use a different discrepancy function see for e.g. Linear Mixture Models (\Cref{subsec:linearmm}). 

  We now provide some intuition for definition of Bilinear Class. The first part of the definition (\Cref{eq:assume1}) basically relates the Bellman error for hypothesis $f$ (and hence sub-optimality) to the sum of bilinear forms $\Abs{\langle W_h(f) - W_h(f^\star), X_h(f)\rangle}$ (see for example proof of \Cref{lemma:opt}). Crucially, the second part of the definition (\Cref{eq:assume2}), allows us to \say{reuse} data from hypothesis $f$ to estimate the bilinear form $\Abs{\langle W_h(g) - W_h(f^\star), X_h(f)\rangle}$ for \emph{all} hypothesis $g$ in our hypothesis class! This is reminiscent of uniform convergence guarantees in supervised learning, where data can be reused to simultaneously estimate the loss for all hypothesis and eliminate those with high loss.
\subsubsection{Finite Bellman rank $\implies$ Bilinear Class }
Here we show our framework naturally generalizes the Bellman rank framework (\Cref{subsec:bellman}). For $Q$-bellman rank case, we define the discrepancy function $\ell_f$ for observed transition info $o_h = (r_h, s_h,a_h,s_{h+1})$ as:
\begin{align*}
& \ell_f(o_h, g) = Q_{h,g}(s_h,a_h) -  r_h - V_{h+1,g}(s_{h+1}).
\end{align*}
\begin{lemma}[Finite $Q$-Bellman Rank $\implies$ Bilinear Class]
  For given MDP $\Mcal$, suppose our hypothesis class $\Hcal$ has a $Q$-Bellman rank of dimension $d$. Then, for on-policy estimation policies $\pi_{est} = \pi_f$, and the discrepancy function $\ell_f$ defined above, $(\Hcal, \ell_f, \Pi_{\mathrm{est}},\Mcal)$ is (\emph{implicitly}) a \emph{Bilinear Class}.
\end{lemma}
\begin{proof}
  Its straightforward to see that in this case, both \Cref{eq:assume1} and \Cref{eq:assume2} are satisfied.
\end{proof}
In the $V$-Bellman rank setting, we define the discrepancy function $\ell_f$ for observed transition info $o_h = (r_h, s_h,a_h,s_{h+1})$  as:
\begin{align*}
&\ell_{f}(o_h, g)=  \frac{\mathbf{1}\{a_h = \pi_g(s_h)\}}{1/|\Acal|} \left( V_{h,g}(s_h) - r_h - V_{h+1,g}(s_{h+1}) \right).
\end{align*}
\begin{lemma}[Finite $V$-Bellman Rank $\implies$ Bilinear Class]
  For given MDP $\Mcal$, suppose our hypothesis class $\Hcal$ has a $V$-Bellman rank of dimension $d$. Then, for uniform estimation policies $\pi_{est} = U(\Acal)$, and the discrepancy function $\ell_f$ defined above, $(\Hcal, \ell_f, \Pi_{\mathrm{est}},\Mcal)$ is (\emph{implicitly}) a \emph{Bilinear Class}.
\end{lemma}
\begin{proof}
  Note that for $g = f$, we have that for observed transition info $o_h = (r_h, s_h,a_h,s_{h+1})$ \[
  \E_{s_h\sim d^{\pi_f}} \E_{a_h \sim U(\Acal)} \left[ \ell(o_h, f)  \right] =  \E_{s_h,a_h,s_{h+1}\sim d^{\pi_f}} \left[ Q_{h,f}(s_h, a_h) - r(s_h,a_h) - V_{h+1,f}(s_{h+1})  \right]
\] Therefore, to prove that this is a Bilinear Class, we will show that a stronger \say{equality} version of \Cref{eq:assume2} holds (which will also prove \Cref{eq:assume1} holds). Observe that for any $h$, \begin{align*}
  &\E_{s_h \sim d^{\pi_f}} \E_{a_h \sim U(\Acal)} \left[ \ell_f(o_h, g)  \right]\\
  &= \E_{s_h \sim d^{\pi_f}}\left[ Q_{h,g}(s_h, \pi_g(s_h)) -   r(s_h, \pi_g(s_h)) -  \E\left[ V_{h+1,g}(s_{h+1}) | s_h,\pi_g(s_h) \right]     \right]  \\
  & = \left\langle  W_h(g) - W_h(f^\star), X_h(f)\right\rangle 
  \end{align*}
  This completes the proof.
\end{proof}

\subsection{Examples}
\label{sec:new_examples}
We now provide examples of Bilinear Classes: two known
models (Linear Bellman Complete and Linear Mixture Models) and two new models
that we propose (Linear $Q^\star/ V^\star$ and Low Occupancy Complexity). We return to these examples to give non-parametric
sample complexities in \Cref{sec:rates-4ex}. See \Cref{sec:examples} for additional examples of Bilinear Classes.  

\subsubsection{Linear Mixture MDP.}
\label{subsec:linearmm}
First, we show our definition naturally captures model-based hypothesis class.
  \begin{definition}[Linear Mixture Model]
    \label{def:linearmm}
  We say that a MDP $\mathcal{M}$ is a \emph{Linear Mixture Model} if  there exists (known) features $\phi: \Scal \times \Acal \times \Scal \mapsto \Vcal$ and $\psi: \Scal\times \Acal\mapsto \Vcal$; and (unknown) $\theta^\star \in \Vcal$ for some Hilbert space $\Vcal$ such that for all $h\in [H]$ and $(s,a,s')\in \Scal\times \Acal \times \Scal$ \begin{align*}
      P_h(s'\mid s,a) = \langle\theta_h^\star, ~\phi(s,a,s')\rangle\quad \text{and} \quad r(s,a) = \langle \theta_h^\star, ~\psi(s,a)\rangle .
  \end{align*}
  \end{definition}
  We denote hypothesis in our hypothesis class $\Hcal$ as tuples $(\theta_0,\ldots\theta_{H-1})$, where $\theta_h\in\Vcal$. Recall that given a model $f\in\Hcal$ (i.e. $f$ is the time-dependent transitions, i.e., $f_h: \Scal\times\Acal\mapsto \Delta(\Scal)$), we denote $V_{h,f}$ as the optimal value function under model $f$ and corresponding reward function (in this case defined by $\psi$). Specifically, for any hypothesis $g = \{\theta_0,\dots, \theta_{H-1}\} \in \Hcal$, $V_{h,g}$ and $Q_{h,g}$ satisfy the following Bellman optimality equation: \begin{equation}
    \label{eq:bellman-mix}
    Q_{h,g} (s_h, a_h) = {\theta_h^{\top}  \bigg(\psi(s_h,a_h) + \sum_{\bar s\in\Scal}  \phi(s_h,a_h,\bar s) V_{h+1,g}(\bar s)\bigg)}
  \end{equation} Note that in this example, discrepancy function will explicitly depend on $f$. For hypothesis $g = \{\theta_0,\dots, \theta_{H-1}\} \in \Hcal$ and observed transition info $o_h = (r_h, s_h,a_h,s_{h+1})$, we define
\begin{align*}
    \ell_f(o_h, g) 
    &= {\theta_h^{\top}  \bigg(\psi(s_h,a_h) + \sum_{\bar s\in\Scal}  \phi(s_h,a_h,\bar s) V_{h+1,f}(\bar s)\bigg)}  - {\bigg(V_{h+1,f}(s_{h+1}) + r_h\bigg)}.
    \end{align*} 
    \begin{lemma}[Linear Mixture Model $\implies$ Bilinear Class]
      Consider a MDP $\Mcal$ which is a Linear Mixture Model. Then, for the hypothesis class $\Hcal$, discrepancy function $\ell_f$ defined above and on-policy estimation policies $\piest(f) = \pi_f$, $(\Hcal, \ell_f, \Pi_{\mathrm{est}},\Mcal)$ is (\emph{implicitly}) a \emph{Bilinear Class}.
    \end{lemma}
    \begin{proof}
      Observe that for $g = f$, using \Cref{eq:bellman-mix}, for observed transition info $o_h = (r_h, s_h,a_h,s_{h+1})$, \[
        \ell_f(o_h, f) = Q_{h,f}(s_h,a_h) -  r_h - V_{h+1,f}(s_{h+1})\, .\] and therefore \[
          \E_{o_h\sim d^{\pi_f}} \Big[ \ell_f(o_h, f)\Big] = \E_{a_{0:h}\sim \pi_f} \Big[Q_{h,f}(s_h,a_h) -  r(s_h, a_h) - V_{h+1,f}(s_{h+1}) \Big]\, .
        \] We consider on-policy estimation $\pi_{est} = \pi_f$. To prove that linear mixture MDP is a \classname, we only need to show that an \say{equality} version of \Cref{eq:assume2} holds (which implies \Cref{eq:assume1} holds by the frame above). For $g = \{\theta_0,\dots, \theta_{H-1}\} \in \Hcal$, observe: \begin{align*}
  &\E_{o_h\sim d^{\pi_f}} \Big[\ell_f\left( o_h, g \right)\Big] \\
  &= \E_{s_h, a_h\sim d^{\pi_f}} \left[ {\theta_h^{\top}  \Big(\psi(s_h,a_h) + \sum_{\bar s\in\Scal}  \phi(s_h,a_h,\bar s) V_{h+1,f}(\bar s)\Big)} - \E_{s_{h+1} \sim P_h(s_h,a_h)}\Big[V_{h+1,f}(s_{h+1}) + r_h\Big]\right].\\
  &= \E_{s_h, a_h\sim d^{\pi_f}} \left[ (\theta_h - \theta_h^\star)^{\top}  \left( \psi(s_h,a_h) + \sum_{\bar s\in\Scal} \phi(s_h,a_h,\bar s) V_{h+1,f}(\bar s)\right)  \right] \\
  &= \left\langle  W_h(g) - W_h(f^\star), X_h(f)\right\rangle
    \end{align*} where we defined the $W_h, X_h$ functions as follows:
    \begin{align*}
    &W_h(g) = \theta_h,\\
    &X_h(f) = \E_{s_h, a_h\sim d^{\pi_f}} \left[ \psi(s_h,a_h) + \sum_{\bar s\in\Scal} \phi(s_h,a_h,\bar s) V_{h+1,f}(\bar s)\right].
    \end{align*} This concludes that Linear Mixture Model also forms a Bilinear Class.
    \end{proof}

\subsubsection{Linear $Q^\star/V^\star$ (new model)}
\label{sec:linearqv}
We introduce a new model: \emph{linear $Q^\star/V^\star$} where we assume both the optimal $Q^\star$ and $V^\star$ are linear functions in features that lie in (possibly infinite dimensional) Hilbert space. 
\begin{definition}[Linear $Q^\star/V^\star$]
  \label{def:linearqv}
  We say that a MDP $\mathcal{M}$ is a \emph{linear $Q^\star/V^\star$} model if  there exist (known) features $\phi: \Scal \times \Acal \mapsto \Vcal_1$, $\psi: \Scal \mapsto \Vcal_2$ and (unknown) $(w^\star, \theta^\star) \in \Vcal_1 \times \Vcal_2$ for some Hilbert spaces $\Vcal_1, \Vcal_2$ such that for all $h\in [H]$ and for all $(s,a,s')\in \Scal\times \Acal \times \Scal$, \[
      Q_h^\star(s,a) = \langle w_h^\star, ~\phi(s,a)\rangle \quad \text{and} \quad V_h^\star(s') = \langle \theta_h^\star, ~ \psi(s')\rangle\, .
  \]
  \end{definition}
Here, our hypothesis class $\Hcal = \Hcal_0 \times \ldots, \Hcal_{H-1}$ is a set of linear functions i.e. for all $h\in [H]$, the set $\Hcal_h$ is defined as: \begin{align*}
  \Big\{ (w, \theta) \in \Vcal_1 \times \Vcal_2\colon\max_{a\in \Acal} w^{\top} \phi(s,a)  = \theta^{\top}\psi(s)\, , ~\forall s \in \Scal \Big\}.
\end{align*}
We define the following discrepancy function
$\ell_f$ (in this case the discrepancy function does not depend on
$f$), for hypothesis $g = \{( w_h, \theta_h )\}_{h=0}^{H-1}$ and observed transition info $o_h = (r_h, s_h,a_h,s_{h+1})$:
\begin{align*}
  \ell_f(o_h,g)  &= Q_{h,g}(s_h,a_h) -  r_h - V_{h+1,g}(s_{h+1})\\
  &= w_h^{\top}\phi(s_h,a_h) - r_h - \theta_{h+1}^{\top}\psi(s_{h+1})\, .
\end{align*}  

\begin{lemma}[Linear $Q^\star/V^\star$ $\implies$ Bilinear Class]
  Consider a MDP $\Mcal$ which is a linear $Q^\star / V^\star$ model. Then, for the hypothesis class $\Hcal$, the discrepancy function $\ell_f$ defined above and on-policy estimation policies $\piest(f) = \pi_f$, $(\Hcal, \ell_f, \Pi_{\mathrm{est}},\Mcal)$ is (\emph{implicitly}) a \emph{Bilinear Class}.
\end{lemma}
\begin{proof}
  Note that we will show that a stronger \say{equality} version of \Cref{eq:assume2} holds, which will also prove \Cref{eq:assume1} holds since for observed transition info $o_h = (r_h, s_h,a_h,s_{h+1})$, \[
    \E_{o_h\sim d^{\pi_f}} \Big[ \ell_f(o_h, f)\Big] = \E_{a_{0:h}\sim \pi_f} \Big[Q_{h,f}(s_h,a_h) -  r(s_h, a_h) - V_{h+1,f}(s_{h+1}) \Big]\, .
  \] Observe that for any $h$ \begin{align*}
  &\E_{o_h\sim d^{\pi_f}} \left[ \ell(o_h, g)  \right] \\
  & = \E_{s_h,a_h,s_{h+1}\sim d^{\pi_f}}\Big[  w_h^{\top}\phi(s_h,a_h) - \theta_{h+1}^{\top}\psi(s_{h+1})  - Q_h^\star(s_h,a_h) + V_{h+1}^\star(s_{h+1}) \Big]\\
  &=\left\langle  W_h(g) - W_h(f^\star), X_h(f)\right\rangle
  \end{align*} 
  where \begin{align*}
    W_h(g) &= [w_h, \theta_{h+1}],\\
    X_h(f) &= \mathbb{E}_{s_h,a_h\sim d^{\pi_f},s_{h+1}\sim P_h(s_h,a_h)}\left[ \phi(s_h,a_h), \psi(s_{h+1}) \right]\, . 
  \end{align*}
This concludes the proof. 
\end{proof}

\subsubsection{Bellman Complete and Linear MDPs}
We now consider Bellman Complete which captures the linear MDP model (see \Cref{sec:linearmdp} for more detail on linear MDP model).  Here, our hypothesis class $\Hcal$ is set of linear functions with respect to
some (known) feature $\phi: \Scal \times \Acal \mapsto \Vcal$, where
$\Vcal$ is a Hilbert space. We denote hypothesis in our hypothesis class $\Hcal$ as tuples $(\theta_0,\ldots\theta_{H-1})$, where $\theta_h\in\Vcal$.
\begin{definition}[Linear Bellman Complete]\label{def:linear_comp}
  We say our hypothesis class $\Hcal$ is \emph{Linear Bellman Complete} with respect to
$\mathcal{M}$ if $\Hcal$ is realizable and there exists $\Tcal_h:
\Vcal \rightarrow \Vcal$ such that
for all
$(\theta_0,\ldots\theta_{H-1}) \in \Hcal$ and $h\in[H]$, 
\[
\Tcal_h(\theta_{h+1})^{\top}\phi(s,a) = r(s,a) + \E_{s'\sim P_h(s,a)}\max_{a'\in\Acal} \theta_{h+1}^{\top} \phi(s',a')  .  
\] 
for all $(s,a) \in \Scal \times \Acal$.
\end{definition}

We define the following discrepancy function
$\ell_f$ (in this case the discrepancy function does not depend on
$f$), for hypothesis $g = (\theta_0,\dots, \theta_{H-1})$ and observed transition info $o_h = (r_h, s_h,a_h,s_{h+1})$:
\begin{align*}
  \ell_f(o_h,g)  &= Q_{h,g}(s_h,a_h) -  r_h - V_{h+1,g}(s_{h+1})\\
  &= \theta_h^{\top}\phi(s_h,a_h) - r_h - \max_{a'\in\Acal} \theta_{h+1}^{\top} \phi(s_{h+1},a')\, .
\end{align*}  
\begin{lemma}[Linear Bellman Complete $\implies$ Bilinear Class]
  Consider an MDP $\Mcal$ and hypothesis class $\Hcal$ such that $\Hcal$ is Linear Bellman Complete with respect to $\Mcal$. Then, for on-policy estimation policies $\piest(f) = \pi_f$ and the discrepancy function $\ell_f$ defined above, $(\Hcal, \ell_f, \Pi_{\mathrm{est}},\Mcal)$ is (\emph{implicitly}) a \emph{Bilinear Class}.
\end{lemma}
\begin{proof}
  Note that in this case, we will show that a stronger version of \Cref{eq:assume2} holds i.e with equality instead of $\leq$ inequality, which will also prove \Cref{eq:assume1} holds since for observed transition info $o_h = (r_h, s_h,a_h,s_{h+1})$, \[
    \E_{o_h\sim d^{\pi_f}} \Big[ \ell_f(o_h, f)\Big] = \E_{a_{0:h}\sim \pi_f} \Big[Q_{h,f}(s_h,a_h) -  r(s_h, a_h) - V_{h+1,f}(s_{h+1}) \Big]\, .
  \] Observe that for any $h$ \begin{align*}
    \mathbb{E}_{o_h\sim d^{\pi_f}} \left[ \ell(o_h, g)  \right] & = \E_{s_h,a_h\sim d^{\pi_f}} \left[  \theta_h^{\top} \phi(s_h,a_h) - \mathcal{T}_h(\theta_{h+1})^{\top} \phi(s_h,a_h)  \right]\\
    & = \left\langle  W_h(g) - W_h(f^\star), X_h(f)\right\rangle
    \end{align*} 
    where \begin{align*}
      W_h(g) &= \theta_h - \mathcal{T}_h(\theta_{h+1})\\
      X_h(f) &= \E_{s_h,a_h\sim d^{\pi_f}} [\phi(s_h,a_h)] .
    \end{align*}
    Observe that $W_h(f^\star) = 0$ for all $h$. 
\end{proof}

\subsubsection{ Low Occupancy Complexity (new model). }
We introduce another new model: \emph{Low Occupancy Complexity}.
\begin{definition}
[Low Occupancy Complexity] \label{def:low_op}
We say that a MDP $\Mcal$ and hypothesis class $\Hcal$ has \emph{low occupancy complexity} with respect to a
(possibly unknown) feature mapping $\phi_h:\Scal\times\Acal\rightarrow
\Vcal$ (where $\Vcal$ is a Hilbert space) if $\Hcal$ is realizable and there exists a (possibly unknown) $\beta_h: \Hcal \mapsto \Vcal$ for $h \in [H]$ such that for all
$f \in \Hcal$ and $(s_h,a_h)\in\Scal\times\Acal$ we have that: 
\[
d^{\pi_f}(s_h,a_h) = \langle \beta_h(f), \phi_h(s_h,a_h) \rangle.
\]
\end{definition}

It is important to emphasize that for this hypothesis class, we are
only assuming realizability, but it is otherwise arbitrary (e.g. it
could be a neural state-action value class) and the algorithm does not need to know the features $\phi_h$ nor $\beta_h$.
It is straight forward to see that such a class is Bilinear Class with
discrepancy function $\ell_f$ defined for hypothesis $g\in \Hcal$ and observed transition info $o_h = (r_h, s_h,a_h,s_{h+1})$ as, 
\begin{align*}
  \ell_f(o_h,g) &=Q_{h,g}(s_h,a_h) -  r_h - V_{h+1,g}(s_{h+1})
\end{align*}

\begin{lemma}[Low Occupancy Complexity $\implies$ Bilinear Class]
  Consider a MDP $\Mcal$ and hypothesis class $\Hcal$ which has  low occupancy complexity. Then, for the the discrepancy function $\ell_f$ defined above and on-policy estimation policies $\piest(f) = \pi_f$, $(\Hcal, \ell_f, \Pi_{\mathrm{est}},\Mcal)$ is (\emph{implicitly}) a \emph{Bilinear Class}.
\end{lemma}\begin{proof}
  To see why this is a \classname, as in previous proofs, we will show that an \say{equality} version of \Cref{eq:assume2} holds, which will also prove \Cref{eq:assume1} holds since \[
    \E_{o_h\sim d^{\pi_f}} \Big[ \ell_f(o_h, f)\Big] = \E_{a_{0:h}\sim \pi_f} \Big[Q_{h,f}(s_h,a_h) -  r(s_h, a_h) - V_{h+1,f}(s_{h+1}) \Big]\, .
  \] Observe that for any $h$ (here observed transition info $o_h = (r_h, s_h,a_h,s_{h+1})$):
\begin{align*}
&\E_{o_h\sim d^{\pi_f}}\big[ \ell_f(o_h,g)\big]\\
&=\sum_{(s_h,a_h) \in \Scal \times \Acal} d^{\pi_f}(s_h,a_h) \big( Q_{h,g}(s_h,a_h) -  r(s_h,a_h) - \E[V_{h+1,g}(s_{h+1})| s_h, a_h]\big)\\
&=\Big\langle \beta_h(f), \sum_{(s_h,a_h) \in \Scal \times \Acal}\phi_h(s_h,a_h) \big( Q_{h,g}(s_h,a_h) -  r(s_h,a_h) - \E[V_{h+1,g}(s_{h+1})|s_h,a_h]\big)\Big\rangle\\
&=  \langle W_h(g) - W_h(f^\star), X_h(f)\rangle
\end{align*}
where the notation $\E[ V(s_{h+1}) | s_h,a_h]$ is shorthand for $\E_{s_{h+1}\sim P_h(s_h,a_h)}[ V(s_{h+1})]$ and we defined the $W_h, X_h$ functions as follows: \begin{align*}
  & X_h(f) := \beta_h(f),\\
  &W_h(g):=\sum_{(s,a) \in \Scal \times \Acal} \phi_h(s,a) \big( Q_{h,g}(s,a) -  r(s,a) - \E_{s'\sim P_h(s,a)}[V_{h+1,g}(s')]\big).
  \end{align*} Note that $W_h(f^\star) = 0$. This completes the proof.
\end{proof}

Note that as such the hypothesis class $\Hcal$ could be arbitrary and unlike other models where we assume linearity, here it
could be a neural state-action value class. Our model can also capture the setting where the state-only occupancy has low complexity, i.e., $d^{\pi_f}(s_h) = \beta_h(f) \mu_h(s_h)$, for some $\mu_h:\Scal\to \Vcal$. In this case, we will use $\pi_{est} = U(\Acal)$. 


\section{The Algorithm and Theory}
\label{sec:alg}


Our algorithm, BiLin-UCB, is described in \Cref{alg:main}, which takes three
parameters as inputs, the number of iterations $T$, the
trajectory batch size $m$ per iteration and a confidence radius
$\epso$.  The key component of the algorithm is a constrained
optimization in Line~\ref{line:constraint_opt}.  For each time step
$h$, we use all previously collected data to form a single constraint
using $\ell_{f}$. The constraint refines the original version space
$\Hcal$ to be a restricted version space containing only 
hypothesis that are consistent with the current batch data. We then
perform an optimistic optimization: we search for a feasible
hypothesis $g$ that achieves the maximum total
reward $V_{g}(s_0)$.   
\begin{algorithm}[!t]
	\caption{\algname}
	\label{alg:general}
	\begin{algorithmic}[1]
		\STATE \textbf{Input}: number of iterations $T$, estimator function $\ell$, batch size $m$, confidence radius $\epso$
		\FOR{ iteration $t = 0, 1,2,\ldots, T-1$ }
		\STATE Set $f_t$ as the solution of the following program: \label{line:constraint_opt} 
		\begin{align*} 
			&\argmax_{g \in \Hcal} V_{g}(s_{0}) ~\text{subject to}~\\
			&\sum_{i=0}^{t-1} (\Lcal_{\Dcal_{i;h}, f_i}(g))^2\leq \epso^2 \quad\forall h \in [H]
		\end{align*}
		\STATE For all $h\in[H]$, create batch datasets $\Dcal_{t;h} = \{(r_h^i, s_h^i,a_h^i, {s}_{h+1}^i)\}_{i=0}^{m-1}$ sampled from distribution induced by $a_{0:h-1}\sim d^{\pi_{f_t}}$ and $a_h\sim \pi_{est}$. \label{line:data_collect}
		\ENDFOR
		\STATE  {\bfseries return} $\max_{t\in [T]} V^{\pi_{f_t}}$.
	\end{algorithmic}
	\label{alg:main}
\end{algorithm}

There are two ways to collect batch samples. For the case where $\pi_{est} = \pi_{f_t}$, then for
data collection in Line~\ref{line:data_collect}, we can generate $m$
length-H trajectories by executing $\pi_{f_t}$ starting from $s_0$.
For the general case (e.g. consider setting $\pi_{est}$ to be a uniform
distribution over $\Acal$), we gather the data for each
$h\in [H]$ independently. For $h\in [H]$, we first roll-in with
$\pi_{f_t}$ to generate $s_h$; then execute $a_h\sim \pi_{est}$; and
then continue to generate $s_{h+1}\sim P_h(\cdot | s_h,a_h)$ and $r_h \sim R(\cdot | s_h, a_h)$. Repeating
this process for all $h$, we need $H m$ trajectories to
form the batch datasets $\{\Dcal_{t;h}\}_{h=0}^{H-1}$.

\subsection{Main Theory: Generalization in Bilinear Classes}

We now present our main result. We first define some notations. We denote the expectation of the function $\ell_f(\cdot, g)$ under distribution $\mu$ over $\R \times \Scal \times \Acal \times \Scal$ by \begin{equation*}
  \Lcal_{\mu,f}(g) = \E_{o\sim \mu} [\ell_{f}(o, g)]
\end{equation*} For a set $\Dcal \subset \Scal \times \Acal \times
\Scal$, we will also use $\Dcal$ to represent the uniform
distribution over this set.

\begin{assumption}[Ability to Generalize]
  \label{assume:linear_regret} We assume there exists functions $\epsg(m,\Hcal)$ and $\conf(\delta)$ such that for any distribution $\mu$ over $ \R \times \Scal\times \Acal \times \Scal$ and for any $\delta \in (0,1/2)$, with probability of at least $1-\delta$ over choice of an i.i.d. sample $\Dcal\sim \mu^m$ of size $m$, 
  \[
  \sup_{g\in\Hcal} \Abs{\Lcal_{\Dcal,f}(g) - \Lcal_{\mu,f}(g) }\leq  \epsg(m, \Hcal) \cdot \conf(\delta)
   \]
 \end{assumption}
 \begin{remark}
  \label{remark:epsg}
 It is helpful to separate the dependence of
 generalization error on failure probability $\delta$ and number of
 samples $m$ in order to state \Cref{thm:main-result}
 concisely. $\epsg(m, \Hcal)$ is related to uniform convergence and
 measures the generalization error of hypothesis class $\Hcal$ and for
 the hypothesis classes discussed in this paper, $\epsg(m, \Hcal) \to
 0$ as $m \to \infty$. One example is when $\pi_{est} = \pi_f$, and
 $\Hcal$ is a discrete function class, then we have $\epsg(m,\Hcal) =
 O\left( \sqrt{ (1 + \ln(|\Hcal|))/m }.   \right)$. In \Cref{app:conc-covering}, we
 also discuss uniform convergence via a novel covering argument for infinite
 dimensional RKHS.  
 \end{remark}
Recall the definitions $\Xcal_h := \{X_h(f)\colon f\in \Hcal\}$ and
$\Xcal := \{\Xcal_h: h \in [H]\}$. We first present our main theorem
for the finite dimensional case i.e. when $\Xcal_h \subset \R^d$ for
all timesteps $h$.
\begin{theorem}
  \label{thm:main-result-finite}
(Finite-dimensional case)
  Suppose $(\Hcal,\ell,\Pi_{\mathrm{est}}, \Mcal)$
 is a Bilinear Class with $\Xcal_h \subset \R^d$ for all timesteps $h$ and \Cref{assume:linear_regret} holds. Assume
  $\sup_{f\in \Hcal,h\in [H]} \norm{W_h(f)}_2 \leq B_W$ and $\sup_{f\in \Hcal,h\in [H]} \norm{X_h(f)}_2 \leq B_X$.
Fix $\delta \in (0,1/3)$ and
  batch sample size $m$ and define:
  \[
  \widetilde d_m = H \Big\lceil 3d \ln\Big( 1 + \frac{3B^2_X B^2_W}{\epsg^2(m, \Hcal)}\Big) \Big\rceil.
  \] Set the parameters as: number of iterations $T = \widetilde d_m$ and confidence radius
  $\epso = \sqrt{T} \epsg(m, \Hcal) \cdot
  \conf(\delta/(T H)) $. 
With probability at least $1-\delta$, $\Cref{alg:general}$ uses at most
  $mHT$ trajectories and returns a hypothesis $f$ such
  that:
    \[
V^\star(s_0) - V^{\pi_{f}}(s_0) 
\leq 3H \epsg(m,\Hcal)\cdot
\Big( 1 + \sqrt{\widetilde d_m}\cdot\conf\big(\frac{\delta}{\widetilde d_m H}\big) \Big)  \, .
    \]
\end{theorem}

As discussed in the \Cref{remark:epsg}, $\epsg(m, \Hcal)$ and
$\conf(\delta)$ measure the uniform convergence of discrepancy
functions $\ell_f$ for the hypothesis class $\Hcal$. Therefore, if
$\epsg(m, \Hcal)$ decays at least as fast as $m^{-\alpha}$ for any
constant $\alpha$, we will get efficient reinforcement learning. In
fact, we will see in our examples (\Cref{sec:rates-4ex}), that this is
true for all known models where efficient reinforcement learning is
possible. One such example is finite hypothesis classes where we
immediately get the following sample complexity bound showing only a
\emph{logarithmic} dependence on the size of the hypothesis space. 
\begin{corollary}
  \label{cor:main-result-finite}
(Finite-dimensional, Finite Hypothesis Case)
  Suppose $(\Hcal,\ell,\Pi_{\mathrm{est}}, \Mcal)$
 is a Bilinear Class with $\Xcal_h \subset \R^d$ for all timesteps $h$, $|\Hcal|> 1$ and \Cref{assume:linear_regret} holds. Assume
  $\sup_{f\in \Hcal,h\in [H]} \norm{W_h(f)}_2 \leq B_W$ and $\sup_{f\in \Hcal,h\in [H]} \norm{X_h(f)}_2 \leq B_X$ for some $B_X, B_W \geq 1$. Assume the discrepancy function $\ell_f$ is bounded i.e. $\sup_{f\in \Hcal}|\ell_f(\cdot)| \leq H + 1$. 
Fix $\delta \in (0,1/3)$ and $\eps \in (0, 1)$. Then there exists absolute constants $c_1, c_2, c_3, c_4$ such that setting the parameters: 
  batch sample size \[
    m = \frac{c_1 dH^5 \ln(dH^2) \ln(|\Hcal|) \ln(1/\delta)}{\eps^2} \ln \Big( \frac{dH  B_X B_W \ln(|\Hcal|) \ln(1/\delta)}{\eps}\Big)\, ,
  \] number of iterations $T = c_2dH \ln\Big(B_X B_W m\Big)$
  and confidence radius
  $\epso = c_3\sqrt{T} \cdot H \sqrt{\ln(|\Hcal|)/m} \cdot
  \ln(T H/\delta) $, with probability at least $1-\delta$, $\Cref{alg:general}$ returns a hypothesis $f$ such that $V^\star(s_0) - V^{\pi_{f}}(s_0) 
\leq \eps$ using at most \[
  \frac{c_4 d^2H^7\ln(dH^2) \ln(|\Hcal|) \ln(1/\delta)}{\eps^2} \ln^2\Big( \frac{dH B_X B_W \ln(|\Hcal|) \ln(1/\delta)}{\eps}\Big)
\] trajectories.
\end{corollary}
The proof for this corollary follows from bounds on $\epsg(m, \Hcal)$ and $\conf(\delta)$ using Hoeffding's inequality (\Cref{lemma:hoeffding}). We present the complete proof in \Cref{app:proofs-main-thm}.

Our next results will be non-parametric in nature and therefore it is helpful
to introduce the \emph{maximum information gain}
\citep{srinivas2009gaussian}, which captures an important notion of the
effective dimension of a set.
Let $\Xcal\subset \Vcal$ , where $\Vcal$ is a Hilbert space. For
$\lambda > 0$ and integer $n>0$, the \emph{maximum information gain}
$\gamma_n(\lambda; \Xcal)$ is defined as:
\begin{align}
	\label{eq:maxinfo}
  \gamma_n(\lambda; \Xcal) := \max_{x_0\dots x_{n-1} \in \Xcal}  \ln\det\left( \iden + \frac{1}{\lambda} \sum_{t=0}^{n-1} x_t x_t^{\top}  \right).
\end{align}
If $\Xcal$ is of the form  $\Xcal = \{\Xcal_h: h \in [H]\}$, we use
the notation \begin{equation}
  \label{eq:sumgamma}
  \gamma_n(\lambda; \Xcal) := \sum_{h\in [H]} \gamma_n(\lambda;
\Xcal_h)\, .
\end{equation}

Define \emph{critical information gain}, denoted by
$\widetilde \gamma(\lambda; \Xcal)$, as the smallest integer $k>0$
s.t. $k \geq \gamma_{k}(\lambda; \Xcal)$, i.e.  
\begin{align}
	\label{eq:crossinginfo}
	\widetilde \gamma(\lambda; \Xcal):= \min_{k \geq \gamma_{k}(\lambda; \Xcal)} k,
\end{align}
(where $k$ is an integer).
Note that such a $\widetilde \gamma(\lambda; \Xcal)$ exists provided
that the information gain $\gamma_{n}(\lambda; \Xcal)$ has a
sufficiently mild growth condition in both $n$ and $1/\lambda$. 
The \emph{critical information gain} can viewed as
an analogous quantity to the \emph{critical radius}, a quantity which
arises in non-parametric statistics~\citep{wainwright2019high}.

\begin{remark}
	For finite dimension setting where $\Xcal \subset \mathbb{R}^d$ and $\| x\| \leq B_{X}$ for any $x\in\Xcal$, we  have:
	$\gamma_n(\lambda;\Xcal) \leq d \ln\left( 1 + n B^2_X / d\lambda \right) $ and $\tilde \gamma(\lambda; \Xcal) \leq 3d \ln\left(1 + 3B^2_X/\lambda\right)$ (see \Cref{lemma:crit-gain} for a proof).
	Note that $1/\lambda$, n, and the norm bound $B_X$ only appear inside the log.  Furthermore, it is possible that $\gamma_n(\lambda;\Xcal)$ is much smaller than the dimension of $\Xcal$ (or $\Vcal$),  when the eigenspectrum of the covariance matrices concentrates in a low-dimension subspace. In fact when $\Xcal$ belongs to some infinite dimensional RKHS, $\gamma_n(\lambda; \Xcal)$ could still be small~\citep{srinivas2009gaussian}. 
	\end{remark}

We now present our main theorem. Recall the definitions $\Xcal_h := \{X_h(f)\colon f\in \Hcal\}$ and   $\Xcal := \{\Xcal_h: h \in [H]\}$.
\begin{theorem}
  \label{thm:main-result}
(RKHS case)
  Suppose $(\Hcal,\ell,\Pi_{\mathrm{est}}, \Mcal)$
 is a Bilinear Class and \Cref{assume:linear_regret} holds. Assume
  $\sup_{f\in \Hcal,h\in [H]} \norm{W_h(f)}_2 \leq B_W$.
Fix $\delta \in (0,1/3)$, 
  batch sample size $m$, and define:
\[
\widetilde d_m = \widetilde \gamma\Big(\epsg^2(m, \Hcal)/B_W^2 ; \Xcal\Big).
\]
Set the parameters as: number of iterations
  $T = \widetilde d_m$ and confidence radius
  $\epso = \sqrt{\widetilde d_m} \epsg(m, \Hcal) \cdot
  \conf(\delta/(\widetilde d_m H)) $. 
With probability at least $1-\delta$, $\Cref{alg:general}$ uses at most
  $mH \widetilde d_m$ trajectories and returns a hypothesis $f$ such
  that:
    \[
V^\star(s_0) - V^{\pi_{f}}(s_0) 
\leq 3H  \epsg(m,\Hcal)\cdot
\Big( 1 + \sqrt{\widetilde d_m}\cdot\conf\big(\frac{\delta}{\widetilde d_m H}\big) \Big) \, .
    \]
\end{theorem}

Next, we provide an elementary and detailed proof for our main theorem using an elliptical potential argument.


\subsection{Proof of Theorem \ref{thm:main-result-finite} and Theorem \ref{thm:main-result}}
\label{sec:analysis}

In this subsection, we prove our main theorems~--~\Cref{thm:main-result-finite} and \Cref{thm:main-result}.
\paragraph{Notation} To simplify notation, we denote by $\mu_{t;h}$ the distribution induced over $\Scal \times \Acal \times \Scal$ by $a_{0:h-1}\sim d^{\pi_{f_t}}$ and $a_h\sim \pi_{est}$; $\Dcal_{t;h}$ the batch dataset collected from distribution $\mu_{t;h}$; $\epsg$ the \emph{generalization error} $\epsg(m, \Hcal) \cdot \conf(\delta/(TH))$. Also, recall that for any distribution $\mu$ over $\R \times \Scal \times \Acal \times \Scal$ and hypothesis $f,g\in \Hcal$ \begin{equation*}
	\Lcal_{\mu,f}(g) = \E_{o\sim \mu} [\ell_{f}(o, g)]
  \end{equation*} 


Note that throughout the proof unless specified, the statements are true for any fixed $\delta \in (0,1)$, integer $m> 0$ and integer $T>0$. Also, we set $\epso = \sqrt{T}\epsg$ throughout the proof. To simplify the proof, we will condition on the event that uniform convergence of $\ell$ holds throughout our algorithm, which we first show holds with high probability. 
\begin{lemma}[Uniform Convergence] \label{lem:concentration} For all $t\in[T]$ and $g\in\Hcal$ and $h\in[H]$, with probability at least $1-\delta$, we have:
\begin{align*}
\Abs{\Lcal_{\Dcal_{t;h}, f_t}(g)  -  \Lcal_{\mu_{t;h}, f_t}(g)} \leq  \epsg
\end{align*}
\end{lemma}
\begin{proof}
	This follows from the uniform convergence (\Cref{assume:linear_regret}) and then union bounding over all $t\in [T]$ and $h\in [H]$.
\end{proof}We start by presenting our main lemma which shows if uniform convergence of $\ell$ holds throughout our algorithm, our algorithm finds a near-optimal policy. This lemma will be enough to prove our main results.

\begin{lemma}[Existence of high quality policy] \label{lemma:bound-iteration} Suppose we run the algorithm for $T$ iterations. Set $\epso = \sqrt{T}\epsg$. Assume the event in \Cref{lem:concentration} holds and $\sup_{f\in \Hcal} \norm{W_h(f)}_2 \leq B_W$ for all $h \in [H]$. Then, for all $\lambda\in \R^+$, there exists $t\in [T]$ such that the following is true for hypothesis $f_t$:
	\begin{align*}
		V^\star - V^{\pi_{f_t}}(s_0) \leq H  \sqrt{(4\lambda B_W^2 + 4T\epsg^2) \left(\exp\left( \frac{1}{T}\gamma_T(\lambda; \Xcal)\right) - 1\right)}
	  \end{align*}
 \end{lemma}

 We now complete the proof of \Cref{thm:main-result-finite} and \Cref{thm:main-result} using \Cref{lem:concentration}, \Cref{lemma:bound-iteration} and setting the parameters using the definition of critical information gain.
 \begin{proof}[Proof of \Cref{thm:main-result-finite} and \Cref{thm:main-result}]
	 Fix $\lambda = \epsg^2(m, \Hcal)/B_W^2$. From definition of critical information gain (\Cref{eq:crossinginfo}), it follows that for $T = \widetilde \gamma(\lambda, \Xcal)$, \begin{align*}
		 T \geq \gamma_T(\lambda, \Xcal)
	 \end{align*}
 Using \Cref{lemma:bound-iteration}, we get that \begin{align*}
		 V^\star - V^{\pi_{f_t}}(s_0) &\leq H \sqrt{\Big(4\lambda B_W^2 + 4T\epsg^2(m, \Hcal) \cdot \conf^2(\delta/TH)\Big) \left(\exp\left( \frac{1}{T}\gamma_T(\lambda; \Xcal)\right) - 1\right)}
	 \end{align*} Observing that for our choice of $T$, $\gamma_T(\lambda; \Xcal)/T \leq 1$ and $e-1< 2$ , we get \begin{align*}
		 V^\star - V^{\pi_{f_t}}(s_0) &\leq \sqrt{8} H \sqrt{\Big(\lambda B_W^2 + \widetilde \gamma(\lambda, \Xcal)\epsg^2(m, \Hcal) \cdot \conf^2(\delta/TH)\Big) }\\
		 &\leq \sqrt{8} H \Big( \sqrt{\lambda} B_W + \sqrt{\widetilde \gamma(\lambda, \Xcal)} \epsg(m, \Hcal) \cdot \conf(\frac{\delta}{\widetilde \gamma(\lambda, \Xcal) H})\Big)\\
		 &= \sqrt{8} H \Big( 1 + \sqrt{\widetilde \gamma(\lambda, \Xcal)} \cdot \conf(\frac{\delta}{\widetilde \gamma(\lambda, \Xcal) H})\Big) \cdot \epsg(m,\Hcal)\\
		 &\leq 3 H \Big( 1 + \sqrt{\widetilde \gamma(\lambda, \Xcal)} \cdot \conf(\frac{\delta}{\widetilde \gamma(\lambda, \Xcal) H})\Big) \cdot \epsg(m,\Hcal)
	 \end{align*} where the second last equality uses the definition of $\lambda$. 
	 
	 Moreover, each iteration of the algorithm, takes only $mH$ trajectories, this gives the total trajectories as $mHT = mH\widetilde \gamma(\lambda, \Xcal)$. This proves \Cref{thm:main-result}. \Cref{thm:main-result-finite} follows from the upper bound on $\widetilde \gamma(\lambda, \Xcal)$ for finite dimensional $\Xcal_h$ using \Cref{lemma:crit-gain}.
 \end{proof}

 In the rest of the section, we will prove our main lemma~--~\Cref{lemma:bound-iteration}. The first
 step shows that under \Cref{assume:linear_regret}, our $\epso$ is
 set properly so that $f^\star$ is always a feasible solution of the
 constrained optimization program in \Cref{alg:general}.
\begin{lemma}[Feasibility of $f^\star$] \label{lemma:feasible}Assume the event in \Cref{lem:concentration} holds. Then for all $t \in [T]$, we have that $f^\star$ is always a feasible solution. 
\end{lemma}
\begin{proof}
Note that $\Lcal_{\mu_{i;h}, f_i}(f^\ast) = 0$ (\Cref{eq:assume2}).  Thus using \Cref{lem:concentration}, we have:
\begin{align*}
\sum_{i=0}^{t-1}\left(\Lcal_{\Dcal_{i;h}, f_i}(f^\ast)\right)^2 \leq t \epsg^2\quad\quad \forall h \in [H]\, .
\end{align*} Noting that $t\leq T$ and in our parameter setup $\epso = \sqrt{T}\epsg$ completes the proof.
\end{proof}
The feasibility result immediately leads to optimism.
\begin{lemma}[Optimism] \label{lem:optimism} Assume the event in \Cref{lem:concentration} holds. Then for all $t\in [T]$, we have $V^\star \leq V_{f_t;0}(s_0)$.
\end{lemma}
\begin{proof}
	\Cref{lemma:feasible} implies $f^\star$ is a feasible solution for the optimization program for all $t\in [T]$. This proves the claim.
\end{proof}

The following lemma relates the sub-optimality to a sum of bilinear forms. Using the performance difference
lemma, we first show that sub-optimality is upper bounded by the Bellman errors of $Q_{h, f_t}$, which are further upper bounded by sum of bilinear forms via our assumption (\Cref{eq:assume1}).

\begin{lemma}[Bilinear Regret Lemma]\label{lemma:opt}
  Assume the event in \Cref{lem:concentration} holds. 
Then, the following holds for all $t\in [T]$: 
\[
V^\star - V^{\pi_{f_t}}(s_0) \leq 
\sum_{h=0}^{H-1} \abs{\langle W_{h}(f_{t}) - W_{h}(f^\star), X_h(f_{t})  \rangle} \, .
\]
\end{lemma}
\begin{proof}
  We can upper bound the regret
\begin{align*}
		&V^\star(s_0) - V^{\pi_{f_t}}(s_0)\\
		&\leq  V_{0,f_t}(s_0) - V^{\pi_{f_t}}(s_0)\tag{since $V_{0,f_t}(s_0) \geq V^\star(s_0)$ (\Cref{lem:optimism})}\\
		&= Q_{0,f_t}(s_0, a_0) -\E_{a_{0:h} \sim d^{\pi_{f_t}}}\Bracks{\sum_{h=0}^{H-1} r(s_h,a_h)}\tag{since $V_{f_t}(s_{0}) = Q_{f_t}(s_{0}, a_{0})$, $a_0 = \argmax_{a} Q_{f_t}(s_0,a)$}\\
		&= \E_{a_{0:h} \sim d^{\pi_{f_t}}}\Bracks{\sum_{h=0}^{H-1} \left(Q_{h,f_t}(s_h,a_h) - r(s_h,a_h) - Q_{h+1, f_t}(s_{h+1},a_{h+1})\right)} \tag{by telescoping sum}\\
		&= \sum_{h=0}^{H-1} \E_{a_{0:h} \sim d^{\pi_{f_t}}}\Bracks{Q_{h,f_t}(s_h,a_h) - r(s_h,a_h) - Q_{h+1, f_t}(s_{h+1},a_{h+1})}\\
		&= \sum_{h=0}^{H-1} \E_{a_{0:h} \sim d^{\pi_{f_t}}}\Bracks{Q_{h,f_t}(s_h,a_h) - r(s_h,a_h) - V_{h+1,f_t}(s_{h+1})}\tag{since $V_{h+1,f_t}(s_{h+1}) = Q_{h+1,f_t}(s_{h+1}, a_{h+1})$}\\
		&= \sum_{h=0}^{H-1} \abs{\langle W_{h}(f_{t}) -
                  W_{h}(f^\star), X_h(f_{t})  \rangle} 
		\end{align*} where the last step follows
                \Cref{eq:assume1} in the Bilinear Class definition.
 \end{proof}

 The following is a variant of the Elliptical Potential Lemma, central in the analysis of linear bandits~\citep{dani2008stochastic,srinivas2009gaussian,abbasi2011improved}.

 \begin{lemma}[Elliptical potential]\label{lemma:potential_argument}Consider any sequence of vectors $\{x_0,\dots, x_{T-1}\}$ where $x_i \in \Vcal$ for some Hilbert space $\Vcal$.  Let $\lambda \in \mathbb{R}^+$.  Denote $\Sigma_0 = \lambda I$ and $\Sigma_t = \Sigma_0 + \sum_{i=0}^{t-1} x_i x_i^{\top}$. We have that:
 \begin{align*}
 \min_{i\in[T]} \ln\left( 1 + \left\|  x_i \right\|^2_{\Sigma_{i}^{-1}}  \right) \leq
 \frac{1}{T}\sum_{i=0}^{T-1} \ln\left( 1 + \left\|  x_i \right\|^2_{\Sigma_{i}^{-1}}  \right) = \frac{1}{T} \ln\frac{\det\left( \Sigma_T \right)  }{ \det(\lambda I)}.
 \end{align*}
 \end{lemma}
 \begin{proof}
 By definition of $\Sigma_t$ and matrix determinant lemma, we have:
 \begin{align*}
 \ln\det( \Sigma_{t+1} ) &  = \ln\det(\Sigma_t) + \ln\det\left( I + (\Sigma_t)^{-1/2} x_t x_t^{\top}  (\Sigma_t)^{-1/2}\right) \\
 & =  \ln\det(\Sigma_t) +  \ln\left( 1 + \| x_t \|^2_{\Sigma_t^{-1}} \right).
 \end{align*}
 Using recursion completes the proof. 
 \end{proof}

 Now, we will finish the proof of \Cref{lemma:bound-iteration} by showing that the sum of bilinear forms in \Cref{lemma:opt} is small for at least for one $t\in [T]$. More precisely, using \Cref{eq:assume2}  together with elliptical potential argument (\Cref{lemma:potential_argument}), we can show that after
 $\widetilde{d}_m$ many iterations, we must have found a policy
 $\pi_{f_t}$ such that $\left\lvert  \langle W_h(f_t) - W_h(f^\star),
   X_h( f_t) \rangle \right\rvert$ is small for all $h$.

 \begin{proof}[Proof of \Cref{lemma:bound-iteration}]
	Our goal (as per \Cref{lemma:opt} and \Cref{eq:assume1}) is to find $t\in [T]$ such that \[
		\Abs{\langle W_{h}(f_{t}) - W_{h}(f^\star), X_h(f_{t})  \rangle} \quad \text{is small for all}~ h\in [H]	
	\] To that end, we will show that \[
		\norm{W_{h}(f_{t}) - W_{h}(f^\star)}_A\quad \norm{X_h(f_{t})}_{A^{-1}}\quad \text{is small for all}~ h\in [H]
	\] for appropriately chosen $A$. We will show existence of such $X_h(f_{t})$ and $A$ (\Cref{eq:existence}) using the potential argument (\Cref{lemma:potential_argument}) and conditions on $W_h(f_{t}) - W_h(f^\star)$ follow from our optimization program. We now show this in more detail. 
	
	Let the hypothesis used by our algorithm at $i$th iteration be $f_i$. Consider the corresponding sequence of representations $\{X_{h}(f_{i})\}_{i,h}$. Then, by \Cref{lemma:potential_argument}, we have that for all $h \in [H]$ and $\lambda \in \R^+$ \begin{align*}
		\sum_{i=0}^{T-1} \ln\left( 1 + \left\|  X_{h}(f_{i}) \right\|^2_{\Sigma_{i;h}^{-1}}  \right) \leq  \ln\frac{\det\left( \Sigma_{T;h} \right)  }{ \det(\lambda \iden)} \leq \gamma_T(\lambda; \Xcal_h)
	 \end{align*} where we have used definition of maximum information gain $\gamma_T(\lambda; \Xcal_h)$ (\Cref{eq:maxinfo}) and \begin{align*}
		 \Sigma_{i;h} &= \lambda \iden + \sum_{j=0}^{i-1} X_{h}(f_{j}) X_{h}(f_{j})^\top
	 \end{align*} Summing these inequalities over all $h \in [H]$, we have that for all $\lambda \in \R^+$ \begin{align*}
		\sum_{i=0}^{T-1} \sum_{h=0}^{H-1}\ln\left( 1 + \left\|  X_{h}(f_{i}) \right\|^2_{\Sigma_{i;h}^{-1}}  \right) \leq  \sum_{h=0}^{H-1}\gamma_T(\lambda; \Xcal_h) = \gamma_T(\lambda; \Xcal)
	 \end{align*} where the last equality follows from \Cref{eq:sumgamma}. Since, each of these terms is $\geq 0$, we get that there exists $t \in [T]$ such that \begin{align*}
		\sum_{h=0}^{H-1}\ln\left( 1 + \left\|  X_{h}(f_{t}) \right\|^2_{\Sigma_{t;h}^{-1}}  \right) \leq  \frac{1}{T}\gamma_T(\lambda; \Xcal)
	 \end{align*} Again, since each of these terms is $\geq 0$, we get that for all $h \in [H]$\begin{align*}
		\ln\left( 1 + \left\|  X_{h}(f_{t}) \right\|^2_{\Sigma_{t;h}^{-1}}  \right) \leq  \frac{1}{T}\gamma_T(\lambda; \Xcal)
	 \end{align*} and simplifying, we get that for all $h \in [H]$,  \begin{align}
		\label{eq:existence}
		\left\|  X_{h}(f_{t}) \right\|^2_{\Sigma_{t;h}^{-1}} \leq \exp\left( \frac{1}{T}\gamma_T(\lambda; \Xcal)\right) - 1
	 \end{align}
	  Also, by construction of our program, for all iterations and in particular for $t$, it holds that for all $h\in [H]$ \[
		 \sum_{j=0}^{t-1} \bigg(\Lcal_{\Dcal_{j; h}, f_{j}}(f_{t})\bigg)^2 \leq T\epsg^2
	  \] and by \Cref{lem:concentration}, for all $h \in [H]$ \begin{align*}
		\sum_{j=0}^{t-1} \bigg(\Lcal_{\mu_{j; h}, f_{j}}(f_{t})\bigg)^2&\leq 2 \sum_{j=0}^{t-1} \bigg(\Lcal_{\Dcal_{j; h}, f_{j}}(f_{t})\bigg)^2 + 2 \sum_{j=0}^{t-1}\epsg^2\\
		&\leq 4T\epsg^2
	  \end{align*} where the first inequality follows from $(a+b)^2 \leq 2a^2 + 2b^2$ and the last step follows from the frame above and $t\in [T]$. Using the definition of \classname~(\Cref{eq:assume2}), for all $h \in [H]$ \[
		 \sum_{j=0}^{t-1} \Abs{\langle W_{h}(f_{t}) - W_{h}(f^\star), X_{h}(f_j)\rangle}^2 \leq 4T\epsg^2
	  \] Using this, we get for all $h \in [H]$ \begin{align}
		&(W_{h}(f_{t}) - W_{h}(f^\star))^\top \Sigma_{t; h} (W_{h}(f_{t}) - W_{h}(f^\star)) \nonumber \\
		& \leq \lambda \norm{(W_{h}(f_{t}) - W_{h}(f^\star))}_2^2  + 4T\epsg^2 \nonumber \\
		& \leq 4\lambda B_W^2 + 4T\epsg^2 \label{eq:linear_regret_bound}
	  \end{align} where the first inequality follows from the frame above and definition of $\Sigma_{t;h}$. Using \Cref{eq:existence} and the frame above, this immediately shows that for all $h \in [H]$\begin{align*}
		\Abs{\langle W_{h}(f_{t}) - W_{h}(f^\star), X_h(f_{t})  \rangle}^2 &\leq \norm{W_{h}(f_{t}) - W_{h}(f^\star)}_{\Sigma_{t; h}}^2 \norm{X_h(f_{t})}_{\Sigma_{t; h}^{-1}}^2\\
		&\leq (4\lambda B_W^2 + 4T\epsg^2) \left(\exp\left( \frac{1}{T}\gamma_T(\lambda; \Xcal)\right) - 1\right)
	  \end{align*}  Summing over all $h \in [H]$, this gives \begin{align*}
		\sum_{h=0}^{H-1}\Abs{\langle W_{h}(f_{t}) - W_{h}(f^\star), X_h(f_{t})\rangle}& \leq H\sqrt{(4\lambda B_W^2 + 4T\epsg^2) \left(\exp\left( \frac{1}{T}\gamma_T(\lambda; \Xcal)\right) - 1\right)}
	  \end{align*} Using \Cref{lemma:opt}, this gives the desired result.
\end{proof}

\subsection{Corollaries for Particular Models}
\label{sec:rates-4ex}
In this section, we apply our main theorem to special models: linear $Q^\star/V^\star$, RKHS bellman complete, RKHS linear mixture model, and low occupancy complexity model.  While linear bellman complete and linear mixture model have been studied, our results extends to infinite dimensional RKHS setting.  


\subsubsection{Linear $Q^\star/V^\star$}
In this subsection, we provide the sample complexity result for the linear $Q^\star / V^\star$ model (\Cref{def:linearqv}). To state our results for linear $Q^\star/V^\star$, we define the following sets:  
\begin{align*}
  \Phi &= \Big\{\phi(s,a)\colon (s,a) \in \Scal \times \Acal\Big\},\;  \Psi= \Big\{\psi(s') \colon s' \in \Scal\Big\}.
\end{align*} and define the concatenation set\footnote{For infinite dimensional $\Phi$ and $\Psi$, we consider the natural inner product space where $\langle [x_1,y_1], [x_2, y_2]\rangle = \langle x_1, x_2 \rangle + \langle y_1, y_2\rangle$.} \[
  \Phi \circ \Psi = \Big\{[x , y]: x \in \Phi, y \in \Psi\Big\}
  \] We first provide the result for the finite dimensional case i.e. when $\Phi \conc \Psi \subset \R^d$. \begin{corollary}[Finite Dimensional Linear $Q^\star / V^\star$]
  \label{cor:linearqv_finite}
  Suppose MDP $\Mcal$ is a linear $Q^\star/V^\star$ model with $\Phi \conc \Psi \subset \R^d$.
  Assume $\sup_{(w, \theta) \in \Hcal_h, h \in [H]}\norm{[w, \theta]}_2 \leq B_W$ and $\sup_{x\in \Phi \conc \Psi}\norm{x}_2 \leq B_X$ for some $B_X, B_W \geq 1$.  
  Fix $\delta \in (0,1/3)$ and $\eps \in (0,H)$. There exists an appropriate setting of batch sample size $m$, number of iteration $T$ and confidence radius $R$ such that with probability at least
   $1-\delta$, $\Cref{alg:general}$ returns a hypothesis $f$ such that $V^\star(s_0) - V^{\pi_{f}}(s_0) 
\leq \eps$ using at most \[
  c_1 \frac{d^3 H^6 \ln(1/\delta)}{\eps^2} \cdot \Big(\ln \big(c_2 \frac{d^3 H^7 B_X^2 B_W^2 \ln(1/\delta) }{\eps^2}\big)  \Big)^5 
\] trajectories for some absolute constant $c_1,c_2$.
\end{corollary}  
To prove this, we will prove a more general sample complexity result for the infinite dimensional RKHS case.
\begin{corollary}[RKHS Linear $Q^\star/V^\star$]
  \label{cor:linearQV_detailed}
  Suppose MDP $\Mcal$ is a linear $Q^\star/V^\star$ model. 
  Assume $\sup_{(w, \theta) \in \Hcal_h, h \in [H]}\norm{[w, \theta]}_2 \leq B_W$ and $\sup_{x\in \Phi \conc \Psi}\norm{x}_2 \leq B_X$. 
  Fix $\delta \in (0,1/3)$, batch sample size $m$, and define: \begin{align}
    \label{def:dmphipsi}
      &\widetilde d_{m}(\Phi \conc \Psi) = \widetilde \gamma \Big( \frac{1}{8 B_W^2 m}; \Phi \conc \Psi \Big) \cdot \nu, \\
      \label{def:dmx}
      &\widetilde d_{m}(\Xcal) = \widetilde \gamma\left(\frac{144H^2\widetilde d_{m}(\Phi \conc \Psi)}{B_W^2m} ; \Xcal\right),
  \end{align} where $\nu := \ln\left(1 + 3B_X B_W \sqrt{m\widetilde \gamma \Big( \frac{1}{8 B_W^2 m}; \Phi \conc \Psi \Big) }\right)$.  
   
   Set the parameters as: $\epso = (12H/\sqrt{m})\sqrt{\widetilde d_{m}(\Xcal)\cdot\widetilde d_{m}(\Phi \conc \Psi)} \cdot
   \sqrt{\ln\big((\widetilde d_{m}(\Xcal) H)/\delta\big)}$ and $T = \widetilde d_{m}(\Xcal)$.
  With probability greater
  than $1-\delta$, $\Cref{alg:general}$ uses at most
  $mH \widetilde d_{m}(\Xcal)$ trajectories and returns a hypothesis $f$:      \begin{equation}
    \label{eq:optthm}
V^\star(s_0) - V^{\pi_{f}}(s_0) 
\leq 72H^2 \frac{\sqrt{\widetilde d_{m}(\Xcal) \cdot \widetilde d_{m}(\Phi \conc \Psi)} \cdot
v}{\sqrt{m}},
  \end{equation} where $v := \sqrt{\ln\left((\widetilde d_{m}(\Xcal) H)/\delta\right)}$. 
\end{corollary}
\begin{proof}
  First, using \Cref{cor:concentration-rkhs-linear}, we get that for any distribution $\mu$ over $\Scal\times \Acal \times \Scal$ and for any $\delta \in (0,1)$, with probability of at least $1-\delta$ over choice of an i.i.d. sample $\Dcal\sim \mu^m$ of size $m$, for all $g = ([w_0,\theta_0], \ldots, [w_{H-1},\theta_{H-1}])\in \Hcal$ (note that $\Lcal_{\mu}(g)$ only depends on $[w_h,\theta_h]$ for distribution $\mu$ over observed transitions $o_h = (r_h,s_h, a_h, s_{h+1})$ at timestep $h$.) \begin{align*}
      \Abs{\Lcal_{\Dcal}(g) - \Lcal_{\mu}(g) }&\leq \frac{4}{\sqrt{m}}  + 2H \sqrt{\frac{ 2\widetilde{\gamma}_m \ln\left( 1+ 3 B_X B_W \sqrt{\widetilde{\gamma}_m m} \right) + 2\ln(1/\delta)  }{m}}\\
      &= \frac{4 + 2H\sqrt{2\widetilde{\gamma}_m \ln\left( 1+ 3 B_X B_W \sqrt{\widetilde{\gamma}_m m} \right) + 2\ln(1/\delta) }}{\sqrt{m}}\\
          &\leq \frac{12H\sqrt{\widetilde{\gamma}_m \ln\left( 1+ 3 B_X B_W \sqrt{\widetilde{\gamma}_m m} \right)} \cdot \sqrt{ \ln(1/\delta) }}{\sqrt{m}}
  \end{align*} where we have used that $\ln(1/\delta) > 1$ and $\widetilde{\gamma}_m= \widetilde \gamma( 1/ (8 B_W^2 m); \Phi \conc \Psi )$ (as defined in \Cref{eq:crossinginfo}). Define \begin{align*}
      \widetilde d_{m}(\Phi \conc \Psi) := \widetilde{\gamma}_m \ln\left( 1+ 3 B_X B_W \sqrt{\widetilde{\gamma}_m m} \right)
  \end{align*}
  This satisfies our \Cref{assume:linear_regret} with \begin{align*}
          \epsg(m, \Hcal) &= \frac{12H\sqrt{\widetilde d_{m}(\Phi \conc \Psi)}}{\sqrt{m}}\\
          \conf(\delta) &= \sqrt{ \ln(1/\delta) }
      \end{align*} Substituting this in \Cref{thm:main-result} gives the result \begin{align*}
          \widetilde d_m(\Xcal) &= \widetilde \gamma\Big(\epsg^2(m, \Hcal)/B_W^2 ; \Xcal\Big)\\
          &= \widetilde \gamma\Big(144H^2\widetilde d_{m}(\Phi \conc \Psi)/mB_W^2 ; \Xcal\Big)\\
          V^\star(s_0) - V^{\pi_{f_t}}(s_0) 
&\leq 6H \sqrt{\widetilde d_m(\Xcal)} \cdot \epsg(m,\Hcal)\cdot
\conf\big(\delta/(\widetilde d_m(\Xcal) H)\big)\\
&= 72H^2 \frac{\sqrt{\widetilde d_{m}(\Xcal) \cdot \widetilde d_{m}(\Phi \conc \Psi)} \cdot
\sqrt{\ln\big((\widetilde d_{m}(\Xcal) H)/\delta\big)}}{\sqrt{m}}
      \end{align*}
\end{proof}
Next, we complete the proof of \Cref{cor:linearqv_finite}. Note that both $\widetilde d_{m}(\Phi \circ \Psi)$ and $\widetilde d_{m}(\Xcal) $ (related to critical information gain under $\Phi$ and $\Xcal$ respectively) scale as $\widetilde{O}(d)$ if $\Phi \circ \Psi \subset\mathbb{R}^d$. 
\begin{proof}[Proof of \Cref{cor:linearqv_finite}]
  First, from \Cref{lemma:crit-gain}, we have that \begin{align*}
      \widetilde \gamma \Big( \frac{1}{8 B_W^2 m}; \Phi \conc \Psi\Big) &\leq 3d\ln\Big( 1 +  24B_X^2B_W^2 m\Big) + 1\\
      &\leq  3d\ln\Big( 25 B_X^2B_W^2 m\Big) + 1\\
      &\leq  4d\ln\Big( 25 B_X^2B_W^2 m\Big)
  \end{align*} and substituting this in \Cref{def:dmphipsi} \begin{align*}
      \widetilde d_{m}(\Phi \conc \Psi) 
      &\leq 4d\ln\Big( 25 B_X^2B_W^2 m\Big) \cdot \ln\left(1 + 3B_XB_W \sqrt{m 4d\ln\Big( 25 B_X^2B_W^2 m\Big) }\right)\\
      &\leq  4d\ln\Big( 25 B_X^2B_W^2 m\Big) \cdot \ln\left( 4B_XB_W \sqrt{m 4d\ln\Big( 25 B_X^2B_W^2 m\Big) }\right)\\
      &\leq 4d\ln\Big( 25 B_X^2B_W^2 m\Big) \cdot \left( \ln(4B_XB_W) + \ln\Big(10 m \sqrt{d} B_XB_W \Big)\right)\\
      &\leq 8d \ln^2(25B_X^2B_W^2 m \sqrt{d}) 
  \end{align*}
  Similarly, as $\sup_{z\in \Xcal}\norm{z} \leq \sup_{x \in \Phi \conc \Psi} \norm{x}$, using \Cref{lemma:crit-gain} and similar analysis as above (and $144H^2\widetilde d_{m}(\Phi \conc \Psi) \geq 1$), we get \begin{align*}
    \widetilde \gamma\left(\frac{144H^2\widetilde d_{m}(\Phi \conc \Psi)}{B_W^2m} ; \Xcal_h\right) &\leq 4d\ln\Big( 25 B_X^2B_W^2 m\Big)
  \end{align*} and substituting this in \Cref{def:dmx}
    \begin{align*}
      \widetilde d_{m}(\Xcal) &\leq 4dH \ln\Big(4B_X^2B_W^2m\Big)
  \end{align*} To get $\epsilon$-optimal policy (from \Cref{eq:optthm}), we have to set \begin{align*}
      &72H^2 \frac{\sqrt{\widetilde d_{m}(\Xcal) \cdot \widetilde d_{m}(\Phi \conc \Psi)} \cdot
      \sqrt{\ln\big((\widetilde d_{m}(\Xcal) H)/\delta\big)}}{\sqrt{m}} \leq \eps\\
m &\geq (72)^2H^4 \frac{\widetilde d_{m}(\Xcal) \cdot \widetilde d_{m}(\Phi \conc \Psi) \cdot
\ln\big((\widetilde d_{m}(\Xcal) H)/\delta\big)}{\eps^2}
  \end{align*}
  Further upper bounding the right hand side of the above inequality by substituting in upper bounds for $\widetilde d_{m}(\Xcal)$ and $\widetilde d_{m}(\Phi \conc \Psi)$ from frames above, we can set $m$ to be as large as:
\begin{align*}
&(72)^2H^5 \frac{32d^2 \ln^3(25B_X^2B_W^2 m \sqrt{d})  \cdot
\ln\big((16d H^2B_X^2B_W^2m)/\delta\big)}{\eps^2}\\
&\leq 32\cdot(72)^2 \frac{d^2 H^5 \ln^4(25B_X^2B_W^2 m dH^2) \ln(1/\delta)}{\eps^2}
  \end{align*} Using \Cref{lemma:log-dominance} for $\alpha = 4$, $a = 32 \cdot (72)^2 d^2 H^5 \ln(1/\delta)/\eps^2$, $b = 25B_X^2B^2_WdH^2$ and $c= 5^4$, we get that \begin{align*}
      m &= 5^4 \cdot 32 \cdot (72)^2 \frac{d^2 H^5 \ln(1/\delta)}{\eps^2} \ln^4\Big(5^4 \cdot 25\cdot 32 \cdot (72)^2 \frac{d^3 H^7 B_X^2B^2_W \ln(1/\delta)}{\eps^2} \Big) \\
      \ln\Big(4B_X^2B_W^2m\Big) &\leq 5\ln \Big(5^6 \cdot 32 \cdot (72)^2 \frac{d^3 H^7 \ln(1/\delta) B_X^2B_W^2}{\eps^2}\Big)
  \end{align*} Substituting this in the expression above for $\widetilde d_{m}(\Xcal)$ and setting this upper bound to $T$, we get \begin{align*}
      T&= 20dH \ln \Big(5^6 \cdot 32 \cdot (72)^2 \frac{d^3 H^7 \ln(1/\delta) B_X^2B_W^2}{\eps^2}\Big)
  \end{align*} Since, we use on policy estimation, i.e., $\pi_{est} = \pi_{f_t}$ for all $t$, the trajectory complexity  is $mT$ which completes the proof.
\end{proof}

\subsubsection{RKHS Bellman Complete.}  \label{sec:rkhs_b_complete}
In this subsection, we provide the sample complexity result for the Linear Bellman Complete model (\Cref{def:linear_comp}). To state our results, we define \[
  \Phi = \{\phi(s,a): s,a\in\Scal\times\Acal\}\, .
  \]  We first provide the result for the finite dimensional case i.e. when $\Phi \subset \Vcal \subset \R^d$. \begin{corollary}[Finite Dimensional Linear Bellman Complete]
  \label{cor:linearbellcomp_finite}
  Suppose $\Hcal$ is \emph{Bellman Complete} with respect to MDP $\mathcal{M}$ for some Hilbert space $\Vcal \subset \R^d$.
   Assume $\sup_{\theta \in \Hcal_h, h \in [H]}\norm{\theta}_2 \leq B_W$ and $\sup_{x\in \Phi}\norm{x}_2 \leq B_X$ for some $B_X, B_W \geq 1$. 
  Fix $\delta \in (0,1/3)$ and $\eps \in (0,H)$. There exists an appropriate setting of batch sample size $m$, number of iteration $T$ and confidence radius $R$ such that with probability at least
   $1-\delta$, $\Cref{alg:general}$ returns a hypothesis $f$ such that $V^\star(s_0) - V^{\pi_{f}}(s_0) 
\leq \eps$ using at most \[
  c_1 \frac{d^3 H^6 \ln(1/\delta)}{\eps^2} \cdot \Big(\ln \big(c_2 \frac{d^3 H^7 B_X^2B_W^2 \ln(1/\delta) }{\eps^2}\big)  \Big)^5   
\] trajectories for some absolute constant $c_1,c_2$.
\end{corollary}  
In comparison, \cite{jin2019provably} has sample complexity $\widetilde O(d^3 H^3/\epsilon^2 \log(1/\delta))$ and \cite{zanette2020learning} has $\widetilde O(d^2 H^3/\epsilon^2 \log(1/\delta))$.  To prove this, we will prove a more general sample complexity result for the infinite dimensional RKHS case. Note that RKHS Linear MDP is a special instance of RKHS Bellman Complete. Prior works that studied RKHS Linear MDP either achieves worse rate \citep{agarwal2020pc} or further assumes finite covering dimension of the space of all possible upper confidence bound Q functions which are algorithm dependent quantities \citep{yang2020bridging}. 
\begin{corollary}[RKHS Bellman Complete]
  \label{cor:linearbellcomp_detailed}
  Suppose $\Hcal$ is \emph{Bellman Complete} with respect to MDP $\mathcal{M}$ for some Hilbert space $\Vcal$.
   Assume $\sup_{h \in [H], \theta \in \Hcal_h}\norm{\theta}_2 \leq B_W$ and $\sup_{x\in \Phi}\norm{x}_2 \leq B_X$. 
  Fix $\delta \in (0,1/3)$, batch sample size $m$, and define: \begin{align*}
      \widetilde d_{m}(\Phi) &= \widetilde \gamma \Big( \frac{1}{8 B_W^2 m}; \Phi \Big) \cdot \nu,\\
      \widetilde d_{m}(\Xcal) &= \widetilde \gamma\Big(\frac{400H^2d_{m}(\Phi )}{B_W^2m} ; \Xcal\Big),
  \end{align*} where $\nu = \ln\left(1 + 3B_X B_W \sqrt{m\widetilde \gamma \Big( \frac{1}{8 B_W^2 m}; \Phi\Big) }\right)$.
  
   Set the parameters as: $\epso = (12H/\sqrt{m})\sqrt{\widetilde d_{m}(\Xcal)\cdot\widetilde d_{m}(\Phi)} ~\cdot
   \sqrt{\ln\big((\widetilde d_{m}(\Xcal) H)/\delta\big)}$ and $T = \widetilde d_{m}(\Xcal)$.
  With probability at least
   $1-\delta$, $\Cref{alg:general}$ uses at most
  $mH \widetilde d_{m}(\Xcal)$ trajectories and returns a hypothesis $f$:
    \[
V^\star(s_0) - V^{\pi_{f}}(s_0) 
\leq 120H^2 \frac{\sqrt{\widetilde d_{m}(\Xcal) \cdot \widetilde d_{m}(\Phi)} \cdot
v}{\sqrt{m}},
    \]  where $v =  \sqrt{\ln\big((\widetilde d_{m}(\Xcal) H)/\delta\big)}$.
\end{corollary}

\begin{proof}
  First, using \Cref{cor:concentration-rkhs}, we get that for any distribution $\mu$ over $\Scal\times \Acal \times \Scal$ and for any $\delta \in (0,1)$, with probability of at least $1-\delta$ over choice of an i.i.d. sample $\Dcal\sim \mu^m$ of size $m$, for all $g = (\theta_0, \ldots, \theta_{H-1})\in \Hcal$ (note that $\Lcal_{\mu}(g)$ only depends on $\theta_h$ for distribution $\mu$ over observed transitions $o_h = (r_h,s_h, a_h, s_{h+1})$ at timestep $h$.) \begin{align*}
      \Abs{\Lcal_{\Dcal}(g) - \Lcal_{\mu}(g) }&\leq \frac{8}{\sqrt{m}}  + 2H \sqrt{\frac{ 2\widetilde{\gamma}_m \ln\left( 1+ 3 B_X B_W \sqrt{\widetilde{\gamma}_m m} \right) + 2\ln(1/\delta)  }{m}}\\
      &= \frac{8 + 2H\sqrt{2\widetilde{\gamma}_m \ln\left( 1+ 3 B_X B_W \sqrt{\widetilde{\gamma}_m m} \right) + 2\ln(1/\delta) }}{\sqrt{m}}\\
          &\leq \frac{20H\sqrt{\widetilde{\gamma}_m \ln\left( 1+ 3 B_X B_W \sqrt{\widetilde{\gamma}_m m} \right)} \cdot \sqrt{ \ln(1/\delta) }}{\sqrt{m}}
  \end{align*} where we have used that $\ln(1/\delta) > 1$ and $\widetilde{\gamma}_m= \widetilde \gamma( 1/ (8 B_W^2 m); \Phi )$ (as defined in \Cref{eq:crossinginfo}). Define \begin{align*}
      \widetilde d_{m}(\Phi) := \widetilde{\gamma}_m \ln\left( 1+ 3 B_X B_W \sqrt{\widetilde{\gamma}_m m} \right)
  \end{align*}
  This satisfies our \Cref{assume:linear_regret} with \begin{align*}
          \epsg(m, \Hcal) &= \frac{20H\sqrt{\widetilde d_{m}(\Phi)}}{\sqrt{m}}\\
          \conf(\delta) &= \sqrt{ \ln(1/\delta) }
      \end{align*} Substituting this in \Cref{thm:main-result} gives the result \begin{align*}
          \widetilde d_m(\Xcal) &= \widetilde \gamma\Big(\epsg^2(m, \Hcal)/B_W^2 ; \Xcal\Big)\\
          &= \widetilde \gamma\Big(400H^2\widetilde d_{m}(\Phi \conc \Psi)/mB_W^2 ; \Xcal\Big)\\
          V^\star(s_0) - V^{\pi_{f_t}}(s_0) 
&\leq 6H \sqrt{\widetilde d_m} \cdot \epsg(m,\Hcal)\cdot
\conf\big(\delta/(\widetilde d_m H)\big)\\
&= 120H^2 \frac{\sqrt{\widetilde d_{m}(\Xcal) \cdot \widetilde d_{m}(\Phi)} \cdot
\sqrt{\ln\big((\widetilde d_{m}(\Xcal) H)/\delta\big)}}{\sqrt{m}}
      \end{align*}
\end{proof}

We now complete the proof of \Cref{cor:linearbellcomp_finite}. Note that both $\widetilde d_{m}(\Phi)$ and $\widetilde d_{m}(\Xcal) $ (related to critical information gain under $\Phi$ and $\Xcal$ respectively) scale as $\widetilde{O}(d)$ if $\Phi \subset\mathbb{R}^d$. 
\begin{proof}[Proof of \Cref{cor:linearbellcomp_finite}]
  Since the proof follows similar to proof of \Cref{cor:linearqv_finite}, we will only provide a proof sketch here. First, from \Cref{lemma:crit-gain}, we have that \begin{align*}
    \widetilde \gamma \Big( \frac{1}{8 B_W^2 m}; \Phi\Big)
    &\leq  4d\ln\Big( 25 B_X^2B_W^2 m\Big)
\end{align*} and therefore \begin{align*}
    \widetilde d_{m}(\Phi)\leq 8d \ln^2(25B_X^2B_W^2 m \sqrt{d}) 
\end{align*}
Similarly, as $\sup_{z\in \Xcal}\norm{z} \leq \sup_{x \in \Phi} \norm{x}$, using \Cref{lemma:crit-gain} (and since $400H^2\widetilde d_{m}(\Phi) \geq 1$), we get \begin{align*}
  \widetilde \gamma\left(\frac{400H^2\widetilde d_{m}(\Phi)}{B_W^2m} ; \Xcal_h\right) &\leq 4d\ln\Big( 25 B_X^2B_W^2 m\Big)
\end{align*} and therefore
  \begin{align*}
    \widetilde d_{m}(\Xcal) &\leq 4dH \ln\Big(4B_X^2B_W^2m\Big)
\end{align*} To get $\epsilon$-optimal policy, we have to set \begin{align*}
    &120H^2 \frac{\sqrt{\widetilde d_{m}(\Xcal) \cdot \widetilde d_{m}(\Phi)} \cdot
    \sqrt{\ln\big((\widetilde d_{m}(\Xcal) H)/\delta\big)}}{\sqrt{m}} \leq \eps
\end{align*}
The rest of the proof follows similarly to proof of \Cref{cor:linearqv_finite}.
\end{proof}

\subsubsection{RKHS linear mixture model}
In this subsection, we provide the sample complexity result for the Linear Mixture model (\Cref{def:linearmm}). To present our sample complexity results, we define:
\begin{align*}
    \Phi_h &= \Big\{\psi(s,a) + \sum_{s'\in\Scal}  \phi(s,a,s') V_{f;h+1}(s')\colon (s,a) \in \Scal \times \Acal, f\in\Hcal\Big\}.
\end{align*} We first provide the result for the finite dimensional case i.e. when $\Phi_h \subset \Vcal \subset \R^d$ for all $h \in [H]$. \begin{corollary}[Finite Dimensional Linear Mixture Model]
  \label{cor:linearmm_finite}
  Suppose MDP $\Mcal$ is a \emph{linear Mixture Model} for some Hilbert space $\Vcal \subset \R^d$.
   Assume $\sup_{\theta \in \Hcal_h, h \in [H]}\norm{\theta}_2 \leq B_W$ and $\sup_{x\in \Phi_h, h \in [H]}\norm{x}_2 \leq B_X$  for some $B_X, B_W \geq 1$. 
  Fix $\delta \in (0,1/3)$ and $\eps \in (0,H)$. There exists an appropriate setting of batch sample size $m$, number of iteration $T$ and confidence radius $R$ such that with probability at least
   $1-\delta$, $\Cref{alg:general}$ returns a hypothesis $f$ such that $V^\star(s_0) - V^{\pi_{f}}(s_0) 
\leq \eps$ using at most \[
  c_1 \frac{d^3 H^6 \ln(1/\delta)}{\eps^2} \cdot \Big(\ln \big(c_2 \frac{d^3 H^7 B_X^2B_W^2 \ln(1/\delta) }{\eps^2}\big)  \Big)^5   
\] trajectories for some absolute constant $c_1,c_2$.
\end{corollary}  
In comparison, \cite{modi2019sample} has sample complexity $\widetilde O(d^2 H^2/\epsilon^2 \log(1/\delta))$. To prove this, we will prove a more general sample complexity result for the infinite dimensional RKHS case. We omit proof of \Cref{cor:linearmm_finite} since it follows same as proof of \Cref{cor:linearqv_finite}.

\begin{corollary}[RKHS linear mixture model]
  \label{cor:linearmm_detailed}
  Suppose MDP $\Mcal$ is a \emph{linear Mixture Model}.        Assume $\sup_{\theta \in \Hcal_h, h \in [H]}\norm{\theta}_2 \leq B_W$ and $\sup_{x\in \Phi_h, h \in [H]}\norm{x}_2 \leq B_X$. 
  Fix $\delta \in (0,1/3)$, batch sample size $m$, and define: \begin{align*}
      \widetilde d_{m}(\Phi) &= \max_{h \in [H]}\widetilde \gamma \Big( \frac{1}{8 B_W^2 m}; \Phi_h \Big) \cdot \nu_h\\
      \widetilde d_{m}(\Xcal) &= \widetilde \gamma\Big(\frac{256H^2\widetilde d_{m}(\Phi)}{B_W^2m} ; \Xcal\Big),
  \end{align*} where $\nu_h =  \ln\left(1 + 3B_X B_W \sqrt{m\widetilde \gamma \Big( \frac{1}{8 B_W^2 m}; \Phi_h \Big) }\right)$.
  
   Set parameters as: $\epso = (12H/\sqrt{m})\sqrt{\widetilde d_{m}(\Xcal)\cdot\widetilde d_{m}(\Phi)} \cdot
   \sqrt{\ln\big((\widetilde d_{m}(\Xcal) H)/\delta\big)}$ and $T = \widetilde d_{m}(\Xcal)$.
  With probability greater
  than $1-\delta$, $\Cref{alg:general}$ uses at most
  $mH \widetilde d_{m}(\Xcal)$ trajectories and returns a hypothesis $f$          \[
V^\star(s_0) - V^{\pi_{f}}(s_0) 
\leq 96H^2 \frac{\sqrt{\widetilde d_{m}(\Xcal) \cdot \widetilde d_{m}(\Phi)} \cdot
v}{\sqrt{m}}\, .
    \] where $v =  \sqrt{\ln\big((\widetilde d_{m}(\Xcal) H)/\delta\big)}$.
\end{corollary}

\begin{proof}
  First, using \Cref{cor:concentration-rkhs-linear} and \Cref{lemma:hoeffding}, we get that for any distribution $\mu$ over $\Scal\times \Acal \times \Scal$ and for any $\delta \in (0,1)$, with probability of at least $1-\delta$ over choice of an i.i.d. sample $\Dcal\sim \mu^m$ of size $m$, for all $g = (\theta_0, \ldots, \theta_{H-1})\in \Hcal$ (note that $\Lcal_{\mu}(g)$ only depends on $\theta_h$ for distribution $\mu$ over observed transitions $o_h = (r_h,s_h, a_h, s_{h+1})$ at timestep $h$.) \begin{align*}
      \Abs{\Lcal_{\Dcal}(g) - \Lcal_{\mu}(g) }&\leq \frac{4}{\sqrt{m}}  + 2H \sqrt{\frac{ 2\widetilde{\gamma}_m \ln\left( 1+ 3 B_X B_W \sqrt{\widetilde{\gamma}_m m} \right) + 2\ln(1/\delta)  }{m}} + \sqrt{2}H \sqrt{\frac{\ln(1/\delta)}{m}}\\
      &= \frac{4 + 2H\sqrt{2\widetilde{\gamma}_m \ln\left( 1+ 3 B_X B_W \sqrt{\widetilde{\gamma}_m m} \right) + 2\ln(1/\delta) } + \sqrt{2}H \sqrt{\ln(1/\delta)}}{\sqrt{m}}\\
          &\leq \frac{16H\sqrt{\widetilde{\gamma}_m \ln\left( 1+ 3 B_X B_W \sqrt{\widetilde{\gamma}_m m} \right)} \cdot \sqrt{ \ln(1/\delta) }}{\sqrt{m}}
  \end{align*} where we have used that $\ln(1/\delta) > 1$ and $\widetilde{\gamma}_m= \max_{h\in [H]}\widetilde \gamma( 1/ (8 B_W^2 m); \Phi_h )$ (as defined in \Cref{eq:crossinginfo}). Define \begin{align*}
      \widetilde d_{m}(\Phi) := \widetilde{\gamma}_m \ln\left( 1+ 3 B_X B_W \sqrt{\widetilde{\gamma}_m m} \right)
  \end{align*}
  This satisfies our \Cref{assume:linear_regret} with \begin{align*}
          \epsg(m, \Hcal) &= \frac{16H\sqrt{\widetilde d_{m}(\Phi)}}{\sqrt{m}}\\
          \conf(\delta) &= \sqrt{ \ln(1/\delta) }
      \end{align*} Substituting this in \Cref{thm:main-result} gives the result \begin{align*}
          \widetilde d_m(\Xcal) &= \widetilde \gamma\Big(\epsg^2(m, \Hcal)/B_W^2 ; \Xcal\Big)\\
          &= \widetilde \gamma\Big(256H^2\widetilde d_{m}(\Phi)/mB_W^2 ; \Xcal\Big)\\
          V^\star(s_0) - V^{\pi_{f_t}}(s_0) 
&\leq 6H \sqrt{\widetilde d_m} \cdot \epsg(m,\Hcal)\cdot
\conf\big(\delta/(\widetilde d_m H)\big)\\
&= 96H^2 \frac{\sqrt{\widetilde d_{m}(\Xcal) \cdot \widetilde d_{m}(\Phi)} \cdot
\sqrt{\ln\big((\widetilde d_{m}(\Xcal) H)/\delta\big)}}{\sqrt{m}}
      \end{align*}
\end{proof}

\subsubsection{Low Occupancy Complexity} 
Recall the low occupancy complexity model in \Cref{def:low_op}.
\begin{corollary}[Low Occupancy Complexity]
  \label{cor:lowoccup_detailed}
  Suppose $\Hcal$ has \emph{low occupancy complexity}.     Assume $\sup_{f \in \Hcal_h, h \in [H]}\norm{W_h(f)}_2 \leq B_W$. 
   Fix $\delta \in (0,1/3)$, 
   batch sample size $m$, and define:
 \[
 \widetilde d_m(\Xcal) = \widetilde \gamma\Big(\frac{8H^2 \big(1+ \ln(|\Hcal|)\big)}{mB_W^2}  ; \Xcal\Big).
 \]
 Set
   $T = \widetilde d_m(\Xcal)$ and 
   $\epso = (2\sqrt{2}H/\sqrt{m}) \cdot \sqrt{\widetilde d_m(\Xcal)} \cdot \sqrt{1+ \ln\big(|\Hcal|\big)} \cdot \sqrt{\ln\big(\widetilde d_m(\Xcal) H\big) + \ln\big(1/\delta\big)}$. 
 With probability greater
   than $1-\delta$, $\Cref{alg:general}$ uses at most
   $mH \widetilde d_m(\Xcal)$ trajectories and returns a hypothesis $f$ such
   that:
     \[
 V^\star(s_0) - V^{\pi_{f}}(s_0) 
 \leq 12\sqrt{2}H^2 \frac{\sqrt{\widetilde d_{m}(\Xcal) } \cdot \sqrt{1 + \ln\big(|\Hcal|\big)}}{\sqrt{m}} \cdot v,
     \] where $v = \sqrt{\ln\big(\widetilde d_{m}(\Xcal) H \big) + \ln\big(1/\delta\big)}$.
\end{corollary}

\begin{proof}
  First, using \Cref{lemma:hoeffding}, we get that for any distribution $\mu$ over $\Scal\times \Acal \times \Scal$ and for any $\delta \in (0,1)$, with probability of at least $1-\delta$ over choice of an i.i.d. sample $\Dcal\sim \mu^m$ of size $m$, for all $g\in \Hcal$ \begin{align*}
      \Abs{\Lcal_{\Dcal}(g) - \Lcal_{\mu}(g) }&\leq 2\sqrt{2}H \sqrt{\frac{\ln(|\Hcal|/\delta)}{m}}\\
      &\leq 2\sqrt{2}H \sqrt{\frac{\ln(|e\Hcal|/\delta)}{m}}\\
      &= 2\sqrt{2}H \sqrt{\frac{1 + \ln(|\Hcal|) + \ln(1/\delta)}{m}}\\
      &\leq 2\sqrt{2}H \sqrt{\frac{1 + \ln(|\Hcal|)}{m}} \cdot \sqrt{\ln(1/\delta)}
  \end{align*} This satisfies our \Cref{assume:linear_regret} with \begin{align*}
          \epsg(m, \Hcal) &= 2\sqrt{2}H \sqrt{\frac{1 + \ln(|\Hcal|)}{m}}\\
          \conf(\delta) &= \sqrt{\ln(1/\delta)}
      \end{align*} Substituting this in \Cref{thm:main-result} gives the result \begin{align*}
          \widetilde d_m(\Xcal) &= \widetilde \gamma\Big(\epsg^2(m, \Hcal)/B_W^2 ; \Xcal\Big)\\
          &= \widetilde \gamma\Big(\frac{8H^2(1 + \ln(|\Hcal|))}{mB_W^2} ; \Xcal\Big)\\
          V^\star(s_0) - V^{\pi_{f_t}}(s_0) 
&\leq 6H \sqrt{\widetilde d_m(\Xcal)} \cdot \epsg(m,\Hcal)\cdot
\conf\big(\delta/(\widetilde d_m(\Xcal) H)\big)\\
&= 12\sqrt{2}H^2 \frac{\sqrt{\widetilde d_{m}(\Xcal) } \cdot \sqrt{1 + \ln\big(|\Hcal|\big)}\cdot
\sqrt{\ln\big((\widetilde d_{m}(\Xcal) H)/\delta\big)}}{\sqrt{m}}
      \end{align*}
\end{proof}

\subsubsection{Finite Bellman Rank} 
In this section, we will prove sample complexity bounds for MDPs with finite Bellman Rank introduced in \cite{jiang2016contextual} (also defined as $V$-Bellman rank in \Cref{subsec:bellman}).

\begin{corollary}[Bellman Rank]
  \label{cor:bellmanrank}
  For a given MDP $\Mcal$, suppose a hypothesis class $\Hcal$ has \emph{Bellman rank} $d$. Assume $\sup_{f \in \Hcal_h, h \in [H]}\norm{W_h(f)}_2 \leq B_W$ and $\sup_{f \in \Hcal, h \in [H]} \norm{X_h(f)} \leq B_X$ for some $B_W, B_X \geq 1$. Fix $\delta \in (0,1/3)$ and $\eps \in (0,H)$. There exists an appropriate setting of batch sample size $m$, number of iteration $T$ and confidence radius $R$ such that with probability at least
  $1-\delta$, $\Cref{alg:general}$ returns a hypothesis $f$ such that $V^\star(s_0) - V^{\pi_{f}}(s_0) 
\leq \eps$ using at most \[
 c_1 \frac{d^2 H^7 |\Acal| (1 + \ln(|\Hcal|))}{\eps^2} \cdot \ln^3 \Big( \frac{c_2 d^2H^7 |\Acal| B_W^2 B_X^2 (1 + \ln(|\Hcal|))}{\delta \eps^2} \Big) 
\] trajectories for some absolute constant $c_1,c_2$.
\end{corollary}
Note that in comparison, \cite{jiang2016contextual} has sample complexity $\widetilde O(d^2 H^5 |\Acal|/\epsilon^2 \log(1/\delta) )$. We now present the proof.
\begin{proof}
  First, as observed in \cite{jiang2016contextual}[Lemma 14], we get that for any distribution $\mu$ over $\Scal\times \Acal \times \Scal$ and for any $\delta \in (0,1)$, with probability of at least $1-\delta$ over choice of an i.i.d. sample $\Dcal\sim \mu^m$ of size $m$, for all $g\in \Hcal$ \begin{align*}
      \Abs{\Lcal_{\Dcal}(g) - \Lcal_{\mu}(g) }&\leq \sqrt{\frac{8|\Acal|H^2 \ln(|\Hcal|/\delta)}{m}} + \frac{2H|\Acal|\ln(|\Hcal|/\delta)}{m}\\
      &\leq 4\sqrt{2}H \sqrt{|\Acal| }\sqrt{\frac{\ln(|e\Hcal|/\delta)}{m}}\\
      &= 4\sqrt{2}H \sqrt{|\Acal| }\sqrt{\frac{1 + \ln(|\Hcal|) + \ln(1/\delta)}{m}}\\
      &\leq 4\sqrt{2}H \sqrt{|\Acal| }\sqrt{\frac{1 + \ln(|\Hcal|)}{m}} \cdot \sqrt{\ln(1/\delta)}
  \end{align*} where the second inequality holds as long as $m > 2H|\Acal|\ln(|\Hcal|/\delta)$. This satisfies our \Cref{assume:linear_regret} with \begin{align*}
          \epsg(m, \Hcal) &= 4\sqrt{2}H \sqrt{|\Acal| }\sqrt{\frac{1 + \ln(|\Hcal|)}{m}}\\
          \conf(\delta) &= \sqrt{\ln(1/\delta)}
      \end{align*} Substituting this in \Cref{thm:main-result} gives the result \begin{align*}
          \widetilde d_m(\Xcal) &= \widetilde \gamma\Big(\epsg^2(m, \Hcal)/B_W^2 ; \Xcal\Big)\\
          &= \widetilde \gamma\Big(\frac{32H^2 |\Acal|(1 + \ln(|\Hcal|))}{mB_W^2} ; \Xcal\Big)\\
          &\leq H\Big(3d \ln \Big(1 + 3 m B^2_WB_X^2\Big) + 1\Big)\\
          & \leq 4dH \ln \Big(4 m B^2_WB_X^2\Big)
      \end{align*} where the second last step follows from \Cref{lemma:crit-gain}. Substituting $\epsg$ and $\conf$ in \Cref{thm:main-result} also gives
          \begin{align*}
          &V^\star(s_0) - V^{\pi_{f_t}}(s_0) \\
&\leq 6H \sqrt{\widetilde d_m(\Xcal)} \cdot \epsg(m,\Hcal)\cdot
\conf\big(\delta/(\widetilde d_m(\Xcal) H)\big)\\
&= 24\sqrt{2}H^2 \sqrt{|\Acal|} \frac{\sqrt{4dH \ln \Big(4 m B^2_WB_X^2\Big) } \cdot \sqrt{1 + \ln\big(|\Hcal|\big)}\cdot
\sqrt{\ln\big((4dH^2 \ln \Big(4 m B^2_WB_X^2\Big)/\delta\big)}}{\sqrt{m}}
      \end{align*} To get $\eps$-optimal policy, we have to set \begin{align*}
        m &\geq \frac{4608 dH^5 |\Acal| \ln \Big(4 m B^2_WB_X^2\Big) \cdot (1 + \ln (|\Hcal|) \cdot \ln\big((4dH^2 \ln \Big(4 m B^2_WB_X^2\Big)/\delta\big)}{\eps^2} 
      \end{align*} Further simplifying the RHS, we can write it as \begin{align*}
        \frac{4608 dH^5 |\Acal| (1 + \ln(|\Hcal|)) \cdot \ln^2 \Big(16dH^2mB_W^2 B_X^2/\delta\Big)}{\eps^2}
      \end{align*} Using \Cref{lemma:log-dominance} for $\alpha = 2$, $a = 4608 dH^5 |\Acal| (1 + \ln(|\Hcal|))/\eps^2$, $b = 16dH^2B_W^2 B_X^2/\delta$ and $c= 9$, we get that \begin{align*}
        m &= \frac{41472 dH^5 |\Acal| (1 + \ln(|\Hcal|))}{\eps^2} \ln^2 \Big( \frac{663552 d^2H^7 |\Acal| B_W^2 B_X^2 (1 + \ln(|\Hcal|))}{\delta \eps^2} \Big)\\
        \ln \Big(4 m B^2_WB_X^2\Big) &= 3 \ln \Big( \frac{663552 d^2H^7 |\Acal| B_W^2 B_X^2 (1 + \ln(|\Hcal|))}{\delta \eps^2} \Big)
      \end{align*} Substituting this in the expression above for $\widetilde d_{m}(\Xcal)$ and setting this upper bound to $T$, we get \begin{align*}
        T&= 12dH \ln \Big( \frac{663552 d^2H^7 |\Acal| B_W^2 B_X^2 (1 + \ln(|\Hcal|))}{\delta \eps^2} \Big)
    \end{align*} Since, we use on policy estimation, i.e., $\pi_{est} = U(\Acal)$ for all $t$, the trajectory complexity  is $mTH$ which completes the proof.
\end{proof}



\section{Extended Bilinear Classes}
\label{sec:extensions}

While Bilinear Classes captures most existing models, in this section, we discuss several straightforward extensions of it to incorporate additional models such as Kernelized Nonlinear Regulator (KNR), generalized linear Bellman complete model, and Witness Rank.

Consider two nonlinear monotone transformations $\xi: \mathbb{R}\mapsto \mathbb{R}$, $\zeta:\mathbb{R}\mapsto\mathbb{R}$, and a set of discriminator classes $\left\{\Fcal_h\right\}_{h=0}^{H-1}$ where $\Fcal_h \subset \Scal\times\Acal\times\Scal \mapsto \mathbb{R}$. Denote $\Fcal$ as the union of all discriminators $\Fcal_h$ from $h = 0$ to $H-1$. We extend \classname~to the following new definition, \emph{Generalized Bilinear Class}. 

\begin{definition}[Generalized \classname]
\label{def:general_linear_regret}
Consider an MDP $\Mcal$, a hypothesis class $\Hcal$,
a discrepancy function
$\ell_f:\R \times \Scal\times\Acal\times\Scal\times\Hcal \times \Fcal \rightarrow \R$ (defined
for $f\in\Hcal$), a set of estimation
policies $\Pi_{\mathrm{est}}=\{\piest(f):f\in\Hcal\}$, and two non-decreasing functions $\xi,\zeta:\mathbb{R}\mapsto \mathbb{R}$ with $\xi(0) = 0, \zeta(0) = 0$, and discriminator classes $\{\Fcal_h\}_{h=0}^{H-1}$. 

We say $(\Hcal, \ell_f, \Pi,\Mcal)$ is (\emph{implicitly}) a
\emph{Generalized Bilinear Class} if 
$\Hcal$ is realizable in $\Mcal$ and if 
there exist functions $W_h: \Hcal\times\Hcal \to \Vcal$
and $X_h: \Hcal \to \Vcal$ for some Hilbert space $\Vcal$, such that
the following two properties hold for all $f\in \Hcal$ and $h \in [H]$:
\begin{enumerate} 
  \item 
We have:
  \begin{align}
   \Abs{\E_{a_{0:h} \sim \pi_f}\big[ Q_{h,f}(s_h,a_h) -  r(s_h,a_h) - V_{h+1,f}(s_{h+1})\big]} \leq
 \xi\left(  \Abs{\langle W_h(f)-W_h(f^\star), X_h(f)\rangle}  \right)   \label{eq:assume1_nonlinear}
  \end{align}
  \item 
    The policy $\piest(f)$
   and discrepancy measure $\ell_{f}(o_h, g, v)$ can be used for
   estimation in the following sense: for any $g\in \Hcal$, we have that (here $o_h = (r_h, s_h, a_h, s_{h+1})$ is the \say{observed transition info})
  \begin{align}
  \Abs{ \max_{\nu\in\Fcal_h} \E_{a_{0:h-1} \sim \pi_f} \E_{a_h \sim \piest} \big[
  \ell_{f}(o_h, g, \nu) \big]} \geq \zeta\left( \Abs{\langle W_h(g) - W_h(f^\star), X_h(f)\rangle}\right).
  \label{eq:assume2_nonlinear}
  \end{align} Typically, $\pi_{\textrm{est}}(f)$ will be either the
  uniform distribution on $\mathcal{A}$ or $\pi_f$ itself; in the
  latter case, we refer to the estimation strategy as being on-policy.
  \item We have ${\E_{a_{0:h-1} \sim \pi_f} \E_{a_h \sim \piest} \big[
  \ell_{f}(o_h, f^\star, \nu) \big]} = 0,\forall \nu \in \Fcal_h$.
\end{enumerate} 
We also define $\Xcal_h := \{X_h(f)\colon f\in \Hcal\}$ and   $\Xcal := \{\Xcal_h: h \in [H]\}$. 
\end{definition}

Below we dive into the details of the the new definition and the examples it captures, we first see how this new definition generalizes \classname. To see that, note that we just need to set $\xi$ and $\zeta$ to be identity function, and set the discriminator classes $\Fcal_h = \emptyset$ for all $h\in [H]$ (i.e. ignore $\nu$ in the discrepancy measure $\ell_f$).

We make the following assumptions on the two nonlinear transformations. 
We assume the slope of $\zeta$ is lower bounded, and $\xi$ is non-decreasing and concave. Similar assumption has been used in generalized linear bandit model (e.g, \cite{russo2014learning}).
\begin{assumption}\label{assume:bounded_slope}
For $\zeta$, we assume $\zeta(0) = 0$ and $\zeta$ is continuously differentiable,  and
\begin{align*}
&{ \min_{f,g,h} \zeta'\left(\langle W_h(g) - W_h(f^\star), X_h(f)  \rangle \right)} \geq \beta \in \mathbb{R}^+.
\end{align*} For $\xi$, we assume $\xi(0) = 0$, and $\xi$ is concave and non-decreasing. 
\end{assumption}

We again rely on a reduction to supervised learning style generalization error by extending \cref{assume:linear_regret} to the following new assumption such that it now includes the additional function class $\Fcal_h$.

We denote the expectation of the function $\ell_f(\cdot, g, \nu)$ under distribution $\mu$ over $\R \times \Scal \times \Acal \times \Scal$ by \begin{equation*}
  \Lcal_{\mu,f}(g, \nu) = \E_{o\sim \mu} [\ell_{f}(o, g, \nu)]
\end{equation*} For a set $\Dcal \subset \mathbb{R}\times \Scal \times \Acal \times
\Scal$, we will also use $\Dcal$ to represent the uniform
distribution over this set.

\begin{assumption}[Ability to Generalize]
  \label{assume:linear_regret_discriminators} We assume there exists functions $\epsg(m,\Hcal,\Fcal)$ and $\conf(\delta)$ such that for any distribution $\mu$ over $ \R \times \Scal\times \Acal \times \Scal$ and for any $\delta \in (0,1/2)$, with probability of at least $1-\delta$ over choice of an i.i.d. sample $\Dcal\sim \mu^m$ of size $m$, 
  \[
  \sup_{g\in\Hcal} \Abs{ \max_{\nu\in\Fcal} \Lcal_{\Dcal,f}(g, \nu) -  \max_{\nu\in\Fcal}\Lcal_{\mu,f}(g,\nu) }\leq  \epsg(m, \Hcal, \Fcal) \cdot \conf(\delta)
   \]
 \end{assumption}
 One simple example of the $\epsg(m,\Hcal,\Fcal)$ is when $\Hcal$ and $\Fcal$ are both discrete, $\epsg(m,\Hcal,\Fcal)$ will scale in the order of $\widetilde{O}\left( \sqrt{ \ln(|\Hcal| |\Fcal|)   / m }  \right)$ via standard uniform convergence analysis.

With the above assumptions, we can show that our algorithm achieves the following regret. 
\begin{theorem}\label{thm:main_g_blinear}
For Generalized Bilinear Class under \Cref{assume:bounded_slope}, setting parameters properly, we have that with probability at least $1-\delta$:
\begin{align*}
V^{\star} - V^{\pi}(s_0)  & \leq 
H \xi\left(\Big( 1 + \sqrt{\widetilde \gamma(\lambda, \Xcal)} \cdot \conf\left(\frac{\delta}{\widetilde \gamma(\lambda, \Xcal) H}\right) / \beta \Big) \cdot \epsg(m,\Hcal,\Fcal)\right).
\end{align*} Furthermore, if $\xi$ is differentiable and has slope being upper bounded, i.e., $\exists \alpha\in\mathbb{R}^+$ such that ${ \max_{f,g,h} \xi'\left( \langle W_h(g) - W_h(f^\star), X_h(f)  \rangle \right)} \leq \alpha$, then we have:
\begin{align*}
V^{\star} - V^{\pi}(s_0)  & \leq  \alpha H \left(  \Big( 1 + \sqrt{\widetilde \gamma(\lambda, \Xcal)} \cdot \conf\left(\frac{\delta}{\widetilde \gamma(\lambda, \Xcal) H}\right) / \beta \Big) \cdot \epsg(m,\Hcal,\Fcal)      \right).  
\end{align*}
\end{theorem}
The proof of the above theorem largely follows the proof of \Cref{thm:main-result}, and is deferred to \Cref{sec:monotone}. 

\subsection{Kernelized Nonlinear Regulator (KNR)}
\label{subsec:knr}
In this section, we show how the above definition captures KNR \citep{kakade2020information} which we define next. We note that neither Bellman rank nor Witness rank could capture KNR directly. Specifically, since $\phi(s,a)$ could be nonlinear transformation and reward could be arbitrary (except being bounded in $[0,1]$), it is not possible to leverage model-free approaches to solve KNR as the value functions and Q functions of a KNR could be too complicated to be captured by function classes with bounded complexity.
\begin{definition}[Kernelized Nonlinear Regulator]
  Given features $\phi: \Scal\times\Acal\to \Vcal$ with $\Vcal$ being some Hilbert space, we say a MDP $\Mcal$ is a Kernelized Nonlinear Regulator (KNR) if it admits the following transition function: 
\begin{align*}
s_{h+1} = U^\star_h \phi(s_h, a_h) + \epsilon, \epsilon \sim \mathcal{N}(0, \sigma^2 I),  
\end{align*} where $U^\star_h$ is a linear operator $\Vcal \mapsto \mathbb{R}^{d_s}$.
\end{definition}
While \cite{kakade2020information} considered arbitrary  unbounded reward function,  for analysis simplicity, we assume bounded reward, i.e., $r(s,a) \in [0,1]$ for all $s,a$, but otherwise it could be arbitrary. We assume $\Scal\subset \mathbb{R}^{d_s}$ and $\|U^\star_h\|_2 := \sup_{x\in \Vcal: \|x\|_2 \leq 1} \| U^\star_h x  \|_2 \leq B_U$. We can define the hypothesis class $\Hcal_h$ as follows: \[
  \Hcal_h = \{ U \in \Vcal\mapsto \mathbb{R}^{d_s}: \|U\|_2 \leq B_U \}
  \] for all $h \in [H]$. We define the discrepancy function $\ell_f$ as follows, for $g := \{ U_0,U_1,\dots, U_{H-1} \}$ with $U_h\in \Hcal_h$ and observed transition info $o_h = (r_h, s_h, a_h, s_{h+1})$:
\begin{align*}
\ell_f(o_h, g) :=   \left\| U_h \phi(s_h,a_h) - s_{h+1} \right\|^2_2 - c,
\end{align*} where $c = \E_{x\sim \mathcal{N}(0,\sigma^2 I)} \| x \|_2^2$.   Note that in this example we set $\Fcal_h = \emptyset$ for all $h\in [H]$, thus for notation simplicity, we drop the discriminator notation from the discrepancy function.

\begin{lemma}[KNR $\implies$ Bilinear Class]
  Consider a MDP $\Mcal$ which is a Kernelized Nonlinear Regulator. Then, for the hypothesis class $\Hcal$, discrepancy function $\ell_f$ defined above and on-policy estimation policies $\piest(f) = \pi_f$, $(\Hcal, \ell_f, \Pi_{\mathrm{est}},\Mcal)$ is (\emph{implicitly}) a \emph{Generalized Bilinear Class}.
\end{lemma}
\begin{proof}  
We follow on-policy strategy and set discriminator classes to be empty, i.e., we set $\pi_{est} = \pi_f$, and $\Fcal_h = \emptyset$ for all $h\in [H]$.  Thus, we have for observed transition info $o_h = (r_h, s_h, a_h, s_{h+1})$:
\begin{align*}
&\E_{a_{0:h-1} \sim \pi_f} \E_{a_h \sim \piest} \big[
  \ell_{f}(o_h, g) \big] \\
& =  \E_{a_{0:h}\sim \pi_f} \left \| U_h\phi(s_h,a_h) - s_{h+1}  \right\|_2^2 - c \\
& = \E_{s_h,a_h\sim d^{\pi_f}_h} \E_{\epsilon\sim \mathcal{N}(0,\sigma^2 I)}  \left\| U_h \phi(s_h,a_h) - U_h^\star \phi(s_h,a_h) - \epsilon \right\|_2^2 - c \\
& = \E_{s_h,a_h\sim d^{\pi_f}_h} \left\| (U_h - U^\star_h) \phi(s_h,a_h) \right\|_2^2 + \E_{\epsilon\sim \mathcal{N}(0,\sigma^2 I)} \| \epsilon \|_2^2 - c \\
& = \E_{s_h,a_h\sim d^{\pi_f}_h} \left\| (U_h - U^\star_h) \phi(s_h,a_h) \right\|_2^2 \\
& = \trace\left(  \E_{s_h,a_h\sim d^{\pi_f}_h} \phi(s_h,a_h)\phi(s_h,a_h)^{\top} \left( (U_h - U_h^\star)^{\top} (U_h - U_h^\star) \right)      \right) \\
& = \left\langle \text{vec}\left( (U_h - U_h^\star)^{\top} (U_h - U_h^\star) \right), \text{vec}\left(  \E_{s_h,a_h\sim d^{\pi_f}_h} \phi(s_h,a_h)\phi(s_h,a_h)^{\top}\right)     \right\rangle
\end{align*} where we use the fact that $\E_{s'\sim P_h(\cdot | s_h,a_h)} s' = U_h^\star \phi(s_h,a_h)$, and we use $\text{vec}$ to represent the operator of vectorizing a matrix by stacking its columns into a long vector.  Also using the definition of $c$, it is easy to verify that $\E_{a_{0:h-1} \sim \pi_f} \E_{a_h \sim \piest} \big[
  \ell_{f}(o_h, g) \big] = 0$.


On the other hand, for Bellman error, use the fact that one step immediate reward is bounded in $[0,1]$, $Q_{h,f}(s_h,a_h) = r(s_h,a_h) + \mathbb{E}_{s'\sim P_{h,f}(\cdot | s_h,a_h)} V_{h+1,f}(s')$ (since $Q_{h,f}$ and $V_{h,f}$ are the corresponding optimal Q and V functions for model $f\in\Hcal$),  we immediately have:
\begin{align*}
&\Abs{\E_{a_{0:h} \sim \pi_f}\big[ Q_{h,f}(s_h,a_h) -  r(s_h,a_h) - V_{h+1,f}(s_{h+1})\big]} \\
&= \Abs{\E_{a_{0:h} \sim \pi_f}\big[ \E_{s'\sim P_{h,f}(\cdot | s_h,a_h)} V_{h+1,f}(s')  - \E_{s'\sim P_h(\cdot | s_h,a_h)} V_{h+1,f}(s')\big]}\\
& \leq H {\E_{a_{0:h} \sim \pi_f} \| P_{h,f}(\cdot | s_h,a_h) - P_{h}(\cdot | s_h,a_h)  \|_1   } \\
&  = 2 H {\E_{a_{0:h} \sim \pi_f} \| P_{h,f}(\cdot | s_h,a_h) - P_{h}(\cdot | s_h,a_h)  \|_{TV}   } \\
& = \frac{2H}{\sigma} \E_{s_h,a_h\sim d^{\pi_f}_h} \left\| (U_h - U_h^\star) \phi(s_h,a_h)  \right\|_{2} \\
& \leq \frac{2H}{\sigma} \sqrt{ \E_{s_h,a_h\sim d^{\pi_f}_h} \| (U_h - U_h^\star ) \phi(s_h,a_h)  \|_2^2   } \\
& \leq \frac{2H}{\sigma}  \sqrt{  \trace\left( \E_{s_h,a_h\sim d^{\pi_f}_h} \phi(s_h,a_h)\phi(s_h,a_h)^{\top} (U_h - U^\star_h)^{\top} (U_h - U_h^\star)  \right)        }\\
&= \frac{2H}{\sigma} \sqrt{\left\langle \text{vec}\left( (U_h - U_h^\star)^{\top} (U_h - U_h^\star) \right), \text{vec}\left(  \E_{s_h,a_h\sim d^{\pi_f}_h} \phi(s_h,a_h)\phi(s_h,a_h)^{\top}\right)    \right\rangle}
\end{align*}

To this end, we can verify that the generalized \classname~captures KNR as follows. We set $\zeta(x) = x$, i.e., $\zeta$ being identity and $\beta = 1$, $\xi(x) = H \sqrt{ x  } / \sigma$ where we see that $\xi(x)$ is a concave and non-decreasing function with $\xi(0) = 0$, $W_h(f) = \text{vec}\left( (U_h - U_h^\star)^{\top}(U_h-U_h^\star) \right)$ (note $W_h(f^\star) = 0$), and $X_h(f) = \text{vec}\left( \E_{s_h,a_h\sim \pi_f} \phi(s_h,a_h)\phi(s_h,a_h)^{\top} \right)$.
\end{proof}

\subsection{Generalized Linear Bellman Complete}

We first introduce the generalized linear Bellman complete model, and then we show how our framework captures it. 

\begin{definition}[Generalized Linear Bellman Complete]
Given a hypothesis class $\Hcal$ with $\Hcal_h := \{\sigma( \theta_h^{\top} \phi(s,a) ): \|\theta_h\|_2 \leq W\}$ where $\sigma: \mathbb{R}\mapsto \mathbb{R}^+$ is some inverse link function, we call it generalized linear Bellman complete model is if we have Bellman Completeness  for $\Hcal$, i.e., there exists $\Tcal_h: \Vcal\mapsto\Vcal$, such that for all $(\theta_0,\dots, \theta_{H-1})$ and $h \in [H]$, we have:
\begin{align*}
\sigma\left( \Tcal_h(\theta_{h+1})^{\top} \phi(s,a) \right) = r(s,a) + \mathbb{E}_{s'\sim P_h(s,a)} \max_{a'\in\Acal} \sigma\left( \theta_{h+1}^{\top} \phi(s',a') \right),
 \end{align*} and $\sigma(\Tcal_h(\theta_{h+1})^{\top} \phi(s,a)) \in \Hcal_h$.
\end{definition}

Let us define discriminators $\Fcal_h := \left\{  f - f' : f\in\Hcal_h, f'\in\Hcal_h \right\}$. Note that the Bellman complete assumption indicates the following. For any $f := \{ \theta_0,\dots, \theta_{H-1} \}$, we have $ \sigma( \Tcal_h(\theta_{h+1})^{\top} \phi(\cdot ,\cdot) ) -   \sigma( \theta_{h}^{\top} \phi(\cdot ,\cdot) )    \in \Fcal_h$ and $\sigma( \theta_{h}^{\top} \phi(\cdot ,\cdot) ) - \sigma( \Tcal_h(\theta_{h+1})^{\top} \phi(\cdot ,\cdot) )   \in \Fcal_h$.

\begin{assumption}\label{assume:sigma-conditions}
  We assume that inverse link function $\sigma$ is non-decreasing and the slope of $\sigma$ is bounded. I.e., for all $x \in \mathbb{R}$, $\sigma'(x) \in [ a, b ]$ for some $0 \le a \le b$.
\end{assumption}

Under this assumption (also used in \cite{wang2019optimism}), we can show that \Cref{def:general_linear_regret} captures the generalized linear Bellman complete model. 

First we will define the discrepancy function $\ell_f$ as follows. For $g: = \{\theta_0,\dots, \theta_{H-1}\}$ and $\nu\in\Hcal$, and the observed transition info $o_h = (r_h,s_h,a_h,s_{h+1})$, define $\ell_f(o_h, g, \nu)$ as:
\begin{align*}
\ell_f(o_h, g,\nu) =  \nu(s_h,a_h) \left( \sigma( \theta_{h}^{\top}  \phi(s_h,a_h) ) - r_h - \max_{a'} \theta_{h+1}^{\top} \phi(s_{h+1},a')   \right).
\end{align*} Note that $\mathbb{E}_{a_{0:h-1}\sim \pi_f} \mathbb{E}_{a_h\sim \pi_{est}} \ell_f(o_h, f^\star, \nu) = 0$ for all $\nu\in\Fcal$ due to the Bellman complete assumption.

\begin{lemma}[Generalized Linear Bellman Complete $\implies$ Bilinear Class]
  Consider a MDP $\Mcal$ and hypothesis class $\Hcal$ which is a Generalized Linear Bellman Complete model. Then, for discrepancy function $\ell_f$, discriminator class $\Fcal_h$ defined above and on-policy estimation policies $\piest(f) = \pi_f$, $(\Hcal, \ell_f, \Pi_{\mathrm{est}},\Mcal)$ is (\emph{implicitly}) a \emph{Generalized Bilinear Class}.
\end{lemma}
\begin{proof}
 

Setting $\pi_{est} = \pi_f$, adding expectation with respect to $s_h, a_h, r_h, s_{h+1}$ under the roll-in policy $\pi_f$, we get:
\begin{align*}&
\max_{\nu\in\Fcal_h} \mathbb{E}_{a_{0:h}\sim \pi_f} \left[ \ell_f(o_h, g, \nu) \right]  \\
& = \max_{\nu\in\Fcal_h}\mathbb{E}_{a_{0:h}\sim \pi_f} \nu(s_h,a_h)\left( \sigma( \theta_{h}^{\top}  \phi(s_h,a_h) ) - r_h - \max_{a'} \sigma( \theta_{h+1}^{\top} \phi(s_{h+1},a') )  \right) \\
& = \max_{\nu\in\Fcal_h} \mathbb{E}_{s_h,a_h\sim \pi_f} \nu(s_h,a_h)\left( \sigma(\theta_h^{\top} \phi(s_h,a_h)) - r(s_h,a_h) - \mathbb{E}_{s_{h+1}\sim P_h(s_h,a_h)} \max_{a'} \sigma(\theta_{h+1}^{\top} \phi(s_{h+1},a') )   \right) \\
& = \max_{\nu\in\Fcal_h}\mathbb{E}_{s_h,a_h\sim \pi_f}  \nu(s_h,a_h)\Big( \sigma(\theta_h^{\top} \phi(s_h,a_h)) - \sigma\big( \Tcal_{h}(\theta_{h+1})^{\top} \phi(s_h,a_h) \big)   \Big)^2\\
& \geq \mathbb{E}_{s_h,a_h\sim \pi_f} \Big(\sigma(\theta_h^{\top} \phi(s_h,a_h)) - \sigma\big( \Tcal_{h}(\theta_{h+1})^{\top} \phi(s_h,a_h) \big)  \Big)^2,
\end{align*} where the third equality uses the generalized linear Bellman complete assumption, and the first inequality uses the fact that $ \sigma(\theta_h^{\top}\phi(s_h,a_h)) - \sigma(\Tcal_h(\theta_{h+1})^{\top} \phi(s_h,a_h)) \in \Fcal$. Now we continue with the property of the inverse link function as follows.
\begin{align*}
& \mathbb{E}_{s_h,a_h\sim \pi_f} \Big(\sigma(\theta_h^{\top} \phi(s_h,a_h)) - \sigma\left( \Tcal_{h}(\theta_{h+1})^{\top} \phi(s_h,a_h) \right)   \Big)^2 \\
& \geq a \mathbb{E}_{s_h,a_h\sim \pi_f} \left(  (\theta_h - \Tcal_h(\theta_{h+1}))^{\top} \phi(s_h,a_h)  \right)^2 \\
& = a \trace\left( \mathbb{E}_{s_h,a_h\sim \pi_f} \phi(s_h,a_h)\phi(s_h,a_h)^{\top}  (\theta_h - \Tcal_{h}(\theta_{h+1}))    (\theta_h - \Tcal_{h}(\theta_{h+1}))^{\top}    \right) \\
& = a \left\langle   \text{vec}\left( (\theta_h - \Tcal_{h}(\theta_{h+1}))    (\theta_h - \Tcal_{h}(\theta_{h+1}))^{\top}  \right), \text{vec}\left( \mathbb{E}_{s_h,a_h\sim \pi_f} \phi(s_h,a_h)\phi(s_h,a_h)^{\top} \right)         \right\rangle
\end{align*} where the first inequality above uses mean value theorem and $\sigma'(x) \geq a, \forall x$  (\Cref{assume:sigma-conditions}). Thus, we can conclude that:
\begin{align*}
&\max_{\nu\in\Fcal_h} \mathbb{E}_{a_{0:h}\sim \pi_f} \left[ \ell_f(o_h, g, \nu) \right] \\
& \geq   \left\langle   \text{vec}\left( (\theta_h - \Tcal_{h}(\theta_{h+1}))    (\theta_h - \Tcal_{h}(\theta_{h+1}))^{\top}  \right), \text{vec}\left( \mathbb{E}_{s_h,a_h\sim \pi_f} \phi(s_h,a_h)\phi(s_h,a_h)^{\top} \right)         \right\rangle.
\end{align*}
The above is captured by \Cref{eq:assume2_nonlinear} with $\zeta(x) = a x$ being a linear function. 

Now we consider upper bounding the Bellman error. Denote $f := \{\theta_0,\dots, \theta_{H-1}\}$. We have
\begin{align*}
& \left\lvert \mathbb{E}_{a_{0:h} \sim \pi_f } [ Q_{h,f}(s_h,a_h) - r_h - V_{h+1,f}(s_{h+1}) ]  \right\rvert \\
& = \left\lvert \mathbb{E}_{a_{0:h} \sim \pi_f} \left[ Q_{h,f}(s_h,a_h) - r(s_h, a_h) - \mathbb{E}_{s_{h+1}\sim P_h(s_h,a_h)} V_{h+1,f}(s_{h+1})  \right]      \right\rvert \\
& = \left\lvert \mathbb{E}_{a_{0:h}\sim \pi_f} \left[ \sigma(\theta_h^{\top} \phi(s_h,a_h)) - \sigma\left( \Tcal_{h}(\theta_{h+1})^{\top} \phi(s_h,a_h) \right)  \right]   \right\rvert \\
& \leq \sqrt{ \mathbb{E}_{a_{0:h}\sim \pi_f} \left( \sigma(\theta_h^{\top} \phi(s_h,a_h)) - \sigma\left( \Tcal_{h}(\theta_{h+1})^{\top} \phi(s_h,a_h) \right)    \right)^2    } \\
& \leq b \sqrt{  \mathbb{E}_{a_{0:h}\sim \pi_f} \left(  (\theta_h - \Tcal_h(\theta_{h+1})^{\top} \phi(s_{h},a_h)   \right)^2  } \\
&  = b \sqrt{ \trace\left( \mathbb{E}_{s_h,a_h\sim \pi_f} \phi(s_h,a_h)\phi(s_h,a_h)^{\top} \left( \theta_h - \Tcal_h(\theta_{h+1}) \right)\left( \theta_h - \Tcal_h(\theta_{h+1}) \right)^{\top}   \right)   } \\
& = b \sqrt{\left\langle \text{vec}\left(  \left( \theta_h - \Tcal_h(\theta_{h+1}) \right)\left( \theta_h - \Tcal_h(\theta_{h+1}) \right)^{\top}  \right), \text{vec}\left( \mathbb{E}_{s_h,a_h\sim \pi_f} \phi(s_h,a_h)\phi(s_h,a_h)^{\top}   \right)   \right\rangle},
\end{align*} where the first inequality above uses Jensen's inequality and the second inequality uses mean value theorem and the fact that $\sigma'(x) \leq b, \forall x$ (\Cref{assume:sigma-conditions}). Thus, we see that the condition in \Cref{eq:assume1_nonlinear} captures this case with $\xi(x) = b\sqrt{ x}$, $X_h(f) = \text{vec}\left( \mathbb{E}_{s_h,a_h\sim \pi_f} \phi(s_h,a_h)\phi(s_h,a_h)^{\top}   \right)   $, and $W_h(f) =  \text{vec}\left(  \left( \theta_h - \Tcal_h(\theta_{h+1}) \right)\left( \theta_h - \Tcal_h(\theta_{h+1}) \right)^{\top}  \right)$ (note that by the Bellman completeness condition, we have $W_h(f^\star) = 0$).

Thus, we have shown that generalized linear MDP is captured by \Cref{def:general_linear_regret} with $\zeta(x) = a x$ and $\xi(x) = b\sqrt{x}$. Note that $\zeta$ and $\xi$ satisfies \Cref{assume:bounded_slope}. 
\end{proof}

\subsection{Witness Rank}

Witness rank \citep{sun2018model} is a structural complexity that captures model-based RL with $\Hcal_h$ being the hypothesis space containing transitions $P_h$.   Witness rank uses a discriminator class $\Fcal_h \subset \Scal\times\Acal\times\Scal\mapsto \mathbb{R}$ (with $\Fcal_{h}$ being symmetric and rich enough to capture $V_{h+1,f}$ for all $f\in\Hcal$) to capture the discrepancy between models.  Here we focus on model-based setting and $\Hcal_h$ contains possible transitions $g_h:\Scal\times\Acal\mapsto \Delta(\Scal)$ and the realizability assumption implies that $P_h\in\Hcal_h$. For simplicity here, we assume reward function is known.

\begin{definition}
  We say a MDP $\Mcal$ has witness rank $d$ if given two models $f\in\Hcal$ and $g\in\Hcal$, there exists $X_h: \Hcal\mapsto \mathbb{R}^d$ and $W_h: \Hcal\mapsto \mathbb{R}^d$ such that:
  \begin{align*}
  &\max_{v\in \Fcal_{h}} \mathbb{E}_{a_{0:h-1}\sim \pi_{f}} \E_{a_h\sim \pi_{g}}\left[ \E_{s'\sim g_h(\cdot | s_h,a_h)} v(s_h,a_h,s') -  \E_{s'\sim P_h(\cdot | s_h,a_h)} v(s_h,a_h,s')    \right]   \geq  \langle W_h(g), X_h(f)  \rangle, \\
  & \kappa \cdot \mathbb{E}_{a_{0:h-1}\sim \pi_{f}} \E_{a_h\sim \pi_{g}}\left[ \E_{s'\sim g_h(\cdot | s_h,a_h)} V_{{h+1}, g}(s') -  \E_{s'\sim P_h(\cdot | s_h,a_h)} V_{h+1,g}(s')    \right] \leq  \langle W_h(g), X_h(f)  \rangle,
  \end{align*} where $\kappa \in (0,1]$.
\end{definition} Similar to Bellman rank, the algorithm and analysis from \cite{sun2018model} rely on $d$ being finite.  Below we show how \cref{def:general_linear_regret} naturally captures witness rank.

We define $\ell_f$ as follows:
\begin{align*}
\ell_f(o, g, v) = \frac{ \mathbf{1}\{ a = \pi_g(s) \} }{1/A} \left[ \E_{\tilde{s}\sim g_h(\cdot | s,a)} v(s,a,\tilde{s}) -   v(s,a,s')    \right]. 
\end{align*} 

\begin{lemma}[Finite Witness Rank $\implies$ Bilinear Class]
  Consider a MDP $\Mcal$ which has finite Witness Rank. Then, for the hypothesis class $\Hcal$, discrepancy function $\ell_f$ defined above and uniform estimation policies $\piest(f) = U(\Acal)$, $(\Hcal, \ell_f, \Pi_{\mathrm{est}},\Mcal)$ is a Bilinear Class with Discrepancy Family.
\end{lemma}
\begin{proof}
  Recall that we denote $f^\star$ as the ground truth which in this case means the ground truth transition $P$. This implies that $\langle W_h(f^\star), X_h(f) \rangle = 0$ for any $f\in\Hcal$. This allows us to write the above formulation as:
  \begin{align*}
  &\max_{v\in \Fcal_{h}} \mathbb{E}_{a_{0:h-1}\sim \pi_{f}} \E_{a_h\sim \pi_{g}}\left[ \E_{s'\sim g_h(\cdot | s_h,a_h)} v(s_h,a_h,s') -  \E_{s'\sim P_h(\cdot | s_h,a_h)} v(s_h,a_h,s')    \right]   \\
  &\quad  \leq \langle W_h(g) - W_h(f^\star), X_h(f) \rangle.
  \end{align*}
  
  For the Bellman error part, since this is the model-based setting, we have $Q_{h,f}(s_h,a_h) = r(s_h,a_h) + \mathbb{E}_{s'\sim g_h(s_h,a_h)} V_{h+1,f}(s')$. Thus, we have:
  \begin{align*}
  &\left\lvert \mathbb{E}_{a_{0:h}\sim \pi_f} \left[ Q_{h,f}(s_h,a_h) - r_h - V_{h+1,f}(s_{h+1}) \right]   \right\rvert \\
  & =  \left\lvert \mathbb{E}_{a_{0:h}\sim \pi_f} \left[ \mathbb{E}_{{s'}\sim g_h(s_h,a_h)} V_{h+1, f}({s'})  - \mathbb{E}_{s'\sim P_h(s_h,a_h)} V_{h+1,f}(s') \right]   \right\rvert 
  \end{align*}
  
  Therefore, it is a \emph{Bilinear Class with Discrepancy Family} with $\zeta(x) = x$ and $\xi(x) = \frac{1}{\kappa} x$.

  
\end{proof}

Here we also give an example for $\epsilon_{gen}$.  For $\Fcal$ and $\Hcal$ with bounded complexity (e.g., discrete $\Fcal$ and discrete $\Hcal$), we still achieve the generalization error, i.e.,  for all $f$, for all $g\in\Hcal$, with probability at least $1-\delta$:
  \begin{align}
  \Big\lvert  & \max_{v\in \Fcal_{h+1}} \mathbb{E}_{a_{0:h-1}\sim \pi_{f}} \E_{a_h\sim \pi_{g}}\left[ \E_{s'\sim g_h(\cdot | s_h,a_h)} v(s_h,a_h,s') -  \E_{s'\sim P_h(\cdot | s_h,a_h)} v(s_h,a_h,s')    \right] \nonumber  \\ 
  & \quad - \max_{v\in\Fcal} \frac{1}{m} \sum_{i=1}^N \ell_f(r_i, s_i,a_i,s_i', g, v)  \Big\rvert \leq  \sqrt{ \frac{ 2A \ln( 2|\Hcal||\Fcal|/\delta ) }{m} } + \frac{2A \ln(2|\Hcal||\Fcal|/\delta)}{3m},  \nonumber\\
  & \leq \left(\sqrt{ \frac{ 2A \ln( 2|\Hcal||\Fcal| ) }{m} } + \frac{2A \ln(2|\Hcal||\Fcal|)}{3m} \right) \cdot \ln(1/\delta)
  \label{eq:eps_gen_witness}\\
  & := \epsg(m, \Hcal, \Fcal) \cdot \conf(\delta),
  \end{align} where $s_i \sim d^{\pi_f}_h, a_i \sim U(\Acal), s'_i \sim P_h(\cdot | s,a)$, and the inequality assumes that $\ln(1/\delta) \geq 1$ (see Lemma 12 from \cite{sun2018model} for derivation).

\subsubsection{Factored MDP}
\label{subsec:factoredMDP}

For completeness, we consider factored MDP as a special example here. We refer readers to \cite{sun2018model} for a detailed treatment of how witness rank capturing factored MDP.

We consider state space $\Scal \subset \Ocal^{d}$ where $\Ocal$ is a discrete set and we denote $s[i]$ as the i-th entry of the state $s$.  For each dimension $i$, we denote $\text{pa}_{i} \subset[d]$ as the set of state dimensions that directly influences state dimension $i$ (we call them the parent set of the i-th dimension).  In factored MDP, the transition is governed by the following factorized transition:
\begin{align*}
\forall h, s,s'\in \Scal, a\in\Acal, \; P_h(s' | s,a ) = \prod_{i=1}^d P_h^{(i)}\left( s'[i] \vert s[\text{pa}_i], a \right) 
\end{align*} where $P^{(i)}$ is the condition distribution that governs the transition from $s[ \text{pa}_i ], a $ to $s'[i]$.  Here, we do not assume any structure on reward function. 

Note that the complexity of the problem is captured by the number of parameters in the transition operator, which in this case is equal to $\sum_{i=1}^d H A |\Ocal|^{1+|\text{pa}_i|}$. Note that when the parent set $\text{pa}_i$ is not too big (e.g., a constant that is independent of $d$), this complexity could be exponentially smaller than $|\Ocal|^{d}$ for a MDP that does not have factorized structure. 

The hypothesis class $\Hcal$ contains possible transitions. In factored MDP, we design the following discrepancy function $\ell_f(o_h, g, v)$ at $h$ for observed transition info $o_h = (r_h, s_h, a_h, s_{h+1})$,
\begin{align*}
\ell_f(o_h, g, v) = \mathbb{E}_{\tilde{s}\sim g_h(\cdot | s_h,a_h)} v(s_h,a_h,\tilde{s}) - v(s_h,a_h,s_{h+1}).
\end{align*} With $\pi_{est} = U(\Acal)$, and discriminators $\Fcal_h = \{w_1+w_2 \dots + w_d: w_i \in \mathcal{W}_i  \}$ where $\mathcal{W}_i = \left\{\mathcal{O}^{|\text{pa}_i| \times\Acal \times \mathcal{O}} \mapsto \{-1,1\}\right\}$, \cite{sun2018model} (Proposition 24) shows that there exists $X_h:\mathcal{H}\mapsto \mathbb{R}^{L}$ and $W_h: \mathcal{H}\mapsto \mathbb{R}^{L}$ with $L = \sum_{i=1}^d K |\Ocal|^{|\text{pa}_i|}$, such that:
\begin{align*}
\left\lvert \max_{v\in \Fcal_{h+1}} \mathbb{E}_{s_h\sim \pi_f, a_h\sim \pi_{est}} \left[ \ell_f(o_h, g, v) \right]  \right\rvert =  \left\lvert\left\langle W_h(g) - W_h(f^\star), X_h(f) \right\rangle \right\rvert,
\end{align*} where we use the fact that $\left\langle W_h(f^\star), X_h(f) \right\rangle = 0$ for all $f\in\Hcal$ due to the design of the discrepancy function $\ell_f$. Moreover, \cite{sun2018model} (Lemma 26) also proved that:
\begin{align*}
&\left\lvert \mathbb{E}_{a_{0:h}\sim \pi_f}\left[ Q_{h,f}(s_h,a_h) - r(s_h,a_h) - V_{h+1,f}(s_{h+1})  \right]\right\rvert \\
&  \leq  A H  \left\lvert \max_{v\in \Fcal_{h+1}} \mathbb{E}_{s_h\sim \pi_f, a_h\sim \pi_{est}} \left[ \ell_f(o_h, g, v) \right]  \right\rvert  = AH\left\lvert\left\langle W_h(g) - W_h(f^\star), X_h(f) \right\rangle \right\rvert.
\end{align*}
Thus factored MDP is captured by \Cref{def:general_linear_regret} where $\zeta(x) = x$, and $\xi(s) = AH x$.  \cite{sun2018model} shows that value function based approaches including Olive \cite{jiang2017contextual} in worst case requires $2^{H}$ many samples to solve factored MDPs, which in turn indicates that the prior structural complexity such as Bellman rank and Bellman Eluder \citep{jin2021bellman} must be exponential in H.

\section{Conclusion}
\label{sec:conclusion}
We presented a new framework, Bilinear Classes, together with a new sample efficient algorithm, \algname. A key emphasis of the new class and algorithm is that many learnable RL models can be analyzed with the same algorithm and proof.

Our framework is more general than existing ones, and incorporates a large number of RL models with function approximation. Along with the general framework, our work also introduces several important new models including linear $Q^\star/V^\star$, RKHS Bellman complete, RKHS linear mixture models and low occupancy complexity. Our rates are non-parametric and depend on a new information theoretic quantity---critical information gain, which is an analog to the critical radius from non-parametric statistics. With this new quantity, our results extend prior finite-dimension results to infinite dimensional RKHS setting.  

The Bilinear Classes can also be flexibly extended to cover many other examples including Witness Rank and Kernelized Nonlinear Regulator. We believe many other models (potentially even those proposed in the future) can be analyzed via extensions of the Bilinear Classes.

\section*{Acknowledgements}
We thank Chi Jin and Qinghua Liu for discussions on Section \ref{sec:extensions} including the generalized linear bellman complete model. We thank Akshay Krishnamurthy for a discussion regarding $Q/V$-Bellman rank.
\bibliography{section/references,section/refs}
\bibliographystyle{plainnat}

\newpage
\tableofcontents
\newpage
\appendix

\section{Additional Examples of Bilinear Classes}
\label{sec:examples}
We now include some other examples of Bilinear Classes in addition to ones discussed in \Cref{sec:new_examples}.

\subsection{FLAMBE / Feature Selection}
We consider the feature selection setting introduced by \cite{agarwal2020flambe}. 
\begin{definition}[Feature Selection]
  We say a MDP $\Mcal$ is \emph{low rank feature selection} model if there exists (unknown) functions $\mu_h^\star: \Scal\mapsto \Vcal$ and (unknown) features $\phi^\star: \Scal\times\Acal\mapsto \Vcal$, $\psi^\star: \Scal \times \Acal$  for some Hilbert space $\Vcal$ such that for all $h\in [H]$ and $(s,a,s') \in \Scal \times \Acal \times \Scal$\begin{align*}
    P_h(s' | s,a) = \mu_h^\star(s')^{\top} \phi^\star(s,a) \label{eqn:flambe}
  \end{align*}
\end{definition}
Note that unlike linear MDP model where $\phi^\star$ is assumed to be known, here $\phi^\star$ is unknown to the learner. We use a function class $\Phi \subset \Scal\times\Acal\mapsto \Vcal$ to capture $\phi^\star$, i.e., we assume realizability $\phi^\star \in \Phi$.

We can define our function class $\Hcal = \Hcal_0 \times \ldots, \Hcal_{H-1}$ as follows \[
    \Hcal_h = \{ w^{\top} \phi(s,a): \|w\|_2 \leq B_W, \phi\in\Phi \}
\]   to capture the optimal value $Q^\star$. Note that since $\phi^\star \in \Phi$, and the optimal Q function is linear with respect to feature $\phi^\star(s,a)$, we immediately have $f^\star := \{Q^\star_0,\dots, Q^\star_{H-1}\} \in \Hcal$.
We define the following discrepancy function
$\ell_f$ (in this case the discrepancy function does not depend on
$f$) for any $g\in \Hcal$ and for observed transition info $o_h = (r_h, s_h, a_h, s_{h+1})$:
\begin{align*}
  \ell_f(o_h,g) = \frac{\mathbf{1}\{a_h = \pi_g(s)\}}{1/A}\left( Q_{h,g}(s_h, a_h) - r_h - V_{h+1,g}(s_{h+1})    \right)\, ,
\end{align*}
\begin{lemma}
  Consider a MDP $\Mcal$ which is a low rank feature selection model. Then, for the hypothesis class $\Hcal$, discrepancy function $\ell_f$ defined above and on-policy estimation policies $\piest(f) = U(\Acal)$, $(\Hcal, \ell_f, \Pi_{\mathrm{est}},\Mcal)$ is (\emph{implicitly}) a \emph{Bilinear Class}.
\end{lemma}
\begin{proof}
  Note that for $g = f$, we have that (here observed transition info $o_h = (r_h, s_h, a_h, s_{h+1})$)\[
  \E_{s_h\sim d^{\pi_f}} \E_{a_h \sim U(\Acal)} \left[ \ell(o_h, f)  \right] =  \E_{s_h,a_h,s_{h+1}\sim d^{\pi_f}} \left[ Q_{h,f}(s_h, a_h) - r(s_h,a_h) - V_{h+1,f}(s_{h+1})  \right]
\] Therefore, to prove that this is a Bilinear Class, we will show that a stronger \say{equality} version of \Cref{eq:assume2} holds (which will also prove \Cref{eq:assume1} holds). Observe that for any $h$, \begin{align*}
  &\E_{s_h \sim d^{\pi_f}} \E_{a_h \sim U(\Acal)} \left[ \ell_f(o_h, g)  \right]\\
  &= \E_{s_h\sim d^{\pi_f}}\left[ Q_{h,g}(s_h, \pi_g(s_h)) -   r(s_h, \pi_g(s_h)) -  \E_{s_{h+1} \sim P_h(\cdot |s_h,\pi_g(s_h))} V_{h+1,g}(s_{h+1})     \right]  \\
  & =  \E_{s_{h-1},a_{h-1}\sim d^{\pi_f}} \int_{s} (\mu_h^\star(s))^{\top} \phi^\star(s_{h-1},a_{h-1})   \left[ V_{h,g}(s) -   r(s, \pi_g(s)) -  \E_{s' \sim P_h(\cdot |s,\pi_g(s))} V_{h+1,g}(s')\right] ds   \\
  & =  \E_{s_{h-1},a_{h-1}\sim d^{\pi_f}} \phi^\star(s_{h-1},a_{h-1}) ^{\top} \int_{s} \mu_h^\star(s)   \left[ V_{h,g}(s) -   r(s, \pi_g(s)) -  \E_{s' \sim P_h(\cdot |s,\pi_g(s))} V_{h+1,g}(s')\right] ds \\
  & = \left\langle  W_h(g) - W_h(f^\star), X_h(f)\right\rangle 
  \end{align*}
  where\begin{align*}
    & X_h(f) := \E_{s_{h-1},a_{h-1}\sim d^{\pi_f}} \big[\phi^\star(s_{h-1},a_{h-1})\big],\\
 &W_h(f):=\int_{s\in \Scal} \mu^\star_h(s) \big( V_{h,f}(s) -  r(s,\pi_f(s)) - \E_{s' \sim P_h(\cdot |s,\pi_f(s))}[V_{h+1,f}(s')]\big) ds.
  \end{align*}
  Observe that $W_h(f^\star) = 0$ due to Bellman optimality condition for $V^\star$ and $\pi^\star$. 
\end{proof}

\subsection{$Q^\star$ irrelevance Aggregation / $Q^\star$ state Aggregation}
We now consider the $Q^\star$ irrelevance aggregation model introduced in \cite{lihong2009disaggregation}.
\begin{definition}[$Q^\star$ irrelevance aggregation model]
We say a MDP $\Mcal$ is the $Q^\star$ irrelevance aggregation model if there exists known function $\zeta: \Scal \mapsto \Vcal$ such that for all states $s_1, s_2 \in \Scal$\[
\zeta(s_1) = \zeta(s_2) \implies Q^\star(s_1, a) = Q^\star(s_2, a)  \quad \forall~ a \in \Acal  
\]
\end{definition}
Let $\Zcal = \{\zeta(s): s \in \Scal\}$. Here, our hypothesis class $\Hcal = \Hcal_0 \times \ldots, \Hcal_{H-1}$ is a set of linear functions i.e. for all $h\in [H]$, the set $\Hcal_h$ is defined as: \begin{align*}
  \Big\{ (w, \theta) \in \R^{|\Zcal| \times |\Acal|} \times \R^{|\Zcal|}\colon\max_{a\in \Acal} w^{\top} \phi(s,a)  = \theta^{\top}\psi(s)\, , ~\forall s \in \Scal \Big\}.
\end{align*} We also define the following discrepancy function
$\ell_f$ (in this case the discrepancy function does not depend on
$f$), for hypothesis $g = \{( w_h, \theta_h )\}_{h=0}^{H-1}$ and observed transition info $o_h = (r_h, s_h, a_h, s_{h+1})$:
\begin{align*}
  \ell_f(o_h,g)  &= Q_{h,g}(s_h,a_h) -  r_h - V_{h+1,g}(s_{h+1})\\
  &= w_h^{\top}\phi(s_h,a_h) - r_h - \theta_{h+1}^{\top}\psi(s_{h+1})\, .
\end{align*}  
\begin{lemma}
  Consider a MDP $\Mcal$ which is the $Q^\star$ irrelevance aggregation model. Then, for the hypothesis class $\Hcal$, discrepancy function $\ell_f$ defined above and on-policy estimation policies $\piest(f) = \pi_f$, $(\Hcal, \ell_f, \Pi_{\mathrm{est}},\Mcal)$ is (\emph{implicitly}) a \emph{Bilinear Class}.
\end{lemma}\begin{proof}
  To prove that this is implicitly a Bilinear Class, we will reduce this into linear $Q^\star/ V^\star$ model (\Cref{def:linearqv}). Let $\Zcal = \{\zeta(s): s \in \Scal\}$. Now, we construct one hot representation functions $\phi: \Scal \times \Acal \mapsto \{0,1\}^{|\Zcal|\times |\Acal|}$ and $\psi: \Scal \mapsto \{0,1\}^{|\Zcal|}$ where\begin{align*}
    \Big(\phi(s,a)\Big)_{z,a'} &= \ind(\zeta(s) = z)\cdot \ind(a = a')\\
    \Big(\psi(s)\Big)_z &= \ind(\zeta(s) = z)
\end{align*}
Then, it clear that we can construct $w^\star \in \R^{|\Zcal|\times |\Acal| }$ and $\theta^\star \in \R^{|\Zcal|}$ as follows: \begin{align*}
  (w^\star)_{z,a} &= Q^\star(s,a)\\
  (\theta^\star)_s  &= V^\star(s)
\end{align*} such that the following holds: \begin{align*}
    (w^\star)^\top \phi(s,a)  &= Q^\star(s,a)\\
    (\theta^\star)^\top \psi(s)  &= V^\star(s)
\end{align*}This is linear $Q^\star/ V^\star$ model (\Cref{def:linearqv}) and therefore is a \classname.
\end{proof}

\subsection{Linear Quadratic Regulator}
In this subsection, we prove that Linear Quadratic Regulators (LQR) forms a Bilinear Class. Note that even though LQR has small bellman rank, the corresponding algorithm in \cite{jiang2017contextual} has action dependence in sample complexity unlike our algorithm which does not have a dependence on number of actions. Here we consider $\Scal \subset \R^d$ and $\Acal \subset \R^K$.
\begin{definition}[Linear Quadratic Regulator]
  We say a MDP $\Mcal$ is a finite-horizon discrete-time Linear Quadratic Regulator if there exists (unknown) $A\in \R^{d \times d}$, (unknown) $B \in \R^{d \times K}$ and (unknown) $Q\in \R^{d\times d}$ such that we can write the transition function and reward function as follows\begin{align*}
    s_{h+1} &= As_h + B a_h + \eps_h\\
    r_h &= s_h^\top Q s_h + a^\top_h a_h + \tau_h
  \end{align*} where noise variables $\eps_h, \tau_h$ are zero centered with $\E[\eps_h \eps_h^\top] = \Sigma$ and $\E[\tau_h^2] = \sigma^2$. 
\end{definition}
To maintain notation of fixed starting state, without loss of generality, we also assume $s_0=0$ and $a_0=0$. An important property of LQR is that for linear non stationary policies $\pi$, the value function $V^\pi$ induced is quadratic (see for e.g. \cite{jiang2017contextual}[Lemma 7] for a proof). \begin{lemma}
  If $\pi$ is a non stationary linear policy $\pi_{h}(s_h) = C_{\pi, h} x$ for some $C_{\pi, h} \in \R^{K \times d}$, then $V_h^{\pi}(s_h) = s^\top_h \Lambda_{\pi, h} s_h + O_{\pi, h}$ for some $\Lambda_{\pi, h} \in \R^{d \times d}$ and $O_{\pi, h}\in \R$. 
\end{lemma}
This allows us to define out hypothesis class $\Hcal = \Hcal_0, \ldots, \Hcal_{H-1}$ as \begin{align*}
  \Hcal_{h} = \{(C_h, \Lambda_h, O_h): C_h \in \R^{K \times d},~ \Lambda_h\in \R^{d \times d}, ~ O_h \in \R\}
\end{align*} with for any $f\in \Hcal$ \[
\pi_f(s_h) = C_{h,f} s_h, \quad V_{h,f}(s_h) = s^\top_h \Lambda_{h,f} s_h + O_{h,f}
\]
We define the following discrepancy function
$\ell_f$ for any hypothesis $g\in \Hcal$ and observed transition info $o_h = (r_h, s_h, a_h, s_{h+1})$:
\begin{align*}
  \ell_f(o_h,g) &= Q_{h,g}(s_h, a_h) - r_h -   V_{h+1,g}(s_{h+1}) \, ,\\
  &= s^\top_h \Lambda_{h,g} s_h + O_{h,g} - s_h^\top Q s_h - s^\top_h C^\top_{h,g} C_{h,g} s_h - \tau_h -   s^\top_{h+1} \Lambda_{h+1,g} s_{h+1} - O_{h+1,g} \, .
\end{align*} 
\begin{lemma}
  Consider a MDP $\Mcal$ which is a Linear Quadratic Regulator. Then, for the hypothesis class $\Hcal$, discrepancy function $\ell_f$ defined above and on-policy estimation policies $\piest(f) = \pi_f$ for $f\in \Hcal$, $(\Hcal, \ell_f, \Pi_{\mathrm{est}},\Mcal)$ is (\emph{implicitly}) a \emph{Bilinear Class}.
\end{lemma}
\begin{proof}
Note that for $g = f$, we have that (here observed transition info $o_h = (r_h, s_h, a_h, s_{h+1})$)\[
  \E_{s_h,a_h,s_{h+1}\sim d^{\pi_f}} \left[ \ell(o_h, f)  \right] =  \E_{s_h,a_h,s_{h+1}\sim d^{\pi_f}} \left[ Q_{h,f}(s_h, a_h) - r(s_h,a_h) - V_{h+1,f}(s_{h+1})  \right]
\] Therefore, to prove that this is a Bilinear Class, we will show that a stronger \say{equality} version of \Cref{eq:assume2} holds (which will also prove \Cref{eq:assume1} holds). Observe that for any $h$, \begin{align*}
  &\E_{a_h,s_h, s_{h+1} \sim d^{\pi_f}} \Big[ \ell_f(o_h, g)  \Big]\\
  &= \E_{s_h, s_{h+1}\sim d^{\pi_f}}\Big[ s^\top_h \Lambda_{h,g} s_h + O_{h,g} - s_h^\top Q s_h - s^\top_h C^\top_{h,g} C_{h,g} s_h -   s^\top_{h+1} \Lambda_{h+1,g} s_{h+1} - O_{h+1,g}     \Big]  \\
  & = \trace\Big(\big(\Lambda_{h,g} - Q - C^\top_{h,g} C_{h,g} - (A + BC_{h,g})^\top\Lambda_{h+1,g} (A + BC_{h,g})\big)  \E_{s_{h}\sim d^{\pi_f}}[s_h s^\top_h] \Big)\\
  &\quad - \trace(\Lambda_{h+1, g}\Sigma) + O_{h,g}  - O_{h+1,g}   \\
  & = \left\langle  W_h(g) - W_h(f^\star), X_h(f)\right\rangle 
  \end{align*}
  where\begin{align*}
    X_h(f) := &[\textrm{vec}(\E_{s_{h}\sim d^{\pi_f}}[s_h s^\top_h]), 1],\\
 W_h(g):= &[\textrm{vec}(\Lambda_{h,g} - Q - C^\top_{h,g} C_{h,g} - (A + BC_{h,g})^\top\Lambda_{h+1,g} (A + BC_{h,g})),\\
  &O_{h,g}  - O_{h+1,g} - \trace(\Lambda_{h+1, g}\Sigma)].
  \end{align*}
  Note that we used $\left\langle W_h(f^\star), X_h(f)\right\rangle = 0$ which follows from the bellman conditions i.e. for $a_h = C_{h,f^\star}s_h$ \begin{align*}
    & s^\top_h \Lambda_{h,f^\star} s_h + O_{h,f^\star} - \E_{s_{h+1} \sim P(s_h, a_h)}[s^\top_{h+1} \Lambda_{h+1,f^\star} s_{h+1}] - O_{h+1,f^\star} - s_h^\top Q s_h \\
    &- s^\top_h C_{h,f^\star}^\top C_{h,f^\star} s_h = 0\\
    \implies &s^\top_h \Lambda_{h,f^\star} s_h + O_{h,f^\star} - \trace\Big([(A + BC_{h,f^\star})^\top \Lambda_{h+1,f^\star} (A + BC_{h,f^\star})] s_h s_h^\top\Big) \\
    &- \trace(\Lambda_{h+1,f^\star}\Sigma) - O_{h+1,f^\star}- s_h^\top Q s_h - \trace(C_{h,f^\star}^\top C_{h,f^\star} s_h s^\top_h) = 0\\
    \implies &\trace\Big(\big(\Lambda_{h,f^\star}- Q - C^\top_{h,f^\star} C_{h,f^\star} - (A + BC_{h,f^\star})^\top \Lambda_{h+1,f^\star} (A + BC_{h,f^\star}) \big) s_h s_h^\top\Big) \\
    &+ O_{h,f^\star}  - O_{h+1,f^\star} - \trace(\Lambda_{h+1, f^\star}\Sigma) = 0
  \end{align*} Taking expectation over $s_h \sim d^{\pi_f}$ proves the claim.
\end{proof}

\subsection{Linear MDP}
\label{sec:linearmdp}
We consider the Linear MDP setting from \cite{yang2019sample,jin2019provably}. \begin{definition}[Linear MDP]
  We say a MDP $\Mcal$ is a Linear MDP with features $\phi: \Scal \times \Acal \mapsto \Vcal$, where $\Vcal$ is a Hilbert space if for all $h \in [H]$, there exists (unknown) measures $\mu_h$ over $\Scal$ and (unknown) $\theta_h \in \Vcal$, such that for any $(s,a) \in \Scal \times \Acal$, we have  \begin{equation*}
    P_h(\cdot \mid s,a) = \langle \phi(s,a), \mu_h(\cdot) \rangle, \quad r_h(s,a) = \langle \phi(s,a), \theta_h\rangle 
  \end{equation*}
\end{definition}
Here, our hypothesis class $\Hcal$ is set of linear functions with respect to $\phi$. We denote hypothesis in our hypothesis class $\Hcal$ as tuples $(\theta_0,\ldots\theta_{H-1})$, where $\theta_h\in\Vcal$. As observed in \cite{jin2019provably}[Proposition 2.3], this satisfies the conditions of Bellman Complete model (\Cref{def:linear_comp}) and therefore is also a Bilinear Class.

\subsection{Block MDP and Reactive POMDP}
Both Block MDP ~\citep{du2019provably,misra2019kinematic} and a Reactive POMDP \citep{krishnamurthy2016pac} are \emph{partially observable MDPs} (POMDPs) which can be described by a finite (unobservable) latent state space $\mathcal{S}$, a finite action space $\mathcal{A}$, and a possibly infinite but observable context space $\mathcal{X}$. The transitions can be described by two conditional probabilities. One is the latent state transition $p: \Scal \times \Acal \mapsto \triangle\left(\Scal\right)$, and the other is the context-emission function $q: \mathcal{S} \mapsto \triangle (\mathcal{X})$.

The key differences among Block MDP and Reactive POMDP are in the assumptions which we define below.
\begin{definition}[Block MDP]
  For Block MDPs, the context space $\mathcal{X}$ can be partitioned into disjoint blocks $\mathcal{X}_s$ for $s \in \mathcal{S}$, each containing the support of the conditional distributiion $q(\cdot|s)$.
\end{definition}

This assumption implies there exists a perfect decoding function $f^*: \mathcal{X} \rightarrow \mathcal{S}$, which maps contexts to their generating states.
Therefore, we have that the transition of contexts satisfies\[
P(x'|x,a) = p(f^*(x')|f^*(x),a) = e_{f^*(x')}^\top p(\cdot |f^*(x),a)
\]
where $e_{f^*(x')} \in \mathbb{R}^{|\Scal|}$ is a one-hot vector where only the entry that corresponds to $f^*(x')$ is $1$.
Note one can define $\mu^*(x') \triangleq e_{f^*(x')}$ and $\phi^*(x,a) \triangleq p(\cdot |f^*(x),a)$ as in the FLAMBE setting.
Thus, Block MDP is a subclass of FLAMBE with the Hilbert space $\Vcal$ being the $|S|$-dimensional Euclidean space.
Since FLAMBE is within our Bilinear Class, Block MDP is also within our framework.

For POMDP, assume reward is known and is a deterministic function over observations and actions and $r(x, a) \in [0,1]$. let us define belief $b_h(\cdot | \mathbf{h}_h) \in \Delta(\Scal)$ as the posterior distribution of state $s$ at time step $h$ given history $\mathbf{h}_h := x_0, a_0,\dots, x_{h-1},a_{h-1}, x_h$, i.e.,  given any state $s$, we have $b_h(s | \mathbf{h}_h) = P(s | x_0,  a_0,\dots, x_{h-1},a_{h-1}, x_h )$. Given $a_h$ and conditioned on $x_{h+1}$ being observed at $h+1$, the belief is updated based on the Bayes rule, deterministically, 
\begin{align*}
\forall s'\in\Scal: \; b_{h+1}(s' | \mathbf{h}_h, a_h, x_{h+1}) \propto   \sum_{s} b_h(s | \mathbf{h}_h)  p(s' | s, a_h) q(x_{h+1} | s'),
\end{align*} with $b_0(s | x_{0}) \propto \mu_0(s) q(x_0 | s)$, where $\mu_0\in\Delta(S)$ is the initial state distribution (in the simplified case where we have a fixed $s_0$, then $\mu_0$ is a delta distribution with all probability mass on $s_0$). 

Note that given $a_h,x_{h+1}$, the above update is deterministic, and $b_{h}(s | \mathbf{h}_h)$ is a function of history $\mathbf{h}_h$. 
Denote the deterministic Belief update procedure as $b_{h+1} = \Gamma(b_h, a_h,  x_{h+1})$.
For POMDP, the optimal policy $\pi^\star$ is a mapping from $\Delta(\Scal)$ to $\Acal$. Given a belief $b$, and an action $a$, we can define $Q_h^\star(b, a)$ backward as follows. Start with $V^\star_H(b) = 0$ for all $b\in\Delta(\Scal)$, 
\begin{align*}
Q^\star_h(b, a) =   \mathbb{E}_{s\sim b} \mathbb{E}_{x\sim q(\cdot | s)} \left[  r(x, a) +  V^\star_{h+1}\left( \Gamma(b, a, x)  \right)   \right],
\end{align*} where $V^\star_h(b) = \argmax_{a} Q^\star_h(b, a), \pi^\star_h(b) = \argmax_{a} Q^\star_h(b, a)$.

\begin{definition}[Reactive POMDP]
  For Reactive POMDPs, the optimal Q function $Q^\star_h$ is only dependent on latest observation and action, i.e.,  for all $h$, there exists $g_h^\star: \Xcal\times\Acal\mapsto [0,H]$, such that, for any given history $\mathbf{h}_h := x_0, a_0,\dots, x_{h-1},a_{h-1}, x_{h}$, we have:
   \begin{align*}
    Q^\star_h\left( b_h(\cdot | \mathbf{h}_h) , a\right) = g^\star_h(x_h, a), \forall a\in\Acal.
    \end{align*}
\end{definition}
Note that in this case, the optimal policy $\pi^\star_h$ only depends on the latest observation $x_h$, i.e., $\pi^\star_h( b(\cdot | \mathbf{h}_h) ) = \argmax_{a\in\Acal} Q^\star_h(b(\cdot|\mathbf{h}_h),a) =  \argmax_{a\in\Acal} g^\star_h(x_h, a)$.
As shown in \cite{jiang2017contextual}, Reactive POMDPs have bellman rank bounded by $|\Scal|$ which implies (see \Cref{subsec:bellman} for more detail) that Reactive POMDPs are a Bilinear Class.
\section{Proofs for Section \ref{sec:alg}}
\label{app:proofs-main-thm}
\begin{proof}[Proof of \Cref{cor:main-result-finite}]
    First, using \Cref{lemma:hoeffding}, we get that for any distribution $\mu$ over $\Scal\times \Acal \times \Scal$ and for any $\delta \in (0,1)$, with probability of at least $1-\delta$ over choice of an i.i.d. sample $\Dcal\sim \mu^m$ of size $m$, for all $g\in \Hcal$ \begin{align*}
      \Abs{\Lcal_{\Dcal}(g) - \Lcal_{\mu}(g) }&\leq 2\sqrt{2}H \sqrt{\frac{\ln(|\Hcal|/\delta)}{m}}\\
      &\leq 2\sqrt{2}H \sqrt{\frac{\ln(|e\Hcal|/\delta)}{m}}\\
      &= 2\sqrt{2}H \sqrt{\frac{1 + \ln(|\Hcal|) + \ln(1/\delta)}{m}}\\
      &\leq 2\sqrt{2}H \sqrt{\frac{1 + \ln(|\Hcal|)}{m}} \cdot \sqrt{\ln(1/\delta)}
  \end{align*} This satisfies our \Cref{assume:linear_regret} with \begin{align*}
          \epsg(m, \Hcal) &= 2\sqrt{2}H \sqrt{\frac{1 + \ln(|\Hcal|)}{m}}\\
          \conf(\delta) &= \sqrt{\ln(1/\delta)}
      \end{align*} Using this in \Cref{thm:main-result-finite}, we set \begin{align*}
        T = 4dH \ln\Big( 1 + 3B^2_X B^2_W \sqrt{m}\Big)
      \end{align*}
        Therefore, we get $\eps$-optimal policy by setting \begin{align*}
        &3H \cdot 2\sqrt{2}H \sqrt{\frac{1 + \ln(|\Hcal|)}{m}} \cdot \Big(1 + \sqrt{4dH \ln\Big( 1 + 3B^2_X B^2_W \sqrt{m}\Big)} \cdot \sqrt{\ln\frac{4dH^2 \ln\Big( 1 + 3B^2_X B^2_W \sqrt{m}\Big)}{\delta}}\leq \eps\\
        \end{align*} or equivalently by setting $m$ at least as large as
        \begin{align*}
        &\frac{720dH^5 (1 + \ln(|\Hcal|)) \ln(1 + 3B^2_X B^2_W \sqrt{m})}{\eps^2} \cdot \ln\frac{4dH^2 \ln\Big( 1 + 3B^2_X B^2_W \sqrt{m}\Big)}{\delta}\\
        &\leq \frac{720dH^5 \ln(4dH^2) (1 + \ln(|\Hcal|)) \ln^2(1 + 3B^2_X B^2_W \sqrt{m}) \ln(1/\delta)}{\eps^2}
      \end{align*} Using \Cref{lemma:log-dominance}, we get a solution for $m$ \begin{align*}
        m = \frac{6480 dH^5 \ln(4dH^2) \ln(1/\delta) (1 + \ln(|\Hcal|))}{\eps^2} \ln \Big( \frac{25920 dH^5  B^2_X B^2_W (1 + \ln(|\Hcal|)) \ln(4dH^2) \ln(1/\delta)}{\eps^2} \Big)
      \end{align*} This gives the total trajectory complexity \[
        mTH = \frac{c d^2H^7\ln(dH^2) \ln(1/\delta) (1 + \ln(|\Hcal|))}{\eps^2} \ln^2( \frac{dH B_X B_W (1 + \ln(|\Hcal|)) \ln(1/\delta)}{\eps^2})
      \] for some absolute constants $c$.
  \end{proof}
\section{An Elliptical Cover for Hilbert Spaces}
\label{sec:elliptic-cover}

The following theorem is a key technical contribution which allows us
to obtain a number of non-parametric convergence rates.

\begin{theorem}\label{lemma:elliptical_covering}
Let $\Xcal \subset \Vcal$, where $\Vcal$ is a
Hilbert space.
Suppose $T \in \mathbb{N}^+, \epsilon\in\mathbb{R}^+$; define $\Wcal \subseteq \{w \in \Vcal: \|w\|\leq
B_W\} $ for some real number $B_W$; and
suppose for all $x\in\Xcal$ that $\|x\|_2\leq B_X$.  Set $\lambda = \eps^2/(8B_W^2)$.
 There exists a set
$\Ccal\subset \Wcal$ (a cover of
$\Wcal$) such that: (i) $\log|\Ccal| \leq T \log(1+3 B_W B_X \sqrt{T}/\epsilon)$
and (ii) for all $w \in \Wcal$, there exists a $w'\in\Ccal$,
such that:
\[
\sup_{x \in \Xcal} | (w-w')\cdot x | \leq   \epsilon  \sqrt{\left( \exp\left( \frac{\gamma_T(\eps^2 / (8B_W^2))} { T}\right)-1\right)}.
\]
\end{theorem}

\begin{proof}
Let us suppose that $\Xcal$ is closed, in order for certain
maximizers (and arg-maximizers) over $\Xcal$ to
exist. If $\mathcal{X}$ is not closed, then
let us replace $\Xcal$ with the closure of $\Xcal$, which is possible
since $\Xcal$ is a bounded set.
Consider the process: Set $\Sigma_ 0 = \lambda I$ with $\lambda\in\mathbb{R}^+$.
\begin{enumerate}
\item For $t = 0, \ldots T-1$, 
\begin{enumerate}
	\item  $x_t = \argmax_{x\in\Xcal} \left\| x \right\|^2_{\Sigma_t^{-1}}$
	\item $\Sigma_{t+1} = \Sigma_t + x_t x_t^{\top}$
\end{enumerate}
\end{enumerate}
Via  \Cref{lemma:potential_argument},  we have that:
\begin{align*}
\sum_{t=0}^{T-1} \ln\left( 1 +  \left\| x_t \right\|^2_{\Sigma^{-1}_t} \right) \leq \ln\frac{\det( \Sigma_T)}{\det(\Sigma_0)}.
\end{align*}
This implies that there must exist a $t \in 0, \dots, T-1$, such that:
\begin{align*}
\ln\left( 1 + \| x_t \|^2_{\Sigma_t^{-1}} \right) \leq \frac{ \gamma_T(\lambda) }{ T },
\end{align*} which means that:
\begin{align*}
\|x_t\|^2_{\Sigma^{-1}_{t}} \leq \exp\left( \frac{ \gamma_T(\lambda) }{T} \right) - 1.
\end{align*} Note that $x_t = \argmax_{x\in\Xcal} \|x\|_{\Sigma_t^{-1}}$. Thus, we have that:
\begin{align*}
\max_{x\in \Xcal} \|x\|^2_{\Sigma_t^{-1}}  \leq \exp\left( \frac{\gamma_T(\lambda)}{T} \right) - 1.
\end{align*}

Note that the above derivation holds for any $\lambda \in
\mathbb{R}^+$. 

Define $M_T = \sum_{i=0}^{T} x_t x_t^{\top}$. Note that the range of
$M_T$, $\textrm{Range}(M_T)$ is a $T+1$-dimensional object. For an
$\eps'$-net, $\Ccal$, in $\ell_2$ distance over $B_W$-norm ball on
$\textrm{Range}(M_T)$, i.e., $\{v\in \Wcal: v\in
\textrm{Range}(M_T)\}$. With a standard covering number bound, we have that
$\ln(|\Ccal |) \leq 2T \ln\left( 1 + 2 B_W / \eps'  \right)$
(e.g. see \Cref{lem:eps_cover}). 

Fix some $w\in\Wcal$.  Denote the projection of $w$ on the the range of $M_T$
by $\overline{w}$.  Let $w' \in \Ccal$ being the closest point to
$\overline{w}$ in $\ell_2$ distance. Note that
$\|\overline{w} - w' \|_2 \leq \eps'$.  For any $x\in\Xcal$, we have:
\begin{align*}
 &\left( (w - w')^{\top} x \right)^2\leq \| w- w' \|_{\Sigma_{T}}^2 \| x\|_{\Sigma_{T}^{-1}}^2 \\
&\leq \|w- w'\|_{\Sigma_{T}}^2  ( \exp\left( \gamma_T(\lambda) / T \right) - 1) \\
& = \left( \lambda  \|w - w'  \|^2 + \left( w - w'  \right)^{\top} \left(\sum_{i=0}^{T} x_i x_i^{\top}\right) \left( w - w'  \right) \right) (  \exp\left( \gamma_T(\lambda) / T \right) - 1) \\
& = \left(\lambda  \|w - w'  \|^2 + \left( \overline{w} - w'  \right)^{\top} \left(\sum_{i=0}^{T} x_i x_i^{\top}\right) \left(\overline{w} - w'  \right)\right)  ( \exp\left( \gamma_T(\lambda) / T \right)-1)\\
& \leq \left( 4 \lambda  B_W^2 +  T \epsilon'^2  B_X^2\right)  (\exp\left( \gamma_T(\lambda) / T \right)-1),
\end{align*} where the equality in the third step uses that $(w -
w')^{\top} x_i = \left(\overline{w} - w'\right)^{\top} x_i$ for all $i
\in 0,\dots, T$.  The proof is completed choosing $\lambda =
\eps^2/(8B_W^2)$ and $(\eps' )^2= \eps^2/(2TB_X^2)$.
\end{proof}

\section{Concentration Arguments for Special Cases}
\label{app:conc-covering}


\paragraph{An application to RKHS Linear MDPs.}

Consider the RKHS linear MDP, where $\phi:\Scal\times\Acal\mapsto
\Hcal$ with $\Hcal$ being some Hilbert space.  Define $\Phi =
\{\phi(s,a): s\in\Scal, a\in\Acal\}$.

\begin{corollary}
    \label{cor:cover-rkhs}
Suppose $T\in\mathbb{N}^+$ and $\epsilon\in\mathbb{R}^+$; define $\Wcal \subseteq \{w \in \Hcal: \|w\|\leq
B_W\}$ for some real number $B_W$; and
suppose for all $\phi(s,a)\in\Phi$ that $\|\phi(s,a)\|_2\leq B_\phi$.  
There exists a set
$\Ccal\subset \Wcal$ 
such that: (i) $\log|\Ccal| \leq T \log(1+ 3 B_\phi B_W \sqrt{T} /\epsilon)$
and (ii) for all $w \in \Wcal$, there exists a $w'\in\Ccal$
such that for all distributions $d$ over
$\Scal\times\Acal\times\Scal$, we have:
\begin{align*}
&
\bigg| \E_{s,a,s' \sim d} \big[w\cdot \phi(s,a) - r(s,a) -\max_{a'} w \cdot\phi(s',a')\big]
\\&\quad-\E_{s,a,s' \sim d}\big[w'\cdot \phi(s,a) - r(s,a) -\max_{a'} w' \cdot\phi(s,a')\big] \bigg|\\
&\qquad \leq 2\epsilon \sqrt{\left( \exp\left( \frac{\gamma_T(\epsilon^2 / (8B_W^2))}{T} \right) - 1  \right)}
\end{align*}
\end{corollary}
\begin{proof}
For any distribution $d$, we seek to bound:
\begin{align*}
&\Big|
\E_{s,a,s' \sim d} \Big[w\cdot \phi(s,a) -w'\cdot \phi(s,a) -\big(\max_{a'} w \cdot\phi(s',a')
-\max_{a'} w' \cdot\phi(s,a')\big)\Big] \Big|\\
&\leq \sup_{s, a} \big|w\cdot \phi(s,a) -w'\cdot \phi(s,a)\big|+\Big|
\E_{s,a,s' \sim d} \Big[ \big(\max_{a'} w \cdot\phi(s',a')
-\max_{a'} w' \cdot\phi(s,a')\big)\Big] \Big|\\
&\leq 
\sup_{s, a} \big|w\cdot \phi(s,a) -w'\cdot \phi(s,a)\big|+
\sup_{s} \big|\sup_a w \cdot\phi(s,a)-\sup_{a} w' \cdot\phi(s,a)\big| \Big|\\
&\leq 
2\sup_{s, a} \big|w\cdot \phi(s,a) -w'\cdot \phi(s,a)\big|
\end{align*}
where the last step follows using that $|\sup_xf(x)- \sup_x g(x)| \leq
\sup_x |f(x)-g(x)|$ (which can be verified by considering both case of
the sign inside the absolute value). The proof is completed by choose
$w'$ to be closest point $\Ccal$ to $w$ and applying \Cref{lemma:elliptical_covering}.
\end{proof}

\begin{corollary}
    \label{cor:concentration-rkhs}
    Define $\Wcal =: \{w \in \Hcal: \|w\|\leq
    B_W, w^{\top} \phi(s,a) \in [0,H] \; \forall s,a\in\Scal\times\Acal\}$ for some real number $B_W$; and
    suppose for all $\phi(s,a)\in\Phi$ that $\|\phi(s,a)\|_2\leq B_\phi$. Let \[
    \ell(r, s,a,s',w) = w \cdot \phi(s,a) - r - \max_{a'} w \cdot \phi(s',a')
    \] with $r\in [0,1]$. Then, for any distribution $\mu$ over $\R \times \Scal\times \Acal \times \Scal$ and for any $\delta \in (0,1)$, with probability of at least $1-\delta$ over choice of an i.i.d. sample $\Dcal\sim \mu^m$ of size $m$, for all $w\in \Hcal$ \[
   \Abs{\Lcal_{\Dcal}(w) - \Lcal_{\mu}(w) }\leq \frac{8}{\sqrt{m}}  + 2H \sqrt{\frac{ 2\widetilde{\gamma}_m \ln\left( 1+ 3 B_\phi B_W \sqrt{\widetilde{\gamma}_m m} \right) + 2\ln(1/\delta)  }{m}}
   \] where $\widetilde{\gamma}_m= \widetilde \gamma( 1/ (8 B_W^2 m);\Phi )$ (as defined in \Cref{eq:crossinginfo}).
\end{corollary} 
\begin{proof}
 First note that for any $w\in \Wcal$, we must have:
 \begin{align*}
 \ell(r,s,a,s', w) \in [ - H-1, H],
 \end{align*} since we eliminate all $w$ such that $w^{\top}\phi(s,a) \not\in[0,H]$ for some $s,a$. 

Consider the cover $\Ccal$ from \Cref{cor:cover-rkhs}. From \Cref{lemma:hoeffding} and a union bound over all $w'\in \Ccal$, for all $w'\in\Ccal$, we have that with probability at least $1-\delta$:
\begin{align*}
\left\lvert \Lcal_{\Dcal}(w') - \Lcal_{\mu}(w')\right\rvert \leq   2H\sqrt{\frac{2\ln(|\Ccal|/\delta)}{m}}.
\end{align*} Now consider any $w\in\Wcal$, via \Cref{cor:cover-rkhs}, we know that there exists a $w'\in\Ccal$ such that:
\begin{align*}
\left\lvert \Lcal_{\mu}(w) - \Lcal_{\mu}(w')  \right\rvert \leq 2\epsilon \sqrt{\left( \exp\left( \frac{\gamma_T(\lambda)}{T} \right) - 1 \right)}. 
\end{align*}
Thus, together with the fact that \Cref{cor:cover-rkhs} holds for both $\mu$ and the uniform distribution over $\Dcal$, we get:
\begin{align*}&
\left\lvert    \Lcal_{\mu}(w)  - \Lcal_{\Dcal}(w)   \right\rvert   \leq \left\lvert \Lcal_{\mu}(w)  - \Lcal_{\mu}(w')   \right\rvert + \left\lvert    \Lcal_{\mu}(w')  - \Lcal_{\Dcal}(w')   \right\rvert + \left\lvert  \Lcal_{\Dcal}(w')    - \Lcal_{\Dcal}(w) \right\rvert \\
& \leq 4 \epsilon \sqrt{ \left( \exp\left( \frac{\gamma_T(\lambda)}{T} \right) - 1 \right)}  + 2H \sqrt{\frac{2\ln(|\Ccal|/\delta)}{m}}\\
& \leq 4 \epsilon \sqrt{ \left( \exp\left( \frac{\gamma_T(\eps^2 / (8 B_W^2))}{T} \right) - 1 \right) }+ 2H \sqrt{\frac{ 2  T \ln\left( 1+ 3 B_\phi B_W \sqrt{T} / \epsilon \right) + 2\ln(1/\delta)  }{m}}
\end{align*} 
Let us set $\epsilon = 1/\sqrt{m}$ and rearrange terms, we get:
\begin{align*}
&\left\lvert    \Lcal_{\mu}(w)  - \Lcal_{\Dcal}(w)   \right\rvert \\
& \leq \frac{4}{\sqrt{m}}  \sqrt{ \left( \exp\left( \frac{\gamma_T(1 / (8 B_W^2 m))}{T} \right) - 1 \right) } + 2H \sqrt{\frac{ 2 T \ln\left( 1+ 3 B_\phi B_W \sqrt{ T m} \right) + 2\ln(1/\delta)  }{m}}.
\end{align*}
Denote $\widetilde{\gamma}_m = T$ where $T$ is the smallest integer that satisfies $T \geq \gamma_T (1/ (8 B_W^2 m ))$.  
Thus, we have:
\begin{align*}
& \left\lvert    \Lcal_{\mu}(w)  - \Lcal_{\Dcal}(w)   \right\rvert \\
& \leq \frac{8}{\sqrt{m}} + 2H \sqrt{\frac{ 2 \widetilde{\gamma}_m \ln\left( 1+ 3 B_\phi B_W \sqrt{ \widetilde{\gamma}_m m} \right) + 2\ln(1/\delta)  }{m}},
\end{align*} where in the inequality we use $\exp\left(\frac{\gamma_T(1 / (8 B_W^2 m))}{T}  \right) -1 \leq e - 1 \leq 2$.

\end{proof}

\paragraph{An application to RKHS linear functions}
Consider features $\zeta: \Scal \times \Acal \times \Scal \mapsto \Vcal$ with $\Vcal$ being some Hilbert space. Define $Z = \{\zeta(s,a,s')\colon (s,a,s') \in \Scal \times \Acal \times \Scal \}$. 
\begin{corollary}
    \label{cor:concentration-rkhs-linear}
    Define $\Wcal =: \{w \in \Vcal: \|w\|\leq
    B_W, w^{\top} \zeta(s,a,s') \in [0,H] \; \forall s,a,s'\in\Scal\times\Acal\times\Scal\}$ for some real number $B_W$; and
    suppose for all $\zeta(s,a,s')\in Z$ that $\|\zeta(s,a,s')\|_2\leq B_\zeta$. Let \[
    \ell(r, s,a,s',w) = w \cdot \zeta(s,a,s') 
    \] Then, for any distribution $\mu$ over $\Scal\times \Acal \times \Scal$ and for any $\delta \in (0,1)$, with probability of at least $1-\delta$ over choice of an i.i.d. sample $\Dcal\sim \mu^m$ of size $m$, for all $w\in \Hcal$ \[
   \Abs{\Lcal_{\Dcal}(w) - \Lcal_{\mu}(w) }\leq \frac{4}{\sqrt{m}}  + 2H \sqrt{\frac{ 2\widetilde{\gamma}_m \ln\left( 1+ 3 B_\zeta B_W \sqrt{\widetilde{\gamma}_m m} \right) + 2\ln(1/\delta)  }{m}}
   \] where $\widetilde{\gamma}_m= \widetilde \gamma( 1/ (8 B_W^2 m); Z )$ (as defined in \Cref{eq:crossinginfo}).
\end{corollary} \begin{proof}
    The proof follows exactly as proof of \Cref{cor:concentration-rkhs}.
\end{proof}

\begin{lemma}[Covering number] \label{lem:eps_cover}For any $\epsilon >0$, the $\epsilon$-covering number of the Euclidean ball in $\mathbb{R}^d$ with radius $R\in\mathbb{R}^+$, i.e., $\mathcal{B} = \{ x\in\mathbb{R}^d: \|x\|_2 \leq R \}$, is upper bounded by $(1 + 2R/\epsilon)^d$.
\end{lemma}

\section{Generalized Bilinear Classes}
\label{sec:monotone}


Recall \Cref{def:general_linear_regret} for Generalized \classname. We next complete the proof of \Cref{thm:main_g_blinear}. 

\begin{proof}[Proof of \Cref{thm:main_g_blinear}]
First notice that a uniform convergence result similar to \Cref{lem:concentration}  still holds:
\begin{align*}
\Abs{ \max_{\nu \in\Fcal_h} \Lcal_{\Dcal_{t;h}, f_t}(g,\nu)  -  \max_{\nu \in\Fcal_h} \Lcal_{\mu_{t;h}, f_t}(g, \nu)} \leq  \epsg,
\end{align*} where $\epsg := \epsg(m, \Hcal, \Fcal)\cdot \text{conf}(\delta / (TH))$.

Also it is easy to verify that the feasibility claim similar to \Cref{lemma:feasible} holds as well since $ \max_{\nu \in\Fcal_h} \Lcal_{\mu_{t;h}, f_t}(f^\star, \nu) = 0$. The feasibility result immediately implies the optimism claimed in \Cref{lem:optimism}. While the derivation of \Cref{lemma:opt} mostly follows, we use \Cref{eq:assume1_nonlinear} rather than \Cref{eq:assume1}, which gives us the following:
\begin{align*}
V^\star - V^{\pi_{f_t}}(s_0) & \leq \sum_{h=0}^{H-1} \xi\left( \left\lvert W_h(f_t) - W_h(f^\star), X_h(f_t)   \right\rvert  \right)\\&  \leq  H \xi\left( \sum_{h=0}^{H-1} \left\lvert W_h(f_t)-W_h(f^\star) , X_h(f_t)   \right\rvert  / H  \right).
\end{align*} where the last step follows from concavity of $\xi$ (\Cref{assume:bounded_slope}) and Jensen's inequality.

To show the existence of a high quality policy, we also mainly follow the steps in the proof of \Cref{lemma:bound-iteration}. First we can verify \Cref{eq:existence}  holds due to the elliptical potential argument. This implies that for all $h$, 
\begin{align*}
\sum_{j=0}^{t-1} \bigg( \max_{\nu\in\Fcal_h}\Lcal_{\mu_{j; h}, f_{j}}(f_{t}, \nu )\bigg)^2&\leq 2 \sum_{j=0}^{t-1} \bigg( \max_{\nu\in\Fcal_h} \Lcal_{\Dcal_{j; h}, f_{j}}(f_{t},\nu)\bigg)^2 + 2 \sum_{j=0}^{t-1}\epsg^2\\
		&\leq 4T\epsg^2
\end{align*}

Thus together with \Cref{eq:assume2_nonlinear}, we have:
\begin{align*}
\sum_{j=0}^{t-1}\zeta\left( \Abs{\langle W_{h}(f_{t})-W_h(f^\star), X_{h}(f_j)\rangle}\right)^2  \leq 4T\epsg^2
\end{align*}
Note that by \Cref{assume:bounded_slope} and an application of mean-value theorem, we have:
\begin{align*}
\sum_{j=0}^{t-1} \beta^2 \Abs{\langle W_{h}(f_{t})-W_h(f^\star) , X_{h}(f_j)\rangle}^2 \leq 4T \epsg^2.
\end{align*} Thus, we have:
\begin{align*}
(W_{h}(f_{t}) - W_h(f^\star))^\top \Sigma_{t; h} (W_{h}(f_{t})-W_h(f^\star) ) \leq   4 \lambda B_W^2 +   4T \epsg^2 / \beta^2.
\end{align*} Together, we arrive:
\begin{align*}
\Abs{\langle W_{h}(f_{t})-W_h(f^\star) , X_h(f_{t})  \rangle}^2 \leq  (4\lambda B_W^2 + 4T\epsg^2 / \beta^2) \left(\exp\left( \frac{1}{T}\gamma_T(\lambda; \Xcal)\right) - 1\right)
\end{align*}
Sum over all h, we have:
\begin{align*}
\sum_{h=0}^{H-1}\Abs{\langle W_{h}(f_{t})-W_h(f^\star), X_h(f_{t})\rangle}& \leq H\sqrt{(4\lambda B_W^2 + 4T\epsg^2 / \beta^2) \left(\exp\left( \frac{1}{T}\gamma_T(\lambda; \Xcal)\right) - 1\right)}.
\end{align*}
Apply $\xi$ on both sides and use the assumption that $\xi$ is non-decreasing, we have:
\begin{align*}
&H\xi\left( \sum_{h=0}^{H-1}\Abs{\langle W_{h}(f_{t})-W_h(f^\star) , X_h(f_{t})\rangle} / H \right)\\
&  \leq  H \xi\left( \sqrt{(4\lambda B_W^2 + 4T\epsg^2 / \beta^2) \left(\exp\left( \frac{1}{T}\gamma_T(\lambda; \Xcal)\right) - 1\right)} \right)
\end{align*}
This means that there exists a $t$:
\begin{align*}
V^\star - V^{\pi_{f_t}}(s_0)  \leq H \xi\left( \sqrt{(4\lambda B_W^2 + 4T\epsg^2 / \beta^2) \left(\exp\left( \frac{1}{T}\gamma_T(\lambda; \Xcal)\right) - 1\right)} \right).
\end{align*}
Now set $\lambda = \epsg^2(m,\Hcal) / B_W^2 $, and $T \geq \widetilde{\gamma}(\lambda, \Xcal)$, we get:
\begin{align*}
V^{\star} - V^{\pi_{f_t}} & \leq H \xi\left(   \sqrt{ 4 \epsg(m,\Hcal)^2 + 4\widetilde{\gamma}(\lambda,\Xcal)\epsg^2 / \beta^2 }       \right) \\
& \leq  H \xi\left( 2\epsg(m,\Hcal) + 2\sqrt{\widetilde{\gamma}(\lambda,\Xcal)} \epsg /\beta    \right).
\end{align*}
This concludes the first part of the theorem. 

When $\xi$ is continuously differentiable, $\xi(0) = 0$,\\ and ${ \max_{f,g,h} \xi'\left( \langle W_h(g,f^\star), X_h(f)  \rangle \right)} \leq \alpha$, we simply have:
\begin{align*}
 \xi\left( 2\epsg(m,\Hcal) + 2\sqrt{\widetilde{\gamma}(\lambda,\Xcal)} \epsg /\beta    \right) \leq \alpha  \left( 2\epsg(m,\Hcal) + 2\sqrt{\widetilde{\gamma}(\lambda,\Xcal)} \epsg /\beta    \right),
\end{align*} via an application of mean-value theorem. 
This concludes the proof. 
\end{proof}

\section{Auxiliary Lemmas}
\begin{lemma}[Azuma-Hoeffding] \label{lemma:hoeffding}Let $X_1,\ldots, X_m$ be independent random variables with mean $\mu$ such that $|X_i| \leq B$ for some $B>0$ almost surely for all $i \in [m]$. Then, with probability $1-\delta$,\[
    \Abs{\frac{1}{m}\sum_{i=1}^m X_i - \mu} \leq \sqrt{2}B \sqrt{\frac{\ln(1/\delta)}{m}}
    \]
\end{lemma}

\begin{lemma}(Log Dominance Rule)
    \label{lemma:log-dominance}
    Suppose $\alpha, a, b\geq 0$ and $c \geq (1+\alpha)^\alpha$. Then, $m = ca \ln^\alpha(abc)$ is a solution to \[
      m \geq a \ln^\alpha(bm)
    \]
\end{lemma}
\begin{proof}
    First note that \begin{align*}
        &a \ln^\alpha(bm)\\
        &= a \ln^\alpha(abc \ln^\alpha(abc))\\
        &= a \left(\ln(abc) + \alpha\ln\ln(abc)\right)^\alpha\\
        &\leq a \left(\ln(abc) + \alpha\ln(abc)\right)^\alpha\\
        &= a (1+\alpha)^\alpha\ln^\alpha(abc)\\
        &\leq ca \ln^\alpha(abc)
    \end{align*}
\end{proof}

\begin{lemma}
    \label{lemma:crit-gain}
    Let $\Xcal \subset \R^d$ and $\sup_{x\in \Xcal}\norm{x}_2 \leq B_X$. Then, the maximum information gain \[
        \gamma_n(\lambda; \Xcal) \leq d \ln \left(1 + \frac{n B^2_X}{d\lambda}\right)
        \]Furthermore, the critical information gain \[
            \widetilde \gamma(\lambda; \Xcal) \leq \left\lceil 3d \ln\left(1 + \frac{3B^2_X}{\lambda}\right)\right\rceil
            \]
\end{lemma}\begin{proof}
    \[
        \gamma_n(\lambda; \Dcal) := \max_{x_0\dots x_{n-1} \in \Dcal}  \ln\det\left( \iden + \frac{1}{\lambda} \sum_{t=0}^{n-1} x_t x_t^{\top}  \right).    
    \]We have \begin{align*}
        \trace\left( \iden + \frac{1}{\lambda} \sum_{t=0}^{n-1} x_t x_t^{\top}  \right) &= d + \frac{1}{\lambda} \sum_{t=0}^{n-1} \norm{x_t}_2^2\\
        &\leq d + nB^2_X/\lambda 
    \end{align*} Therefore, using the Determinant-Trace inequality, we get the first result\begin{align*}
        \ln\det\left( \iden + \frac{1}{\lambda} \sum_{t=0}^{n-1} x_t x_t^{\top}  \right)& \leq d \ln \frac{\trace\left( \iden + \frac{1}{\lambda} \sum_{t=0}^{n-1} x_t x_t^{\top}  \right)}{d} \\
        &\leq d \ln \left(1 + \frac{n B^2_X}{d\lambda}\right)
    \end{align*} To get the second result, first note that for $n = cd \ln(1 + cB^2_X/\lambda)$ and $c=3$, \begin{align*}
        d \ln \left(1 + \frac{n B^2_X}{d\lambda}\right) &= d \ln \left(1 + \frac{cB^2_X}{\lambda} \ln(1 + cB_X^2/\lambda) \right)\\
 &\leq d \ln \left(1 + \frac{cB^2_X}{\lambda} \max\{\ln(1 + cB_X^2/\lambda), 1\} \right)\\
 &\leq d \ln \left((1 + \frac{cB^2_X}{\lambda}) \max\{\ln(1 + cB_X^2/\lambda), 1\} \right)\\
 &\leq d \left(\ln \big(1 + \frac{cB^2_X}{\lambda}\big) +  \ln\big(\max\{\ln(1 + cB_X^2/\lambda), 1\}\big)\right)\\
 &\leq d \left(\ln \big(1 + \frac{cB^2_X}{\lambda}\big) +  \ln(1 + cB_X^2/\lambda)\right)\\
 &= 2d \ln \big(1 + \frac{cB^2_X}{\lambda}\big)\\
 &\leq n
    \end{align*} where the third last step follows from $\ln(1 + cB_X^2/\lambda) \geq 0$ and $\ln(1 + cB_X^2/\lambda) \geq \ln (\ln(1 + cB_X^2/\lambda))$ and last step follows from $c = 3>2$.
\end{proof}
\section{Sample Complexity Lower Bound for RHKS Bellman Complete and Linear MDP}

\label{sec:lowerbound}
Recall that in Section~\ref{sec:rkhs_b_complete}, we show that under the assumption that $\sup_{h \in [H], \theta \in \Hcal_h}\norm{\theta}_2$ and $\sup_{x\in \Phi}\norm{x}_2$ are both bounded, and the assumption that the maximum information gain is bounded, then our algorithm finds a near-optimal policy using polynomial number of samples for RHKS Bellman Complete and Linear MDP.
One may wonder if the assumption on the maximum information gain can be removed as in the case of contextual bandits~\citep{abe2003reinforcement, foster2020beyond}.
Here we show that for the case of reinforcement learning, without the maximum information gain assumption, there is an exponential sample complexity lower bound (in the problem horizon $H$). 
Therefore, our hardness result justifies the necessity of assuming bounded maximum information gain for the case of RHKS Bellman Complete and Linear MDP.

Our hard instance is based on the binary tree instance (see~\citet{du2019good, krishnamurthy2016pac} for previous hardness results that use such a construction).
In this construction, there are $H$ levels of states, and level $h \in [H]$ contains $2^h$ distinct states.
Thus we have $|\Scal| = 2^{H} - 1$. 
We use $s_0, s_1, \ldots, s_{2^H - 2}$ to name these states.
Here, $s_0$ is the unique state in level $h = 0$, $s_1$ and $s_2$ are the two states in level $h = 1$, $s_3$, $s_4$, $s_5$ and $s_6$ are the four states in level $h= 2$, etc. 
There are two different actions, $a_1$ and $a_2$, in the MDPs.
For a state $s_i$ in level $h$ with $h < H - 1$, playing action $a_1$ transits state $s_i$ to state $s_{2i + 1}$ and playing action $a_2$ transits state $s_i$ to state $s_{2i + 2}$, where $s_{2i + 1}$ and $s_{2i + 2}$ are both states in level $h + 1$.
In the hard instances, $r(s, a) = 0$ for all $(s, a)$ pairs except for a special state $s$ in level $H - 1$ and a special action $a \in \{a_1, a_2\}$.
For the special state $s$ and the special action $a$, we have $r(s, a) = 1$.
It is known that for such hard instances, any algorithm requires $\Omega(2^H)$ to find a policy $\pi$ with $V^{\star}(s_0) - V^\pi(s_0) \le 0.5$ with probability at least $0.9$ (see~\cite{du2019good}).
Now we construct a set of uninformative features and the hypothesis class $\Hcal$ so that $\sup_{h \in [H], \theta \in \Hcal_h}\norm{\theta}_2$ and $\sup_{x\in \Phi}\norm{x}_2$ are both bounded.

Recall that the feature mapping $\phi$ maps $\Scal \times \Acal$ to a Hilbert space $\Vcal$.
In our case, we set $\Vcal = \mathbb{R}^{d}$ with $d = 2|\Scal|$.
For each $i \in [|\Scal|]$, we define $\phi(s_i, a_1) = e_{2i + 1}$ and $\phi(s_i, a_2) = e_{2i + 2}$.
Here, for an integer $k \in [d]$, $e_k$ is the $k$-th standard basis vector. 
For each $h \in [H]$, we have $\Hcal_h = \{e_1, e_2, \ldots, e_{2|\Scal|}\}$.
Clearly, no matter which state-action pair $(s, a) \in \Scal \times \Acal$ is chosen as the special state-action pair, we always have $Q^\star \in \Hcal$, i.e., the realizability assumption is satisfied.
Moreover, both $\sup_{h \in [H], \theta \in \Hcal_h}\norm{\theta}_2$ and $\sup_{x\in \Phi}\norm{x}_2$ are bounded by $1$. 
Formally, we have the following theorem.

\begin{theorem}
For any $H > 0$, there exists a class of MDPs $\mathbb{M}$ where the number of states is $2^H - 1$ and the number of actions is $2$, together with a hypothesis class $\Hcal$ that is Bellman Complete with respect to MDPs in $\mathbb{M}$.
Moreover, $\sup_{h \in [H], \theta \in \Hcal_h}\norm{\theta}_2 \le 1$ and $\sup_{x\in \Phi}\norm{x}_2$ are bounded by $1$, and the transitions and rewards of MDPs in $\mathbb{M}$ are all deterministic. 
Any algorithm that finds a policy $\pi$ with $V^{\star}(s_0) - V^\pi(s_0) \le 0.5$ with probability at least $0.9$ for MDPs in $\mathbb{M}$ requires $\Omega(2^H)$ samples. 
\end{theorem}

\end{document}